\documentclass{article}

\usepackage{microtype}
\usepackage{graphicx}
\usepackage{caption}
\usepackage{subcaption}
\usepackage{booktabs} 
\usepackage{multirow} 
\usepackage{siunitx}
\usepackage{bbold}

\usepackage{xcolor,colortbl}

\usepackage{hyperref}

\newcommand\norm[1]{\left\lVert#1\right\rVert}

\usepackage{tikz}
\newcommand{\circledone}{\mathrel{\mathop{\tikz[baseline=(char.base)]{\node[shape=circle,draw,inner sep=0.7pt,font=\tiny] (char) {1};}}}}

\usepackage[accepted]{icml2023}

\usepackage{amsmath}
\usepackage{amssymb}
\usepackage{mathtools}
\usepackage{amsthm}
\usepackage{nicefrac}

\newcommand{\beginsupplement}{%
        \setcounter{table}{0}
        \renewcommand{\thetable}{S\arabic{table}}%
        \setcounter{figure}{0}
        \renewcommand{\thefigure}{S\arabic{figure}}%
}

\usepackage{amsmath,amsfonts,bm}

\def\ceil#1{\lceil #1 \rceil}

\def\1{\bm{1}}

\def\vtheta{{\bm{\theta}}}

\def\vf{{\bm{f}}}
\def\vg{{\bm{g}}}
\def\vh{{\bm{h}}}

\def\vs{{\bm{s}}}

\def\vx{{\bm{x}}}

\def\vz{{\bm{z}}}

\def\vtheta{{\bm{\theta}}}
\def\vvarphi{{\bm{\varphi}}}

\def\mA{{\bm{A}}}
\def\mB{{\bm{B}}}

\def\mD{{\bm{D}}}

\def\mF{{\bm{F}}}
\def\mG{{\bm{G}}}

\def\mI{{\bm{I}}}
\def\mJ{{\bm{J}}}

\def\mW{{\bm{W}}}
\def\mX{{\bm{X}}}

\def\mZ{{\bm{Z}}}

\def\mSigma{{\bm{\Sigma}}}

\DeclareMathAlphabet{\mathsfit}{\encodingdefault}{\sfdefault}{m}{sl}
\SetMathAlphabet{\mathsfit}{bold}{\encodingdefault}{\sfdefault}{bx}{n}

\theoremstyle{plain}
\newtheorem{theorem}{Theorem}
\newtheorem{proposition}[theorem]{Proposition}

\theoremstyle{definition}

\theoremstyle{remark}
\newtheorem{remark}{Remark}

\icmltitlerunning{Generalized Teacher Forcing for Learning Chaotic Dynamics}

\usepackage[belowskip=-15pt,aboveskip=0pt]{caption}
\begin{document}

\twocolumn[
\icmltitle{Generalized Teacher Forcing for Learning Chaotic Dynamics}



\icmlsetsymbol{equal}{*}

\begin{icmlauthorlist}
\icmlauthor{Florian Hess}{zi,hd,str}
\icmlauthor{Zahra Monfared}{zi,str}
\icmlauthor{Manuel Brenner}{zi,hd}
\icmlauthor{Daniel Durstewitz}{zi,hd,iwr}
\end{icmlauthorlist}

\icmlaffiliation{zi}{Department of Theoretical Neuroscience, Central Institute of Mental Health, Medical Faculty Mannheim, Heidelberg University, Mannheim, Germany}
\icmlaffiliation{hd}{Faculty of Physics and Astronomy, Heidelberg University, Heidelberg, Germany}
\icmlaffiliation{str}{Cluster of Excellence STRUCTURES, Heidelberg University, Heidelberg, Germany}
\icmlaffiliation{iwr}{Interdisciplinary Center for Scientific Computing, Heidelberg University, Heidelberg, Germany}

\icmlcorrespondingauthor{Florian Hess}{Florian.Hess@zi-mannheim.de}
\icmlcorrespondingauthor{Daniel Durstewitz}{Daniel.Durstewitz@zi-mannheim.de}

\icmlkeywords{Recurrent Neural Networks, Dynamical Systems, Nonlinear Dynamics, Attractors, Chaos, BPTT, Exploding and Vanishing Gradient Problem, Teacher forcing, EEG, ECG, Neural ODE}

\vskip 0.3in
]



\printAffiliationsAndNotice{}  

\begin{abstract}
Chaotic dynamical systems (DS) are ubiquitous in nature and society. Often we are interested in reconstructing such systems from observed time series for prediction or mechanistic insight, where by reconstruction we mean learning geometrical and invariant temporal properties
of the system in question (like attractors). However, training reconstruction algorithms like recurrent neural networks (RNNs) on such systems by gradient-descent based techniques faces severe challenges. This is mainly due to exploding gradients caused by the exponential divergence of trajectories in chaotic systems. Moreover, for (scientific) interpretability we wish to have as low dimensional reconstructions as possible, preferably in a model which is mathematically tractable. Here we report that a surprisingly simple modification of teacher forcing leads to provably strictly all-time bounded gradients in training on chaotic systems, and, when paired with a simple architectural rearrangement of a tractable RNN design, piecewise-linear RNNs (PLRNNs), allows for faithful reconstruction in spaces of at most the dimensionality of the observed system. We show on several DS that with these amendments we can reconstruct DS  better than current SOTA
algorithms, in much lower dimensions. Performance differences were particularly compelling on real world data with which most other methods severely struggled. This work thus led to a simple yet powerful DS reconstruction algorithm which is highly interpretable at the same time.
\end{abstract}

\section{Introduction}
In many scientific and engineering settings we are interested in learning the dynamics of an unknown, or hard to tackle, underlying DS that we have observed through a set of time series measurements. This may, for instance, be temperature measurements to assess climate dynamics, tracking infected cases for epidemiological dynamics, or neural recordings from the brain. Having a faithful model of the underlying dynamics enables rigorous scientific and mathematical analysis, as well as optimal predictions of the system's future temporal evolution. For many complex DS such as the brain or social networks we have only rudimentary and insufficient knowledge about the governing equations. Given the success of deep learning in so many areas and disciplines, there has been a burgeoning interest in recent years in purely data-driven approaches to DS reconstruction to bypass the traditional tedious and protracted scientific model building process \citep{brunton2016discovering, chen2018neural,  koppe2019identifying, vlachas2020backpropagation, schmidt2021identifying, brenner22, lejarza2022data}.

In DS reconstruction one aims to retrieve from data an (approximate) model of the underlying dynamics that encapsulates all its geometrical and invariant temporal properties, such as attractor states and other invariant sets (for some applications, even just capturing the topological structure of the underlying DS may be sufficient; \citet{takens1981detecting, sauer1991embedology}).
Training machine learning systems such as RNNs on time series observations from DS poses a number of severe challenges, however. In fact, DS reconstruction algorithms have so far mostly been probed only on \textit{simulated} DS \citep{brunton2016discovering, vlachas2018data, pathak2018model, champion2019data, otto_linearly_2019,bakarji2022discovering}, such as the Lorenz-63 \citep{lorenz1963deterministic} and Lorenz-96 \citep{lorenz1996predictability} models of atmospheric convection, or Navier-Stokes equations for turbulent flow. While evaluation of algorithms on such ground truth 
benchmarks
is obviously important, the relative sparsity of real-world applications demonstrates there is still a lot of ground to make (cf. \citet{brenner22, Mikhaeil22}). One important challenge here is that most complex DS are inherently chaotic \citep{olsen_oscillations_1988, van_vreeswijk_chaos_1996, durstewitz2007dynamical, duarte_quantifying_2010, faggini_chaotic_2014,  kesmia_control_2020, mangiarotti_chaos_2020, inoue_transient_2021, kamdjeu_kengne_dynamics_2021}. Chaotic DS harbor trajectories diverging exponentially fast, which leads to exploding loss gradients when training with the most common, scalable gradient-based techniques \citep{engelken2020lyapunov, Mikhaeil22}. Recent theoretical and empirical results emphasized that this a \textit{principle} problem which cannot be addressed through specifically designed network architectures
\citep{Mikhaeil22},\footnote{For instance, both classical solutions like LSTM \citep{hochreiter1997long} or GRU \citep{cho2014properties} as well as more recent architectures, like coRNN \citep{rusch2020coupled}, do not offer any way out of this conundrum
\citep{Mikhaeil22}.} but needs to be addressed at the level of training algorithms. At the same time, especially for detailed scientific analysis, we are interested in reconstruction models that are tractable and reproduce the observed data (and the underlying dynamics) in as low dimensions as possible.\footnote{Note that classical delay embedding reconstruction theorems \citep{takens1981detecting, sauer1991embedology} demand that the reconstruction space is larger than two times the box-counting dimension of the underlying attractor, but latent variable models or deeper architectures may have other means to `fill in' missing dimensions.} 

Here we tackle both these challenges. In particular, we prove that a simple amendment of teacher forcing (TF), a classical technique to keep trajectories on track while training, leads to strictly contained loss gradients, for arbitrary time horizons, without diminishing a model's ability to capture chaotic dynamics. While simple in spirit, this solves an outstanding problem in the field
\citep{Mikhaeil22} and leads to training results that surpass a wide range of previous SOTA techniques, often by large margins, especially on real-world data. We also rearrange the tractable PLRNN structure, previously suggested for DS reconstruction \citep{brenner22}, into a `shallow' one-hidden-layer design. While this yields an almost standard multi-layer-perceptron (MLP), 
we observed that -- somewhat surprisingly -- it allows for reconstruction of DS with at most as many latent dynamical variables as those of the underlying system, and makes the resulting architecture particularly well-suited for our specific variant of TF. Meanwhile, we show the resulting model can still be rewritten in standard PLRNN form, preserving its tractable design.

\section{Related Work}
\paragraph{Dynamical Systems Reconstruction} 
DS reconstruction aims at producing a generative model of an unknown DS underlying a set of observed time series variables. By `generative' here we mean that, after successful training, the (not necessarily probabilistic) model should be capable of producing time series with the same temporal and geometrical properties as those produced by the true (but unknown) DS, where `geometrical' refers to the geometry of dynamical objects (like attractors) in state space and `temporal' includes invariant properties like a system's power spectrum. Commonly these models work by approximating the unknown governing equations (or their flow), e.g. through a sufficiently expressive library of basis functions \cite{brunton2016discovering, lejarza2022data}, or by RNNs \citep{vlachas2018data, pathak2018model, koppe2019identifying, vlachas2020backpropagation, schmidt2021identifying, brenner22}, 
which are universal approximators of DS \citep{funahashi1993approximation, kimura1998learning, hanson2020universal}. Algorithms for training RNNs on DS reconstruction problems were based on probabilistic techniques like Expectation-Maximization and nonlinear Kalman filters \citep{voss2004nonlinear, koppe2019identifying, zhao2020variational} or variational inference \citep{hernandez2020nonlinear, kramer2022reconstructing}, but the to date most successful methods (cf. \citet{brenner22}) simply relied on gradient-based procedures like Back-Propagation-Through-Time (BPTT; \citet{rumelhart1986learning}). Continuous-time RNNs like Neural ODEs \citep{chen2018neural} were also widely tested for DS reconstruction \citep{raissi2018deep}, with extensions to systems of partial differential equations. Other recent ideas include Fourier Neural Operators for approximating spatially extended systems in function space \citep{li2021fourier}, RNNs or other DS approximators embedded within deep auto-encoders to extract suitable coordinate transformations or lower-dimensional state space representations \citep{champion2019data, bakarji2022discovering}, and approaches for extrapolating to `unseen' dynamical regimes \citep{patel2022using, pmlr-v162-kirchmeyer22a, ricci2022phase2vec}.
For scientific applications, mathematical tractability and dynamical interpretability\footnote{Note we mean `interpretability' in a DS sense, i.e. such that topological and geometrical properties of the resulting system can easily be analyzed, visualized, and understood.} of trained DS reconstruction models is imperative. To these ends, often locally linear techniques like piecewise-linear RNNs (PLRNNs; \citet{durstewitz2017state,koppe2019identifying})
, Koopman operator theory \citep{lusch2018deep, brunton2022modern}
or co-training of switching linear DS \citep{smith2021reverse}, were designed, but most of these require moving to a higher-dimensional space. We show that this latter issue can be avoided by simply reshaping PLRNNs
into a $1$-hidden-layer structure.

\vspace{-.1cm}
\paragraph{Controlling exploding gradients}
The exploding and vanishing gradient problem has long been recognized as a severe challenge in training RNNs on time series prediction or classification tasks which involve large time spans between relevant pieces of information \citep{bengio1994learning, hochreiter2001gradient}. For DS, a related problem is that these may contain dynamics evolving on widely differing, including very slow, time scales \cite{schmidt2021identifying}. Many approaches tried to address this issue by designing specific RNN architectures \citep{hochreiter1997long, cho2014properties, rusch2020coupled, rusch2021coupled, rusch2022lem} or imposing specific constraints on RNN parameters \citep{arjovsky2016unitary, chang2018antisymmetricrnn,erichson2021lipschitz}. However, all of these either severely limit expressiveness such that chaotic dynamics cannot be learned to begin with (because constraints or design prevent maximum Lyapunov exponents from exceeding $0$), or they still struggle with exploding gradients and hence time series from chaotic systems \cite{Mikhaeil22}. \citet{Mikhaeil22} therefore suggested to address the problem in training by controlling trajectory divergence through `sparse TF', based on the older control-theoretic idea of TF \citep{pineda1988dynamics, pearlmutter1989learning, williams1989learning, jordan1990attractor}. Sparse TF (STF) resets RNN states to control values inferred from the data, thus pulling diverging trajectories back on track and cutting off exploding gradients, at times determined from the underlying system's Lyapunov spectrum. However, this requires knowledge or empirical estimates of the system's Lyapunov exponents, and -- on the other hand -- it does not strictly avoid exploding gradients either, just attempts to strike an optimal balance between gradient divergence and not loosing relevant longer time scales. Here we suggest a simple modification of TF that solves both these problems such that gradients are \textit{strictly} contained yet chaotic dynamics can be learned, without
any knowledge required of the underlying DS' maximum Lyapunov exponent. This is one of the main results of the present work.

\section{Theoretical Analysis \& Methods}
This section will first discuss the basic problem in training RNNs on time series from chaotic systems. We will then outline a solution and its mathematical foundation, before addressing the specific RNN architecture for DS reconstruction championed in this work.
\subsection{Problem setting: Loss gradients and chaotic dynamics}
Our exposition here summarizes the key observations made in \citet{Mikhaeil22}. Most RNNs are parameterized discrete-time recursive maps of the form
\begin{equation}\label{eq:RNN_DS}
    \vz_t = \mF_\vtheta\big(\vz_{t-1}, \vs_t \big),
\end{equation}
with states $\vz_t$, optional external inputs $\vs_t$, and parameter set $\vtheta$. If unconstrained, depending on its set of parameters, an RNN may exhibit any type of limit dynamics, like convergence to fixed points, cycles of any order, quasi-periodic behavior, or chaos (in fact, RNNs are dynamically universal, cf. \citet{hanson2020universal}). If an RNN is used to reconstruct a chaotic system, it needs to be chaotic itself (otherwise the reconstruction would have failed), which entails that its 
maximum Lyapunov exponent needs to be larger than $0$. The Lyapunov exponent quantifies the exponential divergence of initially nearby trajectories. 
Defining the Jacobian of \eqref{eq:RNN_DS} by
\begin{align}\label{eq:jacobian}
    \mJ_t := \frac{\partial\mF_\vtheta\big(\vz_{t-1}, \vs_t \big)}{\partial\vz_{t-1}} = \frac{\partial\vz_t}{\partial\vz_{t-1}},
\end{align}
the maximum Lyapunov exponent of an RNN orbit $\mZ = \{\vz_1, \vz_2, \dots, \vz_T, \dots \}$  is given through the product of Jacobians $\mJ_t$ along the orbit by
\begin{align}\label{eq:lyap}
\lambda_{max} := \lim_{T \rightarrow \infty} \frac{1}{T} 
\log   \norm{ \ \prod_{r=0}^{T-2}  \mJ_{T-r} \ }_2,
\end{align}
where $\norm{\cdot}_2$ denotes the spectral norm.
\citet{Mikhaeil22} showed that the problem of training RNNs on chaotic time series using gradient descent based algorithms is ill-posed, as the condition $\lambda_{max} > 0$ will inevitably lead to diverging loss gradients.

Let us illustrate this with the most common algorithm used for training RNNs, Backpropagation Through Time (BPTT; \citet{rumelhart1986learning, werbos1990backpropagation}). Given a loss function $\mathcal{L} = 
\sum_{t=1}^T \mathcal{L}_t(\vz_t, \ \bar{\vz}_t)$ where $\vz_t$ are RNN-generated and $\bar{\vz}_t$ 
target states,
the BPTT algorithm employs the chain rule along the RNN unrolled in time to compute loss gradients w.r.t. model parameters 
$\theta_i \in \vtheta$,
\begin{align}\label{eq:BPTT}
    \frac{\partial{\mathcal{L}}}{\partial{\theta_i}} = 
    \sum_{t=1}^T \frac{\partial{\mathcal{L}_t}}{\partial{\theta_i}} \hspace{.4cm} \text{with} \hspace{.5cm} \frac{\partial{\mathcal{L}_t}}{\partial{\theta_i}} = \sum_{r=1}^t
    \frac{\partial{\mathcal{L}_t}}{\partial{\vz_t}}
     \frac{\partial{\bm{z}_t}}{\partial{\bm{z}_r}} \frac{\partial^+{\bm{z}_r}}{\partial{\theta_i}},
\end{align}
where $\partial^+$ denotes the immediate derivative. The crucial observation now is that Eq. \eqref{eq:BPTT} contains the same product of Jacobians,
\begin{align}\label{eq:prod_jacobians}
\begin{split}
\frac{\partial{\bm{z}_t}}{\partial{\bm{z}_r}} &= \frac{\partial{\bm{z}_t}}{\partial{\bm{z}_{t-1}}} \frac{\partial{\bm{z}_{t-1}}}{\partial{\bm{z}_{t-2}}} \cdots \frac{\partial{\bm{z}_{r+1}}}{\partial{\bm{z}_{r}}} \\
&= \bm{J}_t\bm{J}_{t-1} \cdots \bm{J}_{r+1} = \prod_{k=0}^{t-r-1}\bm{J}_{t-k}, 
\end{split}
\end{align}
that occurs in the definition of the maximum Lyapunov exponent in Eq. \eqref{eq:lyap}. As \citet{Mikhaeil22} strictly prove, this entails exponentially exploding loss gradients for $T \rightarrow \infty$ when training RNNs on chaotic systems whose behavior they are supposed to reproduce. 
In practice,
unreliable and ill-behaved training sets in even for moderate sequence lengths $T$, and architectures like LSTM designed to control gradient flows or simple gradient clipping techniques are not sufficient to contain the problem \cite{Mikhaeil22}. 
\subsection{Generalized Teacher Forcing}
What is needed in addition is a procedure for forcing diverging trajectories back onto their targets while training for DS reconstruction. This has been recognized long ago, and the classical control technique here is TF 
\citep{pineda1988dynamics, pearlmutter1989learning, williams1989learning, jordan1990attractor}.
TF replaces current RNN states \textit{after} computing the loss by data-inferred values, classically at each time step
, or at times strategically chosen according to the system's maximum Lyapunov exponent in sparse-TF \cite{Mikhaeil22}. 
Motivated by the problem of unstable solutions after RNN training with classical TF, \citet{doya1992bifurcations} suggested to linearly interpolate between RNN-generated states $\vz_t$ and target states $\bar{\vz}_t$:
\begin{align}\label{eq:GTF}
    \tilde{\vz}_t := (1-\alpha)\vz_t + \alpha \bar{\vz}_t,
\end{align}
with $0 \leq \alpha \leq 1$. This simple idea, which we will refer to as \textit{Generalized Teacher Forcing} (GTF), has in fact not been systematically studied so far. It turns out that GTF, with the right choice of $\alpha$, can \textit{fully rectify} the exploding gradients problem in learning chaotic dynamics. 

Now let us unwrap how GTF \eqref{eq:GTF} will impact RNN \eqref{eq:RNN_DS} training. The state replacement, Eq. \eqref{eq:GTF}, is performed prior
to applying RNN map $\mF_\vtheta$ at each training time step, i.e. $\vz_t = \mF_\vtheta(\tilde{\vz}_{t-1})$.
According to the chain rule, this leads to the following factorization of Jacobians $\mJ_t$:
\begin{align}\nonumber
    \bm{J}_t &\, = \, \frac{\partial{\bm{z}_t}}{\partial{\bm{z}_{t-1}}} = \frac{\partial{\bm{z}_t}}{\partial{\tilde{\vz}_{t-1}}}\frac{\partial{\tilde{\vz}_{t-1}}}{\partial{\bm{z}_{t-1}}}
    = \frac{\partial{\mF_\vtheta(\tilde{\vz}_{t-1})}}{\partial{\tilde{\vz}_{t-1}}} \frac{\partial{\tilde{\vz}}_{t-1}}{\partial{\bm{z}_{t-1}}}
    \\\label{eq:GTF_jacobian}
    & \, = \, (1-\alpha) \tilde{\mJ}_t,
\end{align}
where $\tilde{\mJ}_t$ is the Jacobian of $\mF_\vtheta$ evaluated at the forced state $\tilde{\vz}_{t-1}$. Plugging into \eqref{eq:prod_jacobians} gives the following expression for the product of Jacobians under GTF:
\begin{align}\label{eq:jac_chain_GTF}
\begin{split}
    \frac{\partial{\bm{z}_t}}{\partial{\bm{z}_r}} = \prod_{k=0}^{t-r-1}\bm{J}_{t-k} &= \prod_{k=0}^{t-r-1}(1-\alpha)\tilde{\bm{J}}_{t-k} \\ &= (1-\alpha)^{t-r}\prod_{k=0}^{t-r-1}\tilde{\bm{J}}_{t-k}.
\end{split}
\end{align}
For $\alpha = 0$ (no forcing) this simply yields vanilla BPTT \eqref{eq:prod_jacobians}, while for $\alpha = 1$ Eq. \eqref{eq:jac_chain_GTF} evaluates to zero and there will be no gradient propagation. For values in between, GTF controls the Jacobian product norm $\norm{\frac{\partial{\bm{z}_t}}{\partial{\bm{z}_r}}}$ as illustrated in Fig. \ref{fig:jac_norms_gtf}. Fig. \ref{fig:gtf_scheme} summarizes the GTF principle and notation.

\begin{figure}[!htb]
    \centering
     \includegraphics[width=0.99\linewidth]{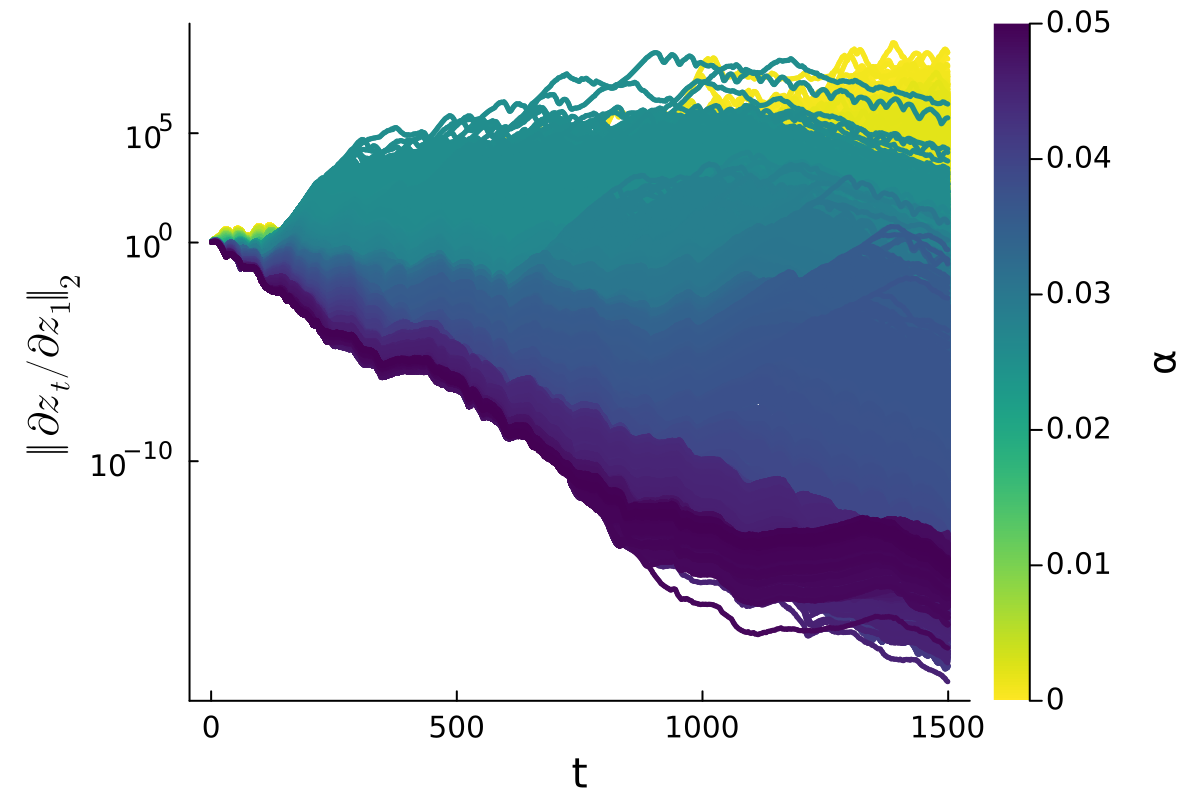}
     \vspace{-.3cm}
    \caption{Norms of Jacobian product series given in Eq. \eqref{eq:jac_chain_GTF} for an RNN trained by BPTT+GTF on the Lorenz-63 system as function of sequence length $t$. The RNN state is initialized based on the ground truth and then propagated for $t$ time steps. 
    The Jacobians diverge if $\alpha$ is chosen too small, but converge if too large.}
    \label{fig:jac_norms_gtf}
\end{figure}
\begin{figure*}[!htp]
    \centering    \includegraphics[width=0.85\linewidth]{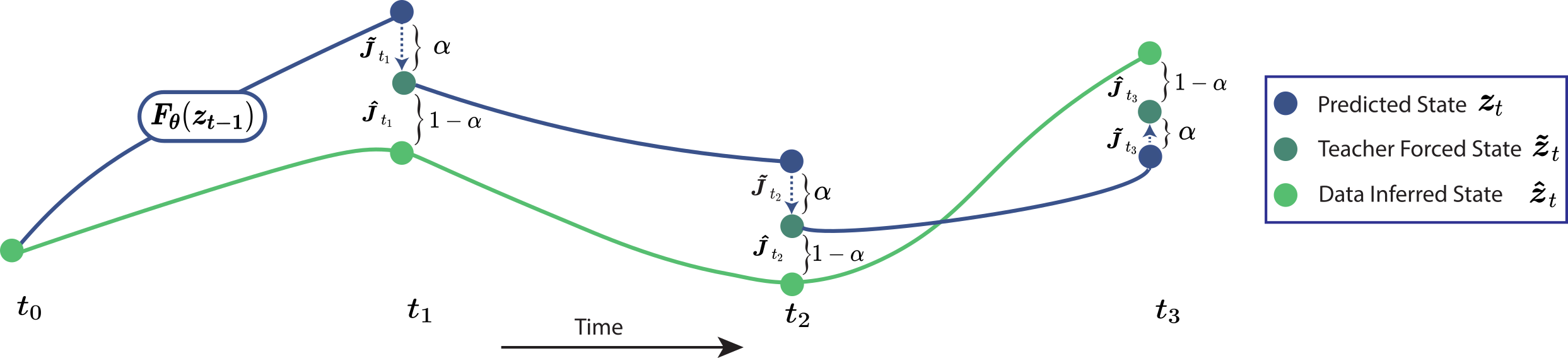}
    \caption{Principle of Generalized Teacher Forcing.
    }
    \label{fig:gtf_scheme}
\end{figure*}
Now assume that for any fixed given set of parameters $\vtheta$, 
$\, \mathcal{J} \, = \, \{ \tilde{\mJ}_{\kappa} \}_{\kappa \in \mathcal{K}}\, $ is the set of all different Jacobians of an RNN (which may be countable or uncountable).
We define 
\begin{align}\nonumber
& \tilde{\sigma}_{\max} =  \sup \mathcal{S}_1 \, := \, \sup \bigg\{ \norm{\tilde{\mJ}_{\kappa}}=
\sigma_{\max}(\tilde{\mJ}_{\kappa})\, : \, \tilde{\mJ}_{\kappa} \in \mathcal{J} \bigg\},
\\[1ex]\label{eq:sup:inf}
& \tilde{\lambda}_{\min} = \inf \mathcal{S}_2 \, := \, \inf \bigg\{ \lambda_{\min}(\tilde{\mJ}_{\kappa}) : \, \tilde{\mJ}_{\kappa} \in \mathcal{J} \bigg\}  \geq 0, 
\end{align}
where $\lambda_{\min}(\tilde{\mJ}_{\kappa}) \, = \, \min \bigg\{ |\lambda_{j}| \, : \, \lambda_j \in eig(\tilde{\mJ}_{\kappa}) \bigg\}$. 

Obviously the nonempty set $\mathcal{S}_2 \subset \mathbb{R}$ is bounded from below by definition, but the set $\mathcal{S}_1 \subset \mathbb{R}$ is not necessarily bounded from above. However, here we will consider the set of extended real numbers $\bar{\mathbb{R}} := \mathbb{R} \cup \{- \infty, + \infty\}$ and define the supremum of a set not bounded from above as $+ \infty$. 
The following proposition states a necessary condition for the existence of chaos in RNNs:
\begin{proposition}\label{prop:chaos}
If the RNN \eqref{eq:RNN_DS} is chaotic, then $\, \tilde{\sigma}_{\max} \, > \, 1$.    
\end{proposition}
\begin{proof}
See Appx. \ref{app:proof_prop_chaos}.  \end{proof}
The next proposition now sets the stage for a proper choice of the GTF parameter $\alpha$:
\begin{proposition}\label{prop:2}
Consider an RNN given by \eqref{eq:RNN_DS} and let $\, \tilde{\sigma}_{\max} \geq 1$. 
\begin{itemize}
\item[(i)] If $ \, \alpha \, = \, \alpha^{*} \, := \, 1- \frac{1}{\tilde{\sigma}_{\max}}$, then the Jacobian product series $\frac{\partial{\vz_t}}{\partial{\vz_r}}$ will be bounded from above, i.e. its norm will not diverge as $t-r \to \infty$.
    \item[(ii)] Assume that  $\alpha \, = \, \alpha^{*}$ and define $\tilde{\gamma} \, := \, \frac{\tilde{\lambda}_{\min}} {\tilde{\sigma}_{\max}}$ ($0 \leq \tilde{\gamma} \leq 1$). If $\tilde{\gamma} \, = \, 1$, then $\frac{\partial{\vz_t}}{\partial{\vz_r}}$ neither explodes nor vanishes for $t-r \to \infty$. When $\tilde{\gamma} \, \neq \, 1$, the Jacobian $\frac{\partial{\vz_t}}{\partial{\vz_r}}$ will not explode, but may potentially vanish as $t-r$ goes to infinity. Furthermore if $ \, \lim_{t-r \to \infty}\norm{\frac{\partial{\vz_t}}{\partial{\vz_r}}} =0$, then the closer $\tilde{\gamma}$ is to $1$, the slower $\frac{\partial{\vz_t}}{\partial{\vz_r}}$ tends to zero for $t-r \to \infty$.
\end{itemize}
\end{proposition}
\begin{proof}
See Appx. \ref{app:proof_prop2}.  \end{proof}
\subsection{Reconstruction models}\label{subsec:rec_modles}
We will base our experiments on the specific class of PLRNNs 
\citep{durstewitz2017state, koppe2019identifying}, since they are mathematically tractable in the sense that fixed points, cycles, and other invariant sets can, in principle, be computed exactly and semi-analytically for them \citep{monfared2020transformation, schmidt2021identifying, brenner22}. This facilitates their post-training DS analysis in scientific contexts. In its simplest form, the PLRNN is given by
\begin{equation}\label{eq:con_PLRNN}
    \vz_t = \mA \vz_{t-1} + \mW \phi(\vz_{t-1}) + \vh_0,
\end{equation}
with diagonal matrix $\mA \in \mathbb{R}^{M\times M}$, off-diagonal matrix $\mW \in \mathbb{R}^{M \times M}$, bias $\vh_0 \in \mathbb{R}^{M}$, and rectified-linear-unit nonlinearity $\phi(\cdot) = \textrm{ReLU}(\cdot) = \max(0, \cdot)$.
\citet{brenner22} extended this basic structure by adding a linear spline basis expansion, dubbed the dendritic PLRNN (dendPLRNN), to increase the expressivity of each unit's nonlinearity and thus enable DS reconstructions in lower dimensions,
\vspace{-.2cm}
\begin{equation}\label{eq:dendPLRNN}
    \vz_t = \mA \vz_{t-1} + \mW \sum_{b=1}^B\alpha_b\phi(\vz_{t-1} - \vh_b) + \vh_0,
\end{equation}
with slope-threshold pairs $\{\alpha_b, \vh_b\}_{b=1}^B$, where $B$ is the number of bases. It can be shown that the dendPLRNN can be reformulated as a higher-dimensional conventional PLRNN \cite{brenner22}.

However, in this work we will specifically consider the following ``$1$-hidden-layer'' ReLU-based RNN, which we will refer to as \textit{shallow} PLRNN (shPLRNN):
\begin{align}\label{eq:shPLRNN}
     \vz_t = \mA \vz_{t-1} + \mW_1 \phi (\mW_2 \vz_{t-1} + \vh_2) +\vh_1,
\end{align}
with latent states $\vz_t \in \mathbb{R}^M$, diagonal matrix $\mA \in \mathbb{R}^{M\times M}$, rectangular connectivity matrices $\mW_1 \in \mathbb{R} ^{M\times L}$ and $\mW_2 \in \mathbb{R} ^{L\times M}$, and thresholds $\vh_2 \in \mathbb{R}^{L}$ and $\vh_1 \in \mathbb{R}^{M}$.
With $L > M$, by expanding each unit's activation into a weighted sum of ReLU nonlinearities, this formulation appears similar to the dendPLRNN, and indeed this intuition is confirmed by the following Proposition: 

\begin{proposition}\label{prop:sh}
Any shPLRNN given by \eqref{eq:shPLRNN} can be rewritten in the form of the dendPLRNN \eqref{eq:dendPLRNN}. It follows, in particular, that fixed points and cycles of \eqref{eq:shPLRNN} can be computed in an analogous way as for the dendPLRNN.
Vice versa, any $M$-dimensional dendPLRNN can be reformulated as an $M$-dimensional shPLRNN with hidden layer size $L=M \cdot B$.
\end{proposition}
\begin{proof}
See Appx. \ref{app:proof_prop_sh} \end{proof}
Similar to the dendPLRNN, the shPLRNN can be equipped with a clipping mechanism that prevents states from diverging to infinity due to the unbounded ReLU nonlinearity:
\begin{align}\nonumber
     \vz_t 
    = \mA \vz_{t-1} + \mW_1 \big[\phi(\mW_2 \vz_{t-1} + \vh_2) 
    \\\label{eq:bounded_shPLRNN}
   - \phi\left(\mW_2 \vz_{t-1}\right)\big] + \vh_1.
\end{align}
This guarantees bounded orbits provided the largest absolute eigenvalue of $\bm{A}$ is smaller than $1$ (as shown in Appx. \ref{app:proof_boundedness}).

Finally, to map from the PLRNN's $M$-dimensional latent to the $N$-dimensional observation space, a linear observation model (i.e., linear output layer) is used:
\begin{equation}\label{eq:obs_model}
    \hat{\vx}_t = \mG_\vvarphi (\vz_t) = \mB \vz_t.
\end{equation}
\subsection{Model training}\label{subsec:model_training}
We use BPTT combined with GTF to train RNNs on time series $\bm{X} \in \mathbb{R}^{N \times T}$ from chaotic DS, where $T$ is the length of the time series and $N$ is the number of observed variables. Control signals $\hat{\vz}_t$ 
used for TF are inferred from the observations by inversion of the observation model $\mG_\vvarphi$ 
 \citep{Mikhaeil22},
\vspace{-.2cm}
\begin{equation}\label{eq:invTF}
    \hat{\vz}_t := \mG_\vvarphi^{-1}(\bm{x}) = \mB^+ \vx_t,
\end{equation}
where $\mB^+$ denotes the Moore–Penrose (pseudo-) inverse of $\mB$.
\footnote{More generally, one may think of invertible neural networks (INNs; \citet{dinh2017density, ardizzone2018analyzing}) for linking latent states to observations, but here we contend ourselves with a simple linear model.} 
In the simplest case, if the RNN's latent dimension and the dimension of the observations match, i.e. $N=M$, one could also simply fix $\bm{B} = \mathbb{1}_{N\times N}$ and train on the observations directly, taking $\hat{\vz}_t = \vx_t$ as the control signal (see \citet{brenner22}). 

How to choose the GTF parameter $\alpha$ in practice? Building on Proposition \ref{prop:2}, to avoid exploding gradients altogether, one would need to set $\alpha$ according to the maximum singular value of the RNN Jacobians, Eq. \eqref{eq:sup:inf}. However, exactly computing ${\tilde{\sigma}_{\max}}$ is intractable for most RNN architectures. This can easily be seen for the shPLRNN \eqref{eq:shPLRNN}, for which the Jacobian \eqref{eq:jacobian} is given by
\begin{equation}\label{eq:shPLRNN_jac}
    \mJ^{(sh)}_t = \mA + \mW_1 \tilde{\mD}_{\Omega(t-1)}\mW_2,
\end{equation}
where $\tilde{\mD}_{\Omega(t-1)} = \text{diag} \big(\tilde{d}_{1,t-1}, \tilde{d}_{2,t-1}, \cdots, \tilde{d}_{L,t-1} \big)$ is an $L \times L$ diagonal binary indicator matrix with $d_{l,t-1}=1$ if $\sum_{j=1}^{M} w^{(2)}_{lj} \,z_{j,t-1}> -h^{(2)}_l$, for $\mW_2 = \big[w^{(2)}_{ij}\big]$, $\vh_2 = \big[h^{(2)}_i\big]$, and $0$ otherwise. Each possible configuration 
of $1$'s and $0$'s on the diagonal of $\Tilde{\mD}_\Omega$ corresponds to a different linear subregion of the state space, in which the Jacobian \eqref{eq:shPLRNN_jac} is constant and the dynamics of \eqref{eq:shPLRNN} is linear. Hence, to compute ${\tilde{\sigma}_{\max}}$ one would need to evaluate the Jacobians of all $2^L$ linear subregions, which is generally infeasible, especially since ${\tilde{\sigma}_{\max}}$ would need to be re-evaluated after each parameter update during training. The cheapest way around this issue is to choose $\alpha$ according to an upper bound of ${\tilde{\sigma}_{\max}}$:
\vspace{-.2cm}
\begin{align}\label{eq:alpha_upper_bound}
    \alpha^{(n)} = 1 - \frac{1}{\ceil{\tilde{\sigma}_{\max}}^{(n)}},
\end{align}
where $n$ denotes the $n$-th optimization step and $\ceil{\tilde{\sigma}_{\max}}$ is the upper bound given by
\begin{align}\label{eq:sigma_upper_bound}
\begin{split}
    \ceil{\tilde{\sigma}_{\max}} &:= \norm{\mA} + \norm{\mW_1} \norm{\mW_2} \\ &\geq \max_{\Tilde{\mD}_\Omega} \norm{\mA + \mW_1 \Tilde{\mD}_{\Omega}\mW_2} = \tilde{\sigma}_{\max}.
\end{split}
\end{align}
A slightly more expensive alternative heuristic is to approximate ${\tilde{\sigma}_{\max}}$ by computing the Jacobians at states given by the teacher signals of a given training batch $\mX$ (since, ultimately, the RNN is required to generate time series with these same signatures). Letting $\vx_t^{(p)}$ denote the observation vector of the $p$-th sequence in the batch at time $t$, the teacher signals are given by $\hat{\vz}^{(p)}_t = \mG_\vvarphi^{-1}(\vx_t^{(p)})$, for which the Jacobians $\hat{\mJ}^{(p)}_{t}$ (cf. Fig. \ref{fig:gtf_scheme}) can be computed using \eqref{eq:shPLRNN_jac}. We then have an estimate $\hat{\sigma}_{\max}$ as
\vspace{-.2cm}
\begin{align}\label{eq:sigma_max_estimate}
    \hat{\sigma}_{\max} = \max_{t, p} \norm{\hat{\mJ}^{(p)}_{t}},
\end{align}
and can choose $\alpha$ accordingly.

However, we find that in practice the estimates above are too conservative, leading to suboptimal performance. 
Instead, we derive 
an estimate directly based on Eq. \eqref{eq:jac_chain_GTF}.  
First, note that ideal error propagation is achieved when the chain of Jacobians connecting temporally most distal states in a training sequence of length $T$ is close to the identity,
\begin{align}\label{eq:jac_prod_identity}
\begin{split}
    \frac{\partial{\bm{z}_T}}{\partial{\bm{z}_1}} = (1-\alpha)^{T-1}\prod_{k=0}^{T-2}\tilde{\bm{J}}_{T-k} \overset{!}{=} \mathbb{1}.
\end{split}
\end{align} 
Provided that the Jacobian product series is non-singular, 
it follows that
\begin{align}
    (1-\alpha)\mathcal{\mG}(\tilde{\mJ}_{T:2}) &\overset{!}{=} \mathbb{1}, \\ \textrm{where} \hspace{.4cm} \mathcal{\mG}(\tilde{\mJ}_{T:2}) &:= \left(\prod_{k=0}^{T-2}\tilde{\bm{J}}_{T-k}\right)^{\frac{1}{T-1}}.\label{eq:jac_prod_power}
\end{align}
Taking the norm on both sides of Eq. \eqref{eq:jac_prod_power}, rearranging,  and assuming, as above, that we can replace forced Jacobians $\tilde{\mJ}_t^{(p)}$ of the $p$-th sequence in the current batch with Jacobians $\hat{\mJ}_t^{(p)}$ evaluated at data-inferred states, we obtain a collection $\{\alpha^{(p)}\}$ which we condense into a single estimate by taking
\begin{equation}\label{eq:alpha_geomean}
    \alpha= \max_p \alpha^{(p)} = \max_p \left[ 1 - \frac{1}{\norm{\mathcal{\mG}(\hat{\mJ}_{T:2}^{(p)})}} \right] .
\end{equation}
Since computing \eqref{eq:jac_prod_power} requires evaluating Jacobian products which cause exploding gradients in the first place, we
use an approximation which foregoes 
computing those products altogether, see Appx. \ref{appx:geomean_approx} for details. Furthermore, to ensure that replacing Jacobians $\Tilde{\mJ}_t$ at forced states with data inferred Jacobians $\hat{\mJ}_t$ remains valid throughout training, we employ an annealing strategy, which starts with strong forcing ($\alpha = 1$) that decays throughout training while remaining bounded from below by
\eqref{eq:alpha_geomean}.\footnote{More generally, annealing protocols have previously been observed to improve TF-based training in RNNs \citep{bengio2015scheduled, vlachas2023learning}.} The full training protocol 
is detailed in Appx. \ref{appx:annealing}, and will be referred to as \textit{adaptive} GTF (aGTF). 
As a reference, in our experiments below we will also employ fine-tuning $\alpha$ via grid search. Note that GTF is only used for \textit{training} the model, not during actual testing.

\section{Results}
We evaluate GTF using the shPLRNN 
on  
simulated DS 
and real-world data sets, and compare its performance to a variety of other DS reconstruction algorithms. We will first introduce the data sets and evaluation measures used.
\subsection{DS data sets}\label{subsec:datasets}
\paragraph{Lorenz-63, Lorenz-96, and multiscale Lorenz-96} Both the 3d Lorenz-63 \cite{lorenz1963deterministic} and the higher-dimensional, spatially extended Lorenz-96 \cite{lorenz1996predictability} ODE systems were conceived as simple models of atmospheric convection with chaotic behavior in some regime (see Appx. \ref{appx:datasets} for more details). Here we include them mainly because they often served as DS benchmarks in the past. But by now almost all DS reconstruction algorithms achieve satisfactory performance on them, so that the focus here will be on the much more challenging empirical data sets.
As a somewhat more challenging benchmark, we also consider a (partially observed) multiscale variant of the Lorenz-96 system \citep{thornes2017use}.

\vspace{-.1cm}
\paragraph{Electrocardiogram (ECG) data} 
ECG recordings of a human subject were taken from the open-access PPG-DaLiA dataset \citep{Reiss_2019}. The ECG is a scalar physiological (heart muscle potential) time series, which we delay-embed \cite{sauer1991embedology} into $5d$ using the PECUZAL algorithm \citep{kramer2021unified}. The time series shows signs of chaotic behavior, as indicated by its empirically determined maximum Lyapunov exponent (cf. Appx \ref{appx:ecg}).

\vspace{-.1cm}
\paragraph{Electroencephalogram (EEG) data} 
As another challenging 
real-world data set we use open-access EEG recordings from a human subject under task-free, quiescent conditions with eyes open \citep{schalk_bci2000_2004}. The $64d$ dataset consists of $\SI{60}{\second}$ of brain activity measured through $64$ electrodes placed across the scalp. This dataset has previously been shown to bear signatures of chaotic activity, like a fractal dimensionality and positive maximum Lyapunov exponent ($\lambda_{max} \approx 0.017$; \citet{Mikhaeil22}).

We emphasize that -- unlike 
common simulated benchmarks -- empirical data often come with a number of additional challenges, like only a 
fraction of all dynamical variables observed, from potentially very high-dimensional generating systems, through only partly known and often intricate observation functions, with possibly high levels of dynamical and observation noise and unknown external perturbations, 
multiple temporal and spatial scales (brain recordings in particular), and potentially non-stationary behavior.

\subsection{Evaluation measures}\label{subsec:evaluation_measures}
To assess DS reconstruction quality, besides short-term prediction, we need to ensure that \textit{invariant} properties of the underlying system, like an attractor's geometrical structure in state space and its temporal signatures, have been adequately captured. In our choice of measures we follow \citet{koppe2019identifying, brenner22, Mikhaeil22}. For implementation details we refer to Appx. \ref{appx:eval_measures}.

\vspace{-.1cm}
\paragraph{Geometrical Measure}
To assess the agreement in true and reconstructed attractor geometries, we compute a Kullback-Leibler (KL) divergence based on the natural measure induced by system trajectories in state space on the invariant set ($D_{\textrm{stsp}}$). Specifically, $D_{\textrm{stsp}}$ 
is defined in the (potentially extended; \citet{sauer1991embedology}) observation space between  
estimated probability distributions 
$p(\vx)$, generated by trajectories of the true system, and $q(\vx)$, generated by orbits of the reconstructed system, see Appx. \ref{appx:dstsp} for details.

\vspace{-.1cm}
\paragraph{Temporal Measure}
As in \citet{Mikhaeil22}, we check for invariant temporal agreement of ground truth and generated orbits by computing the Hellinger distance ($H$) between the power spectra of all dynamical variables (see Appx. \ref{appx:hellinger_distance} for details). 
The Hellinger distance is a proper metric with $0\leq H \leq 1$, where $0$ represents perfect agreement between its arguments. To produce a single number, the Hellinger distance is averaged across all dynamical variables of the observed system, referred to here as $D_H$. 

\vspace{-.1cm}
\paragraph{Prediction error (PE)}
We further employ a short-term prediction error to assess each model's capability to perform forecasts on the chaotic time series. Note that, in general, accurate prediction of chaotic systems is an ill-posed problem as even tiny model approximation errors, differences in initial conditions, or noise sources are exponentially amplified, leading to rapid divergence from ground truth orbits. Hence, measures based on prediction accuracy are only of limited validity or unsuitable for assessing reconstruction quality, and only work for limited time horizons determined by the system's Lyapunov spectrum \citep{Mikhaeil22}. To assess short-term predictability, here we compute a mean-squared $n$-step prediction error 
between $n$-step forward propagated initial states from the test set and their corresponding ground truth values, see Appx. \ref{appx:prediction_error}.
\subsection{Experimental evaluation}\label{subsec:experimental_evaluation}
We compared DS reconstructions on the benchmarks defined above for the shPLRNN trained by GTF 
to a variety of other SOTA reconstruction methods, chosen to represent four major classes of model architectures and training strategies in use. 
These include DS reconstruction techniques based on 1) gated RNN models, specifically LSTMs trained using truncated BPTT (TBPTT) 
\citep{vlachas2018data}, 2) Reservoir Computing (RC) / Echo State Machines \citep{pathak2018model}, 3) library methods employing symbolic regression, namely Sparse Identification of Nonlinear Dynamical Systems (SINDy) \citep{brunton2016discovering}, and 4) ODE-based formulations like Neural ODEs (N-ODE) \citep{chen2018neural} and Long-Expressive-Memory (LEM) \citep{rusch2022lem}\footnote{LEMs, although not specifically designed for DS reconstruction, are universal approximators that have been exemplified on DS problems and are a current SOTA for addressing the EVGP.} (we also tested other N-ODE variants like Latent-ODE and ODE-RNN \citep{rubanova_latent_2019}, with similar results, see Appx. \ref{appx:sota_settings}). 
Moreover, we compare our method to the similar approach of \citet{brenner22} using dendPLRNNs trained with BPTT and a specific sparse TF protocol, dubbed identity TF (id-TF). For all comparison methods we determined optimal hyper-parameters through grid search (see Appx. \ref{appx:sota_settings}), while trying to keep the total number of trainable parameters about the same (see Table \ref{tab:SOTA_comparison}).\footnote{This was, however, not fully possible. For instance, RC required many more trainable parameters than other methods to yield any sensible reconstruction results.} For our shPLRNN we report in Table \ref{tab:SOTA_comparison} results with both fixed GTF parameter $\alpha$ 
determined by grid search (see 
Fig. \ref{fig:alpha_gs}) as well as 
for aGTF, obtaining similar performance. Of course, GTF -- 
like TF more generally -- is only employed for model training, not during testing where systems evolve autonomously.%
\begin{table*}[!ht]
\definecolor{Gray}{gray}{0.925}
\vspace{-.2cm}
\caption{SOTA comparisons. Reported values are median $\pm$ median absolute deviation over $20$ independent training runs. `dim' refers to the model's state space dimensionality (number of dynamical variables). $\lvert \bm{\theta} \rvert$ denotes the total number of \textit{trainable} parameters. }
\centering
\scalebox{0.99}{
\begin{tabular}{l l c c c c c}
        \toprule
        Dataset	&	Method	&	$D_{\textrm{stsp}}$ $\downarrow$ & $D_H$ $\downarrow$  & $\textrm{PE}(20)$ $\downarrow$	&  dim &
        $\lvert \bm{\theta} \rvert$ \\
        \midrule
        \multirow{8}{4em}{ECG (5d)}
        &    shPLRNN + GTF  & $\textbf{4.3}\pm\textbf{0.6}$ & $\textbf{0.34}\pm\textbf{0.02}$ & $(\textbf{2.4}\pm\textbf{0.1}) \cdot \textbf{10}^{\textbf{-3}}$ & $5$ & $2785$ \\
        &    shPLRNN + aGTF  & $\textbf{4.5} \pm \textbf{0.4}$ & $\textbf{0.34} \pm \textbf{0.02}$ & $(\textbf{2.4} \pm \textbf{0.2}) \cdot {\textbf{10}}^{{\textbf{-3}}}$ & $5$ & $2785$ \\
        &    shPLRNN + STF      & $7.1 \pm 1.8$ & $0.38 \pm 0.03$ & $(5 \pm 2) \cdot 10^{-3}$ & $5$ & $2785$ \\
        &    dendPLRNN + id-TF     & $5.8\pm0.6$ & $0.37\pm0.06$ & $(4.0\pm0.4) \cdot 10^{-3}$ & $35$ &  $3245$ \\
        &	 {RC}	            & $5.3\pm1.7$ & $0.39\pm0.05$ &  $(4\pm1) \cdot 10^{-3}$ & $1000$ & $5000$ \\
        &	 {LSTM-TBPTT}	            	& $15.2\pm0.5$ & $0.73\pm0.02$ & $(2.5\pm0.5) \cdot 10^{-2}$  & $70$ & $5920$  \\
        &	 {SINDy}	        & diverging &	diverging & diverging & $5$ & $3960$  \\
        &	 {N-ODE}	   & $12.2\pm0.7$ & $0.7\pm0.03$ &  $(4.1\pm0.1) \cdot 10^{-1}$ & $5$ & $4955$  \\
        &	 {LEM}	       & $16.3\pm0.2$ & $0.56\pm0.04$ &  $(7.4\pm0.1) \cdot 10^{-1}$ & $62$ & $4872$ \\

        \midrule
        \multirow{8}{4em}{EEG (64d)}
        &    shPLRNN + GTF        & $\textbf{2.1}\pm\textbf{0.2}$ & $\textbf{0.11}\pm\textbf{0.01}$ & $(5.5\pm0.1) \cdot 10^{-1}$ & $16$ & $17952$  \\
        &    shPLRNN + aGTF        & $\textbf{2.4} \pm \textbf{0.2}$ & $\textbf{0.13} \pm \textbf{0.01}$ & $(5.4\pm0.6) \cdot 10^{-1}$ & $16$ & $17952$  \\
        &    shPLRNN + STF    & $14 \pm 7$ & $0.50 \pm 0.16$ & $(\textbf{2.5}\pm\textbf{0.3}) \cdot \textbf{10}^{\textbf{-1}}$  & $16$ & $17952$  \\
        &    dendPLRNN + id-TF     & $3\pm1$ & $0.13\pm0.04$ & $(3.4\pm0.1) \cdot 10^{-1}$  & $105$ & $18099$  \\
        &	 {RC}	            & $14\pm7$ & $0.54\pm0.15$ & $(5.9\pm0.3) \cdot 10^{-1}$ & $448$ & $28672$ \\
        &	 {LSTM-TBPTT}	            & $30\pm21$ & $0.2 \pm 0.1$ & $(9.2\pm 2.3) \cdot 10^{-1}$ & $160$ & $51584$  \\
        &	 {SINDy}	        & diverging &	diverging & diverging & $64$ & $133120$  \\
        &	 {N-ODE}	        &  $20\pm0.5$  &	$0.47\pm0.01$ & $(5.5\pm0.2) \cdot 10^{-1}$ & $64$ & $17995$  \\
        &	 {LEM}	        	 & $10.2\pm1.5$ & $0.38\pm0.06$ & $(8.2\pm0.6) \cdot 10^{-1}$ & $76$ & $18304$ \\
        \bottomrule
\end{tabular}
}
\label{tab:SOTA_comparison}
\end{table*}

As Table \ref{tab:SOTA_comparison_benchmark} in Appx. \ref{appx:sota_settings} reveals, shPLRNN+GTF outperforms, or performs about on par with, the other methods on the simulated DS benchmarks Lorenz-63, Lorenz-96, and multiscale Lorenz-96 (see also Figs. \ref{fig:reduced-mslor96-longterm}, \ref{fig:red_mslor96_forecast_traces}, \ref{fig:reduced-mslor96-shortterm}). However, as noted above, 
by now almost any recent DS reconstruction method performs reasonably well on these, essentially leading to a `ceiling effect'.  
More importantly, shPLRNN+GTF has a clear edge over all other methods on the empirical EEG and ECG data, in some comparisons by almost an order of magnitude (Table \ref{tab:SOTA_comparison}).\footnote{As observed previously \cite{brenner22}, SINDy severely struggles if the ``right'' functional terms are not in its library, as is likely for any empirical situation. Unlike for the Lorenz-63/96, we only observed diverging solutions for SINDy on the EEG and ECG data.} 
This is significant, as DS reconstruction methods have so far mostly been tested on simulated DS models only, like the Lorenz equations, but rarely on empirical data. 
As  
noted above, real-world data usually contain multiple additional complexities not present in most simulated benchmarks. 
The reconstruction algorithm advanced here is apparently better able to deal with these various complicating factors 
met in empirical data.
\begin{figure}[!ht]
    \centering
\includegraphics[width=0.99\linewidth]
{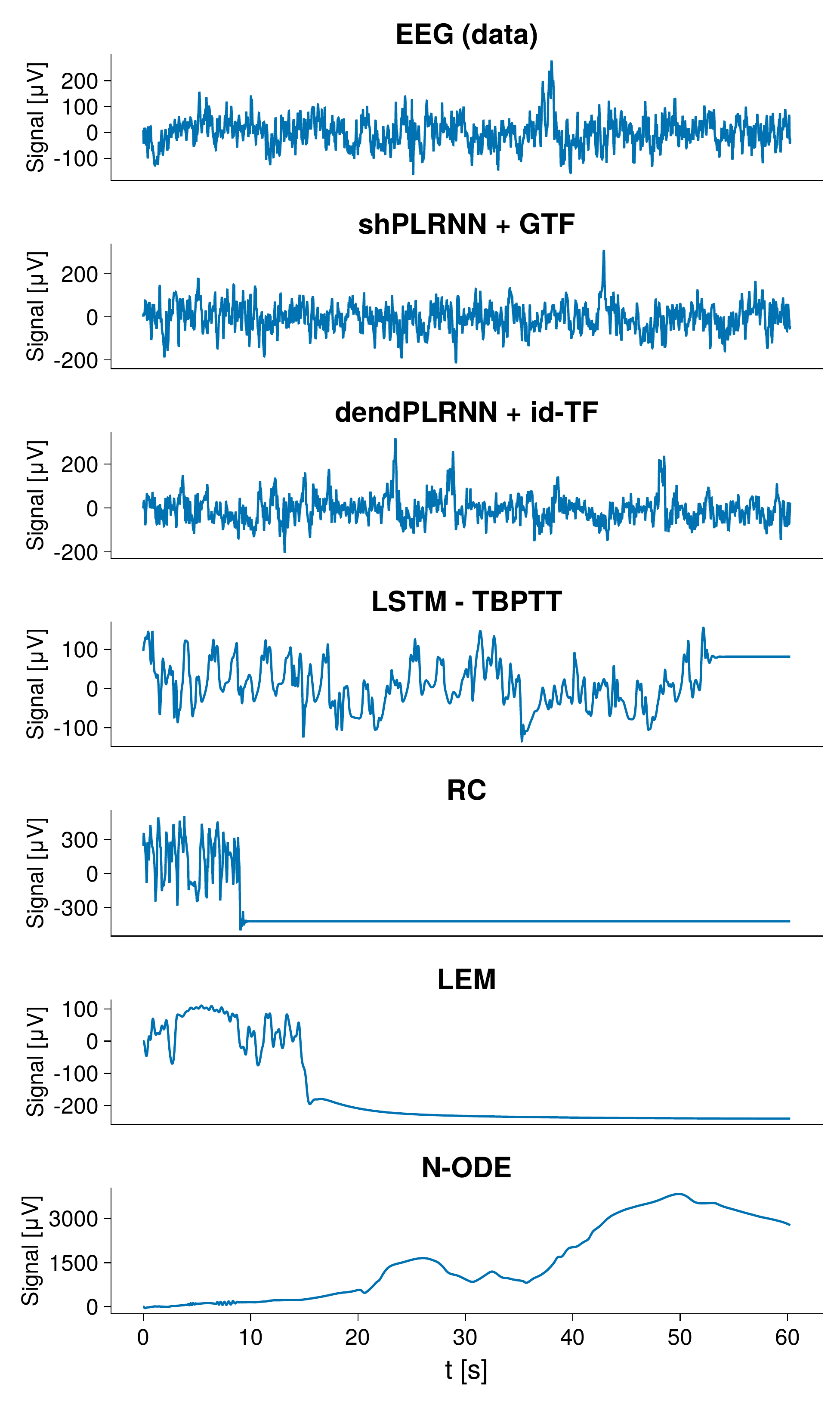}
\vspace{-.4cm}
    \caption{Example time traces of EEG reconstructions provided by  
    the methods employed in Table \ref{tab:SOTA_comparison}. For each method we took the best model out of $20$. For most runs ($\approx 90\%$), the shPLRNN trained with GTF and the dendPLRNN \cite{brenner22} provide similarly good results, while all other methods have highly variable outcomes and overall poor reconstructions of the long-term behavior.}
    \label{fig:eeg_reconstructions}
\end{figure}

This is particularly evident in Fig. \ref{fig:eeg_reconstructions} which provides examples of reconstructed 
time series on the EEG data for all of the compared methods, illustrating the trained models' \textit{long-term} behavior (see Fig. \ref{fig:eeg_reconstructions_heatmaps} for all EEG channels, and Fig. \ref{fig:ecg_reconstructions} for ECG data).  
For creating these graphs, each method was initialized  with a value inferred from the data, and then freely evolved forward 
for a long time (unconstrained by the data) from that one initial condition. As the underlying system is likely chaotic (according to an estimate of its maximal Lyapunov exponent and fractal dimensionality;  
\citet{Mikhaeil22}), in each case reliable forecasts can only be obtained for a couple of time steps ahead.  
Importantly, however, the  
temporal structure of the real data is conserved in the simulated data for the shPLRNN+GTF, and a bit less so for dendPLRNN, but for none of the other methods which often just converge to an equilibrium point after some time, or even diverge. Thus, other methods were not able to reconstruct the systems' attractors
(although they may still produce reasonable short-term predictions; 
cf. Tab. \ref{tab:SOTA_comparison} \& Fig. \ref{fig:eeg_forecast}). 
At the same time, shPLRNN+GTF requires only $16$ latent states to reconstruct the EEG, compared to $105$ for dendPLRNN (and $64$ for some of the others methods). 

Finally, replacing GTF by the previously proposed STF \cite{Mikhaeil22} in shPLRNN training  
led to 
worse performance on the reconstruction measures (Table \ref{tab:SOTA_comparison}, shPLRNN+STF), suggesting 
an important role for GTF.\footnote{As in \citet{Mikhaeil22}, the forcing interval was simply 
set according to the predictability time (ECG: $\tau_{pred} \approx 221$, EEG: $\tau_{pred} \approx 41$, see Appx. \ref{appx:ecg}) and not by systematic grid search.}

\section{Conclusions}
In this paper we develop and test the idea of GTF for learning chaotic DS. GTF addresses the central challenge posed in \citet{Mikhaeil22}, namely that exploding gradients cannot be avoided by architectural or parameter constraints in RNNs trained on chaotic systems, and hence should be targeted at the level of the training procedure itself. Although the basic idea of GTF has been introduced a long time ago \citep{doya1992bifurcations}, to our knowledge it never has been systematically empirically tested or theoretically examined. Here, for the first time, we show this training strategy can be amended such that exploding gradients can be completely avoided along theoretically infinite-length trajectories, resolving the training issue for chaotic systems. Another mechanism by which GTF may enhance training is by smoothing the loss landscape (Fig. \ref{fig:loss_landscape_smoothing}), as observed for other control and annealing methods \citep{abarbanel2013predicting}. 

Especially when paired with a simple modification of the PLRNN, the shPLRNN, GTF results in a powerful DS reconstruction algorithm which outperforms
a diversity of other 
models and techniques 
on the empirical datasets by sometimes large margins. At the same time,
it returns a dynamically interpretable DS model for which fixed points and cycles can be determined semi-analytically (cf. \ref{subsec:rec_modles}), and which achieves 
the lowest-dimensional reconstructions of all methods tested, requiring at most as many dimensions 
as the
observed system, potentially even less than empirically observed (Tab. \ref{tab:SOTA_comparison}). 
Although (or perhaps because) conceptually our advancements are simple, this sets a new standard in the field, especially for more relevant real-world data with which many algorithms profoundly struggle with. In fact, up to now most of the DS reconstruction literature has focused on simulated DS benchmarks only. 
While studying ground truth systems with known properties is clearly important, ultimately, of course, these algorithms are to be taken to the real world, where additional and partly unforeseeable issues may wait (some of them pointed out in sect. \ref{subsec:datasets}).
Chaotic dynamics are ubiquitous in natural and many human-made systems \citep{durstewitz2007dynamical, duarte_quantifying_2010, faggini_chaotic_2014, mangiarotti_chaos_2020,  kamdjeu_kengne_dynamics_2021}. In fact they are essentially the rule in any type of complex multi-component system with some heterogeneity in its elements or connections \cite{van_vreeswijk_chaos_1996}, including the brain \cite{durstewitz_dynamical_2007},  
climate systems \cite{bury_deep_2021}, or social and economical networks \cite{xu_chaotic_2011}. By strictly controlling the gradients in training, without imposing any other constraints,
GTF may prove valuable more generally for any time series prediction, regression or classification tasks.
While in principle applicable to any other RNN architecture, however, we point out that it is also the specific design of the shPLRNN which may make it particularly well suited for straightforward implementation of GTF. For instance, we were not able to obtain similar performance boosts through GTF for the dendPLRNN (at least in a naive implementation) which lacks the shPLRNN's 1:1 relation between observations and latent states (nor did we achieve similar performance with $\tanh$ rather than ReLU activation; see Appx. \ref{appx:ablation}). Thus, despite its general design, how to best utilize GTF in other model architectures, as well as how it compares to STF, will still require further research. 

All code is available at \url{https://github.com/DurstewitzLab/GTF-shPLRNN}.
\section*{Acknowledgements}
This work was funded by the German Research Foundation (DFG) within Germany’s Excellence Strategy EXC 2181/1 – 390900948 (STRUCTURES), by DFG grants Du354/10-1 \& Du354/15-1 to DD, and by the European Union Horizon-2020 consortium SC1-DTH-13-2020 (IMMERSE).


\bibliography{main}

\begin{thebibliography}{88}
\providecommand{\natexlab}[1]{#1}
\providecommand{\url}[1]{\texttt{#1}}
\expandafter\ifx\csname urlstyle\endcsname\relax
  \providecommand{\doi}[1]{doi: #1}\else
  \providecommand{\doi}{doi: \begingroup \urlstyle{rm}\Url}\fi

\bibitem[Abarbanel(2013)]{abarbanel2013predicting}
Abarbanel, H.
\newblock \emph{Predicting the future: completing models of observed complex
  systems}.
\newblock Springer, 2013.

\bibitem[Ardizzone et~al.(2019)Ardizzone, Kruse, Rother, and
  Köthe]{ardizzone2018analyzing}
Ardizzone, L., Kruse, J., Rother, C., and Köthe, U.
\newblock Analyzing inverse problems with invertible neural networks.
\newblock In \emph{International Conference on Learning Representations}, 2019.
\newblock URL \url{https://openreview.net/forum?id=rJed6j0cKX}.

\bibitem[Arjovsky et~al.(2016)Arjovsky, Shah, and Bengio]{arjovsky2016unitary}
Arjovsky, M., Shah, A., and Bengio, Y.
\newblock Unitary evolution recurrent neural networks.
\newblock In \emph{International conference on machine learning}, pp.\
  1120--1128. PMLR, 2016.

\bibitem[Bakarji et~al.(2022)Bakarji, Champion, Kutz, and
  Brunton]{bakarji2022discovering}
Bakarji, J., Champion, K., Kutz, J.~N., and Brunton, S.~L.
\newblock Discovering governing equations from partial measurements with deep
  delay autoencoders.
\newblock \emph{arXiv preprint arXiv:2201.05136}, 2022.

\bibitem[Bengio et~al.(2015)Bengio, Vinyals, Jaitly, and
  Shazeer]{bengio2015scheduled}
Bengio, S., Vinyals, O., Jaitly, N., and Shazeer, N.
\newblock Scheduled sampling for sequence prediction with recurrent neural
  networks.
\newblock \emph{Advances in neural information processing systems}, 28, 2015.

\bibitem[Bengio et~al.(1994)Bengio, Simard, and Frasconi]{bengio1994learning}
Bengio, Y., Simard, P., and Frasconi, P.
\newblock Learning long-term dependencies with gradient descent is difficult.
\newblock \emph{IEEE transactions on neural networks}, 5\penalty0 (2):\penalty0
  157--166, 1994.

\bibitem[B{\"o}ttcher \& Wenzel(2008)B{\"o}ttcher and
  Wenzel]{bottcher2008frobenius}
B{\"o}ttcher, A. and Wenzel, D.
\newblock The frobenius norm and the commutator.
\newblock \emph{Linear algebra and its applications}, 429\penalty0
  (8-9):\penalty0 1864--1885, 2008.

\bibitem[Brenner et~al.(2022)Brenner, Hess, Mikhaeil, Bereska, Monfared, Kuo,
  and Durstewitz]{brenner22}
Brenner, M., Hess, F., Mikhaeil, J.~M., Bereska, L.~F., Monfared, Z., Kuo,
  P.-C., and Durstewitz, D.
\newblock Tractable dendritic {RNN}s for reconstructing nonlinear dynamical
  systems.
\newblock In \emph{Proceedings of the 39th International Conference on Machine
  Learning}, volume 162, pp.\  2292--2320. PMLR, 2022.

\bibitem[Brunton et~al.(2016)Brunton, Proctor, and
  Kutz]{brunton2016discovering}
Brunton, S.~L., Proctor, J.~L., and Kutz, J.~N.
\newblock Discovering governing equations from data by sparse identification of
  nonlinear dynamical systems.
\newblock \emph{Proceedings of the national academy of sciences}, 113\penalty0
  (15):\penalty0 3932--3937, 2016.

\bibitem[Brunton et~al.(2022)Brunton, Budisic, Kaiser, and
  Kutz]{brunton2022modern}
Brunton, S.~L., Budisic, M., Kaiser, E., and Kutz, J.~N.
\newblock Modern koopman theory for dynamical systems.
\newblock \emph{SIAM Review}, 64\penalty0 (2):\penalty0 229--340, 2022.

\bibitem[Bury et~al.(2021)Bury, Sujith, Pavithran, Scheffer, Lenton, Anand, and
  Bauch]{bury_deep_2021}
Bury, T.~M., Sujith, R.~I., Pavithran, I., Scheffer, M., Lenton, T.~M., Anand,
  M., and Bauch, C.~T.
\newblock Deep learning for early warning signals of tipping points.
\newblock \emph{Proceedings of the National Academy of Sciences}, 118\penalty0
  (39):\penalty0 e2106140118, 2021.
\newblock \doi{10.1073/pnas.2106140118}.
\newblock URL \url{https://www.pnas.org/doi/10.1073/pnas.2106140118}.

\bibitem[Champion et~al.(2019)Champion, Lusch, Kutz, and
  Brunton]{champion2019data}
Champion, K., Lusch, B., Kutz, J.~N., and Brunton, S.~L.
\newblock Data-driven discovery of coordinates and governing equations.
\newblock \emph{Proceedings of the National Academy of Sciences}, 116\penalty0
  (45):\penalty0 22445--22451, 2019.

\bibitem[Chang et~al.(2019)Chang, Chen, Haber, and
  Chi]{chang2018antisymmetricrnn}
Chang, B., Chen, M., Haber, E., and Chi, E.~H.
\newblock Antisymmetric{RNN}: A dynamical system view on recurrent neural
  networks.
\newblock In \emph{International Conference on Learning Representations}, 2019.
\newblock URL \url{https://openreview.net/forum?id=ryxepo0cFX}.

\bibitem[Chattopadhyay et~al.(2020)Chattopadhyay, Hassanzadeh, and
  Subramanian]{chattopadhyay2020data}
Chattopadhyay, A., Hassanzadeh, P., and Subramanian, D.
\newblock Data-driven predictions of a multiscale lorenz 96 chaotic system
  using machine-learning methods: reservoir computing, artificial neural
  network, and long short-term memory network.
\newblock \emph{Nonlinear Processes in Geophysics}, 27\penalty0 (3):\penalty0
  373--389, 2020.

\bibitem[Chen et~al.(2018)Chen, Rubanova, Bettencourt, and
  Duvenaud]{chen2018neural}
Chen, R.~T., Rubanova, Y., Bettencourt, J., and Duvenaud, D.~K.
\newblock Neural ordinary differential equations.
\newblock \emph{Advances in neural information processing systems}, 31, 2018.

\bibitem[Cho et~al.(2014)Cho, Van~Merri{\"e}nboer, Bahdanau, and
  Bengio]{cho2014properties}
Cho, K., Van~Merri{\"e}nboer, B., Bahdanau, D., and Bengio, Y.
\newblock On the properties of neural machine translation: Encoder-decoder
  approaches.
\newblock \emph{arXiv preprint arXiv:1409.1259}, 2014.

\bibitem[Cooley \& Tukey(1965)Cooley and Tukey]{cooley_algorithm_1965}
Cooley, J.~W. and Tukey, J.~W.
\newblock An algorithm for the machine calculation of complex fourier series.
\newblock \emph{Mathematics of Computation}, 19\penalty0 (90):\penalty0
  297--301, 1965.
\newblock ISSN 0025-5718.
\newblock \doi{10.2307/2003354}.
\newblock URL \url{https://www.jstor.org/stable/2003354}.
\newblock Publisher: American Mathematical Society.

\bibitem[Datseris(2018)]{Datseris2018}
Datseris, G.
\newblock Dynamicalsystems.jl: A julia software library for chaos and nonlinear
  dynamics.
\newblock \emph{Journal of Open Source Software}, 3\penalty0 (23):\penalty0
  598, mar 2018.
\newblock \doi{10.21105/joss.00598}.
\newblock URL \url{https://doi.org/10.21105/joss.00598}.

\bibitem[de~Silva et~al.(2020)de~Silva, Champion, Quade, Loiseau, Kutz, and
  Brunton]{desilva2020}
de~Silva, B., Champion, K., Quade, M., Loiseau, J.-C., Kutz, J., and Brunton,
  S.
\newblock Pysindy: A python package for the sparse identification of nonlinear
  dynamical systems from data.
\newblock \emph{Journal of Open Source Software}, 5\penalty0 (49):\penalty0
  2104, 2020.
\newblock \doi{10.21105/joss.02104}.
\newblock URL \url{https://doi.org/10.21105/joss.02104}.

\bibitem[Dinh et~al.(2017)Dinh, Sohl-Dickstein, and Bengio]{dinh2017density}
Dinh, L., Sohl-Dickstein, J., and Bengio, S.
\newblock Density estimation using real {NVP}.
\newblock In \emph{International Conference on Learning Representations}, 2017.
\newblock URL \url{https://openreview.net/forum?id=HkpbnH9lx}.

\bibitem[Doya et~al.(1992)]{doya1992bifurcations}
Doya, K. et~al.
\newblock Bifurcations in the learning of recurrent neural networks 3.
\newblock \emph{learning (RTRL)}, 3:\penalty0 17, 1992.

\bibitem[Duarte et~al.(2010)Duarte, Januário, Martins, and
  Sardanyés]{duarte_quantifying_2010}
Duarte, J., Januário, C., Martins, N., and Sardanyés, J.
\newblock Quantifying chaos for ecological stoichiometry.
\newblock \emph{Chaos: An Interdisciplinary Journal of Nonlinear Science},
  20\penalty0 (3):\penalty0 033105, September 2010.

\bibitem[Durstewitz(2017)]{durstewitz2017state}
Durstewitz, D.
\newblock A state space approach for piecewise-linear recurrent neural networks
  for identifying computational dynamics from neural measurements.
\newblock \emph{PLoS computational biology}, 13\penalty0 (6):\penalty0
  e1005542, 2017.

\bibitem[Durstewitz \& Gabriel(2007{\natexlab{a}})Durstewitz and
  Gabriel]{durstewitz2007dynamical}
Durstewitz, D. and Gabriel, T.
\newblock Dynamical basis of irregular spiking in nmda-driven prefrontal cortex
  neurons.
\newblock \emph{Cerebral cortex}, 17\penalty0 (4):\penalty0 894--908,
  2007{\natexlab{a}}.

\bibitem[Durstewitz \& Gabriel(2007{\natexlab{b}})Durstewitz and
  Gabriel]{durstewitz_dynamical_2007}
Durstewitz, D. and Gabriel, T.
\newblock Dynamical basis of irregular spiking in {NMDA}-driven prefrontal
  cortex neurons.
\newblock \emph{Cerebral Cortex (New York, N.Y.: 1991)}, 17\penalty0
  (4):\penalty0 894--908, 2007{\natexlab{b}}.
\newblock ISSN 1047-3211.
\newblock \doi{10.1093/cercor/bhk044}.

\bibitem[Engelken et~al.(2020)Engelken, Wolf, and Abbott]{engelken2020lyapunov}
Engelken, R., Wolf, F., and Abbott, L.~F.
\newblock Lyapunov spectra of chaotic recurrent neural networks.
\newblock \emph{arXiv preprint arXiv:2006.02427}, 2020.

\bibitem[Erichson et~al.(2021)Erichson, Azencot, Queiruga, Hodgkinson, and
  Mahoney]{erichson2021lipschitz}
Erichson, N.~B., Azencot, O., Queiruga, A., Hodgkinson, L., and Mahoney, M.~W.
\newblock Lipschitz recurrent neural networks.
\newblock In \emph{International Conference on Learning Representations}, 2021.
\newblock URL \url{https://openreview.net/forum?id=-N7PBXqOUJZ}.

\bibitem[Faggini(2014)]{faggini_chaotic_2014}
Faggini, M.
\newblock Chaotic time series analysis in economics: {Balance} and
  perspectives.
\newblock \emph{Chaos: An Interdisciplinary Journal of Nonlinear Science},
  24\penalty0 (4):\penalty0 042101, 2014.
\newblock Publisher: American Institute of Physics.

\bibitem[Funahashi \& Nakamura(1993)Funahashi and
  Nakamura]{funahashi1993approximation}
Funahashi, K.-i. and Nakamura, Y.
\newblock Approximation of dynamical systems by continuous time recurrent
  neural networks.
\newblock \emph{Neural networks}, 6\penalty0 (6):\penalty0 801--806, 1993.

\bibitem[Goldberger et~al.(2000)Goldberger, Amaral, Glass, Hausdorff, Ivanov,
  Mark, Mietus, Moody, Peng, and Stanley]{goldberger2000physiobank}
Goldberger, A.~L., Amaral, L.~A., Glass, L., Hausdorff, J.~M., Ivanov, P.~C.,
  Mark, R.~G., Mietus, J.~E., Moody, G.~B., Peng, C.-K., and Stanley, H.~E.
\newblock Physiobank, physiotoolkit, and physionet: components of a new
  research resource for complex physiologic signals.
\newblock \emph{circulation}, 101\penalty0 (23):\penalty0 e215--e220, 2000.

\bibitem[Govindan et~al.(1998)Govindan, Narayanan, and
  Gopinathan]{govindan1998evidence}
Govindan, R., Narayanan, K., and Gopinathan, M.
\newblock On the evidence of deterministic chaos in ecg: Surrogate and
  predictability analysis.
\newblock \emph{Chaos: An Interdisciplinary Journal of Nonlinear Science},
  8\penalty0 (2):\penalty0 495--502, 1998.

\bibitem[Hanson \& Raginsky(2020)Hanson and Raginsky]{hanson2020universal}
Hanson, J. and Raginsky, M.
\newblock Universal simulation of stable dynamical systems by recurrent neural
  nets.
\newblock In \emph{Learning for Dynamics and Control}, pp.\  384--392. PMLR,
  2020.

\bibitem[Hernandez et~al.(2020)Hernandez, Moretti, Saxena, Wei, Cunningham, and
  Paninski]{hernandez2020nonlinear}
Hernandez, D., Moretti, A.~K., Saxena, S., Wei, Z., Cunningham, J., and
  Paninski, L.
\newblock Nonlinear evolution via spatially-dependent linear dynamics for
  electrophysiology and calcium data.
\newblock \emph{Neurons, Behavior, Data analysis, and Theory}, 3\penalty0
  (3):\penalty0 13476, 2020.

\bibitem[Hershey \& Olsen(2007)Hershey and Olsen]{hershey2007approximating}
Hershey, J.~R. and Olsen, P.~A.
\newblock Approximating the kullback leibler divergence between gaussian
  mixture models.
\newblock In \emph{2007 IEEE International Conference on Acoustics, Speech and
  Signal Processing-ICASSP'07}, volume~4, pp.\  IV--317. IEEE, 2007.

\bibitem[Higham \& Lin(2011)Higham and Lin]{higham2011}
Higham, N.~J. and Lin, L.
\newblock A schur-padé algorithm for fractional powers of a matrix.
\newblock \emph{SIAM Journal on Matrix Analysis and Applications}, 32\penalty0
  (3):\penalty0 1056--1078, 2011.

\bibitem[Hochreiter \& Schmidhuber(1997)Hochreiter and
  Schmidhuber]{hochreiter1997long}
Hochreiter, S. and Schmidhuber, J.
\newblock Long short-term memory.
\newblock \emph{Neural computation}, 9\penalty0 (8):\penalty0 1735--1780, 1997.

\bibitem[Hochreiter et~al.(2001)Hochreiter, Bengio, Frasconi, Schmidhuber,
  et~al.]{hochreiter2001gradient}
Hochreiter, S., Bengio, Y., Frasconi, P., Schmidhuber, J., et~al.
\newblock Gradient flow in recurrent nets: the difficulty of learning long-term
  dependencies, 2001.

\bibitem[Inoue et~al.(2021)Inoue, Ohara, Kuniyoshi, and
  Nakajima]{inoue_transient_2021}
Inoue, K., Ohara, S., Kuniyoshi, Y., and Nakajima, K.
\newblock Transient chaos in bert.
\newblock \emph{arXiv:2106.03181}, 2021.

\bibitem[Jordan(1990)]{jordan1990attractor}
Jordan, M.~I.
\newblock Attractor dynamics and parallelism in a connectionist sequential
  machine.
\newblock In \emph{Artificial neural networks: concept learning}, pp.\
  112--127. IEEE Press, 1990.

\bibitem[Kamdjeu~Kengne et~al.(2021)Kamdjeu~Kengne, Mboupda~Pone, and
  Fotsin]{kamdjeu_kengne_dynamics_2021}
Kamdjeu~Kengne, L., Mboupda~Pone, J.~R., and Fotsin, H.~B.
\newblock On the dynamics of chaotic circuits based on memristive diode-bridge
  with variable symmetry: {A} case study.
\newblock \emph{Chaos, Solitons \& Fractals}, 145:\penalty0 110795, 2021.

\bibitem[Kantz(1994)]{kantz1994robust}
Kantz, H.
\newblock A robust method to estimate the maximal lyapunov exponent of a time
  series.
\newblock \emph{Physics letters A}, 185\penalty0 (1):\penalty0 77--87, 1994.

\bibitem[Kaptanoglu et~al.(2022)Kaptanoglu, de~Silva, Fasel, Kaheman,
  Goldschmidt, Callaham, Delahunt, Nicolaou, Champion, Loiseau, Kutz, and
  Brunton]{Kaptanoglu2022}
Kaptanoglu, A.~A., de~Silva, B.~M., Fasel, U., Kaheman, K., Goldschmidt, A.~J.,
  Callaham, J., Delahunt, C.~B., Nicolaou, Z.~G., Champion, K., Loiseau, J.-C.,
  Kutz, J.~N., and Brunton, S.~L.
\newblock Pysindy: A comprehensive python package for robust sparse system
  identification.
\newblock \emph{Journal of Open Source Software}, 7\penalty0 (69):\penalty0
  3994, 2022.
\newblock \doi{10.21105/joss.03994}.
\newblock URL \url{https://doi.org/10.21105/joss.03994}.

\bibitem[Kesmia et~al.(2020)Kesmia, Boughaba, and Jacquir]{kesmia_control_2020}
Kesmia, M., Boughaba, S., and Jacquir, S.
\newblock Control of continuous dynamical systems modeling physiological
  states.
\newblock \emph{Chaos, Solitons \& Fractals}, 136:\penalty0 109805, 2020.

\bibitem[Kimura \& Nakano(1998)Kimura and Nakano]{kimura1998learning}
Kimura, M. and Nakano, R.
\newblock Learning dynamical systems by recurrent neural networks from orbits.
\newblock \emph{Neural Networks}, 11\penalty0 (9):\penalty0 1589--1599, 1998.

\bibitem[Kirchmeyer et~al.(2022)Kirchmeyer, Yin, Dona, Baskiotis,
  Rakotomamonjy, and Gallinari]{pmlr-v162-kirchmeyer22a}
Kirchmeyer, M., Yin, Y., Dona, J., Baskiotis, N., Rakotomamonjy, A., and
  Gallinari, P.
\newblock Generalizing to new physical systems via context-informed dynamics
  model.
\newblock In Chaudhuri, K., Jegelka, S., Song, L., Szepesvari, C., Niu, G., and
  Sabato, S. (eds.), \emph{Proceedings of the 39th International Conference on
  Machine Learning}, volume 162 of \emph{Proceedings of Machine Learning
  Research}, pp.\  11283--11301. PMLR, 17--23 Jul 2022.
\newblock URL \url{https://proceedings.mlr.press/v162/kirchmeyer22a.html}.

\bibitem[Koppe et~al.(2019)Koppe, Toutounji, Kirsch, Lis, and
  Durstewitz]{koppe2019identifying}
Koppe, G., Toutounji, H., Kirsch, P., Lis, S., and Durstewitz, D.
\newblock Identifying nonlinear dynamical systems via generative recurrent
  neural networks with applications to fmri.
\newblock \emph{PLoS computational biology}, 15\penalty0 (8):\penalty0
  e1007263, 2019.

\bibitem[Kramer et~al.(2022)Kramer, Bommer, Durstewitz, Tombolini, and
  Koppe]{kramer2022reconstructing}
Kramer, D., Bommer, P.~L., Durstewitz, D., Tombolini, C., and Koppe, G.
\newblock Reconstructing nonlinear dynamical systems from multi-modal time
  series.
\newblock In \emph{International Conference on Machine Learning}, pp.\
  11613--11633. PMLR, 2022.

\bibitem[Kr{\"a}mer et~al.(2021)Kr{\"a}mer, Datseris, Kurths, Kiss,
  Ocampo-Espindola, and Marwan]{kramer2021unified}
Kr{\"a}mer, K.-H., Datseris, G., Kurths, J., Kiss, I.~Z., Ocampo-Espindola,
  J.~L., and Marwan, N.
\newblock A unified and automated approach to attractor reconstruction.
\newblock \emph{New Journal of Physics}, 23\penalty0 (3):\penalty0 033017,
  2021.

\bibitem[Lejarza \& Baldea(2022)Lejarza and Baldea]{lejarza2022data}
Lejarza, F. and Baldea, M.
\newblock Data-driven discovery of the governing equations of dynamical systems
  via moving horizon optimization.
\newblock \emph{Scientific Reports}, 12\penalty0 (1):\penalty0 1--15, 2022.

\bibitem[Li et~al.(2021)Li, Kovachki, Azizzadenesheli, liu, Bhattacharya,
  Stuart, and Anandkumar]{li2021fourier}
Li, Z., Kovachki, N.~B., Azizzadenesheli, K., liu, B., Bhattacharya, K.,
  Stuart, A., and Anandkumar, A.
\newblock Fourier neural operator for parametric partial differential
  equations.
\newblock In \emph{International Conference on Learning Representations}, 2021.
\newblock URL \url{https://openreview.net/forum?id=c8P9NQVtmnO}.

\bibitem[Liu et~al.(2020)Liu, Jiang, He, Chen, Liu, Gao, and Han]{liu2019radam}
Liu, L., Jiang, H., He, P., Chen, W., Liu, X., Gao, J., and Han, J.
\newblock On the variance of the adaptive learning rate and beyond.
\newblock In \emph{Proceedings of the Eighth International Conference on
  Learning Representations (ICLR 2020)}, April 2020.

\bibitem[Lorenz(1963)]{lorenz1963deterministic}
Lorenz, E.~N.
\newblock Deterministic nonperiodic flow.
\newblock \emph{Journal of atmospheric sciences}, 20\penalty0 (2):\penalty0
  130--141, 1963.

\bibitem[Lorenz(1996)]{lorenz1996predictability}
Lorenz, E.~N.
\newblock Predictability: A problem partly solved.
\newblock In \emph{Proc. Seminar on predictability}, volume~1, 1996.

\bibitem[Lusch et~al.(2018)Lusch, Kutz, and Brunton]{lusch2018deep}
Lusch, B., Kutz, J.~N., and Brunton, S.~L.
\newblock Deep learning for universal linear embeddings of nonlinear dynamics.
\newblock \emph{Nature communications}, 9\penalty0 (1):\penalty0 1--10, 2018.

\bibitem[Mangiarotti et~al.(2020)Mangiarotti, Peyre, Zhang, Huc, Roger, and
  Kerr]{mangiarotti_chaos_2020}
Mangiarotti, S., Peyre, M., Zhang, Y., Huc, M., Roger, F., and Kerr, Y.
\newblock Chaos theory applied to the outbreak of {COVID}-19: an ancillary
  approach to decision making in pandemic context.
\newblock \emph{Epidemiology and Infection}, 148:\penalty0 e95, 2020.

\bibitem[Mikhaeil et~al.(2022)Mikhaeil, Monfared, and Durstewitz]{Mikhaeil22}
Mikhaeil, J.~M., Monfared, Z., and Durstewitz, D.
\newblock On the difficulty of learning chaotic dynamics with {RNNs}.
\newblock In \emph{Proceedings of the 36th Conference on Neural Information
  Processing Systems}. NeurIPS, 2022.

\bibitem[Monfared \& Durstewitz(2020)Monfared and
  Durstewitz]{monfared2020transformation}
Monfared, Z. and Durstewitz, D.
\newblock Transformation of relu-based recurrent neural networks from
  discrete-time to continuous-time.
\newblock In \emph{International Conference on Machine Learning}, pp.\
  6999--7009. PMLR, 2020.

\bibitem[Olsen et~al.(1988)Olsen, Truty, and Schaffer]{olsen_oscillations_1988}
Olsen, L.~F., Truty, G.~L., and Schaffer, W.~M.
\newblock Oscillations and chaos in epidemics: a nonlinear dynamic study of six
  childhood diseases in {Copenhagen}, {Denmark}.
\newblock \emph{Theoretical Population Biology}, 33\penalty0 (3):\penalty0
  344--370, 1988.

\bibitem[Otto \& Rowley(2019)Otto and Rowley]{otto_linearly_2019}
Otto, S.~E. and Rowley, C.~W.
\newblock Linearly recurrent autoencoder networks for learning dynamics.
\newblock \emph{{SIAM} Journal on Applied Dynamical Systems}, 18\penalty0
  (1):\penalty0 558--593, 2019.
\newblock \doi{10.1137/18M1177846}.
\newblock URL \url{https://epubs.siam.org/doi/abs/10.1137/18M1177846}.
\newblock Publisher: Society for Industrial and Applied Mathematics.

\bibitem[Patel \& Ott(2022)Patel and Ott]{patel2022using}
Patel, D. and Ott, E.
\newblock Using machine learning to anticipate tipping points and extrapolate
  to post-tipping dynamics of non-stationary dynamical systems.
\newblock \emph{arXiv preprint arXiv:2207.00521}, 2022.

\bibitem[Pathak et~al.(2018)Pathak, Hunt, Girvan, Lu, and Ott]{pathak2018model}
Pathak, J., Hunt, B., Girvan, M., Lu, Z., and Ott, E.
\newblock Model-free prediction of large spatiotemporally chaotic systems from
  data: A reservoir computing approach.
\newblock \emph{Physical review letters}, 120\penalty0 (2):\penalty0 024102,
  2018.

\bibitem[Pearlmutter(1989)]{pearlmutter1989learning}
Pearlmutter, B.~A.
\newblock Learning state space trajectories in recurrent neural networks.
\newblock \emph{Neural Computation}, 1\penalty0 (2):\penalty0 263--269, 1989.

\bibitem[Perko(2001)]{perko2001}
Perko, L.
\newblock \emph{Differential Equations and Dynamical Systems}, volume~7.
\newblock Springer, New York, NY, 2001.

\bibitem[Pineda(1988)]{pineda1988dynamics}
Pineda, F.~J.
\newblock Dynamics and architecture for neural computation.
\newblock \emph{Journal of Complexity}, 4\penalty0 (3):\penalty0 216--245,
  1988.

\bibitem[Raissi(2018)]{raissi2018deep}
Raissi, M.
\newblock Deep hidden physics models: Deep learning of nonlinear partial
  differential equations.
\newblock \emph{The Journal of Machine Learning Research}, 19\penalty0
  (1):\penalty0 932--955, 2018.

\bibitem[Reiss et~al.(2019)Reiss, Indlekofer, Schmidt, and
  Van~Laerhoven]{Reiss_2019}
Reiss, A., Indlekofer, I., Schmidt, P., and Van~Laerhoven, K.
\newblock Deep ppg: Large-scale heart rate estimation with convolutional neural
  networks.
\newblock \emph{Sensors}, 19\penalty0 (14), 2019.
\newblock ISSN 1424-8220.
\newblock \doi{10.3390/s19143079}.
\newblock URL \url{https://www.mdpi.com/1424-8220/19/14/3079}.

\bibitem[Ricci et~al.(2022)Ricci, Moriel, Piran, and
  Nitzan]{ricci2022phase2vec}
Ricci, M., Moriel, N., Piran, Z., and Nitzan, M.
\newblock Phase2vec: Dynamical systems embedding with a physics-informed
  convolutional network.
\newblock \emph{arXiv preprint arXiv:2212.03857}, 2022.

\bibitem[Rubanova et~al.(2019)Rubanova, Chen, and
  Duvenaud]{rubanova_latent_2019}
Rubanova, Y., Chen, R.~T., and Duvenaud, D.~K.
\newblock Latent ordinary differential equations for irregularly-sampled time
  series.
\newblock \emph{Advances in neural information processing systems}, 32, 2019.

\bibitem[Rumelhart et~al.(1986)Rumelhart, Hinton, and
  Williams]{rumelhart1986learning}
Rumelhart, D.~E., Hinton, G.~E., and Williams, R.~J.
\newblock Learning representations by back-propagating errors.
\newblock \emph{nature}, 323\penalty0 (6088):\penalty0 533--536, 1986.

\bibitem[Rusch \& Mishra(2020)Rusch and Mishra]{rusch2020coupled}
Rusch, T.~K. and Mishra, S.
\newblock Coupled oscillatory recurrent neural network (cornn): An accurate and
  (gradient) stable architecture for learning long time dependencies.
\newblock \emph{arXiv preprint arXiv:2010.00951}, 2020.

\bibitem[Rusch \& Mishra(2021)Rusch and Mishra]{rusch2021coupled}
Rusch, T.~K. and Mishra, S.
\newblock Coupled oscillatory recurrent neural network (co{\{}rnn{\}}): An
  accurate and (gradient) stable architecture for learning long time
  dependencies.
\newblock In \emph{International Conference on Learning Representations}, 2021.
\newblock URL \url{https://openreview.net/forum?id=F3s69XzWOia}.

\bibitem[Rusch et~al.(2022)Rusch, Mishra, Erichson, and Mahoney]{rusch2022lem}
Rusch, T.~K., Mishra, S., Erichson, N.~B., and Mahoney, M.~W.
\newblock Long expressive memory for sequence modeling.
\newblock In \emph{International Conference on Learning Representations}, 2022.

\bibitem[Sauer et~al.(1991)Sauer, Yorke, and Casdagli]{sauer1991embedology}
Sauer, T., Yorke, J.~A., and Casdagli, M.
\newblock Embedology.
\newblock \emph{Journal of statistical Physics}, 65\penalty0 (3):\penalty0
  579--616, 1991.

\bibitem[Schalk et~al.(2000)Schalk, {McFarland}, Hinterberger, Birbaumer, and
  Wolpaw]{schalk_bci2000_2004}
Schalk, G., {McFarland}, D.~J., Hinterberger, T., Birbaumer, N., and Wolpaw,
  J.~R.
\newblock {BCI}2000: a general-purpose brain-computer interface ({BCI}) system.
\newblock \emph{{IEEE} transactions on bio-medical engineering}, 51\penalty0
  (6):\penalty0 1034--1043, 2000.
\newblock ISSN 0018-9294.
\newblock \doi{10.1109/TBME.2004.827072}.

\bibitem[Schmidt et~al.(2021)Schmidt, Koppe, Monfared, Beutelspacher, and
  Durstewitz]{schmidt2021identifying}
Schmidt, D., Koppe, G., Monfared, Z., Beutelspacher, M., and Durstewitz, D.
\newblock Identifying nonlinear dynamical systems with multiple time scales and
  long-range dependencies.
\newblock In \emph{International Conference on Learning Representations}, 2021.
\newblock URL \url{https://openreview.net/forum?id=_XYzwxPIQu6}.

\bibitem[Skokos et~al.(2016)Skokos, Gottwald, and Laskar]{skokos2016chaos}
Skokos, C.~H., Gottwald, G.~A., and Laskar, J.
\newblock \emph{Chaos Detection and Predictability}, volume~1.
\newblock Springer, 2016.

\bibitem[Smith et~al.(2021)Smith, Linderman, and Sussillo]{smith2021reverse}
Smith, J., Linderman, S., and Sussillo, D.
\newblock Reverse engineering recurrent neural networks with jacobian switching
  linear dynamical systems.
\newblock \emph{Advances in Neural Information Processing Systems},
  34:\penalty0 16700--16713, 2021.

\bibitem[Takens(1981)]{takens1981detecting}
Takens, F.
\newblock Detecting strange attractors in turbulence.
\newblock In \emph{Dynamical systems and turbulence, Warwick 1980}, pp.\
  366--381. Springer, 1981.

\bibitem[Thornes et~al.(2017)Thornes, D{\"u}ben, and Palmer]{thornes2017use}
Thornes, T., D{\"u}ben, P., and Palmer, T.
\newblock On the use of scale-dependent precision in earth system modelling.
\newblock \emph{Quarterly Journal of the Royal Meteorological Society},
  143\penalty0 (703):\penalty0 897--908, 2017.

\bibitem[Van~Vreeswijk \& Sompolinsky(1996)Van~Vreeswijk and
  Sompolinsky]{van_vreeswijk_chaos_1996}
Van~Vreeswijk, C. and Sompolinsky, H.
\newblock Chaos in neuronal networks with balanced excitatory and inhibitory
  activity.
\newblock \emph{Science (New York, N.Y.)}, 274\penalty0 (5293):\penalty0
  1724--1726, 1996.
\newblock ISSN 0036-8075.
\newblock \doi{10.1126/science.274.5293.1724}.

\bibitem[Vlachas \& Koumoutsakos(2023)Vlachas and
  Koumoutsakos]{vlachas2023learning}
Vlachas, P.~R. and Koumoutsakos, P.
\newblock Learning from predictions: Fusing training and autoregressive
  inference for long-term spatiotemporal forecasts.
\newblock \emph{arXiv preprint arXiv:2302.11101}, 2023.

\bibitem[Vlachas et~al.(2018)Vlachas, Byeon, Wan, Sapsis, and
  Koumoutsakos]{vlachas2018data}
Vlachas, P.~R., Byeon, W., Wan, Z.~Y., Sapsis, T.~P., and Koumoutsakos, P.
\newblock Data-driven forecasting of high-dimensional chaotic systems with long
  short-term memory networks.
\newblock \emph{Proceedings of the Royal Society A: Mathematical, Physical and
  Engineering Sciences}, 474\penalty0 (2213):\penalty0 20170844, 2018.

\bibitem[Vlachas et~al.(2020)Vlachas, Pathak, Hunt, Sapsis, Girvan, Ott, and
  Koumoutsakos]{vlachas2020backpropagation}
Vlachas, P.~R., Pathak, J., Hunt, B.~R., Sapsis, T.~P., Girvan, M., Ott, E.,
  and Koumoutsakos, P.
\newblock Backpropagation algorithms and reservoir computing in recurrent
  neural networks for the forecasting of complex spatiotemporal dynamics.
\newblock \emph{Neural Networks}, 126:\penalty0 191--217, 2020.

\bibitem[Voss et~al.(2004)Voss, Timmer, and Kurths]{voss2004nonlinear}
Voss, H.~U., Timmer, J., and Kurths, J.
\newblock Nonlinear dynamical system identification from uncertain and indirect
  measurements.
\newblock \emph{International Journal of Bifurcation and Chaos}, 14\penalty0
  (06):\penalty0 1905--1933, 2004.

\bibitem[Werbos(1990)]{werbos1990backpropagation}
Werbos, P.~J.
\newblock Backpropagation through time: what it does and how to do it.
\newblock \emph{Proceedings of the IEEE}, 78\penalty0 (10):\penalty0
  1550--1560, 1990.

\bibitem[Williams \& Zipser(1989)Williams and Zipser]{williams1989learning}
Williams, R.~J. and Zipser, D.
\newblock A learning algorithm for continually running fully recurrent neural
  networks.
\newblock \emph{Neural computation}, 1\penalty0 (2):\penalty0 270--280, 1989.

\bibitem[Xu et~al.(2011)Xu, Chen, Si, Hu, Jiang, and Xu]{xu_chaotic_2011}
Xu, X., Chen, Z., Si, G., Hu, X., Jiang, Y., and Xu, X.
\newblock The chaotic dynamics of the social behavior selection networks in
  crowd simulation.
\newblock \emph{Nonlinear Dynamics}, 64\penalty0 (1):\penalty0 117--126, 2011.
\newblock ISSN 1573-269X.
\newblock \doi{10.1007/s11071-010-9850-z}.
\newblock URL \url{https://doi.org/10.1007/s11071-010-9850-z}.

\bibitem[Zhao \& Park(2020)Zhao and Park]{zhao2020variational}
Zhao, Y. and Park, I.~M.
\newblock Variational online learning of neural dynamics.
\newblock \emph{Frontiers in computational neuroscience}, 14:\penalty0 71,
  2020.

\end{thebibliography}
\bibliographystyle{icml2023}


\beginsupplement
\newpage
\onecolumn
\section{Appendix}\label{app}
\subsection{Theorems: Proofs}\label{app:proofs}
\subsubsection{Proof of proposition \ref{prop:chaos}}\label{app:proof_prop_chaos}
\begin{proof}
Let 
$\, \mathcal{J} \, = \, \{ \tilde{\mJ}_{\kappa} \}_{\kappa \in \mathcal{K}}\, $ be the set of all Jacobians of the RNN \eqref{eq:RNN_DS}, and $\mathcal{O}_{\vz_1}=\{ \vz_T \}_{T=1}^{\infty}\,$ be a chaotic orbit of the system. Then, the largest Lyapunov exponent of $\mathcal{O}_{\vz_1}$ is positive, i.e.
\begin{align}\label{LE}
 \lambda\, = \, \lim_{T\rightarrow\infty} \frac{1}{T} \, \log \norm{ \tilde{\mJ}_{T}\, \tilde{\mJ}_{T-1}\, \cdots \,  \tilde{\mJ}_{2}} \, > \, 0,    
\end{align}
which implies
\begin{align}\label{yaein}
 \lim_{T\rightarrow\infty}  \norm{ \tilde{\mJ}_{T}\, \tilde{\mJ}_{T-1}\, \cdots \,  \tilde{\mJ}_{2}} \, = \, \lim_{T \to \infty} \norm{ D\mF_\vtheta^T(\vz_1)}\, = \,  \infty. 
\end{align}
Accordingly, 
\begin{align}\label{eq:>1}
\exists \, \hat{n}
\in \mathbb{N} \hspace{.5cm} s.t. \hspace{.5cm} \forall m \geq \hat{n} 
\hspace{.8cm} \norm{ D\mF_\vtheta^m(\vz_1)} \, > \, 1.
\end{align}
Therefore $\norm{\tilde{\mJ}_s} \, > \,1 \,$ for some $s \in \{2, 3, \cdots, m \} \,$ $ (m \geq \hat{n})$. Otherwise
\begin{align}
 \norm{ D\mF_\vtheta^m(\vz_1)} \, = \, \norm{ \tilde{\mJ}_{m}\, \tilde{\mJ}_{m-1}\, \cdots \,  \tilde{\mJ}_{2}}\,\leq \, \norm{\tilde{\mJ}_{m}} \norm{ \tilde{\mJ}_{m-1}}  \cdots \norm{\tilde{\mJ}_{2}} \, \leq \, 1,
\end{align}
which is in contradiction to eq. \eqref{eq:>1}. Since $\tilde{\mJ}_s \in \mathcal{J}$, so 
$ \, \tilde{\sigma}_{\max} =  \sup  \bigg\{ \norm{\tilde{\mJ}_{k}}=
\sigma_{\max}(\tilde{\mJ}_{k})\, : \, \tilde{\mJ}_{k} \in \mathcal{J} \bigg\} >1$.
\end{proof}
%
\subsubsection{Proof of proposition \ref{prop:2}}\label{app:proof_prop2}
\begin{proof}
$(i)\,$ According to \eqref{eq:jac_chain_GTF} we have 
\begin{align}
\norm{\frac{\partial{\vz_t}}{\partial{\vz_r}}}_2 \, = \, \norm{(1-\alpha)^{t-r} \prod_{k=0}^{t-r-1}\tilde{\mJ}_{t-k}}_2 \, \leq \, (1-\alpha)^{t-r} \prod_{k=0}^{t-r-1} \norm{\tilde{\mJ}_{t-k}}_2 \, \leq \, \big[(1-\alpha)\tilde{\sigma}_{\max} \big]^{t-r},    
\end{align}
and 
\begin{align}\nonumber
\big[(1-\alpha)\tilde{\lambda}_{\min} \big]^{t-r} &\, \leq \, (1-\alpha)^{t-r} \prod_{k=0}^{t-r-1}  \lambda_{\min}(\tilde{\mJ}_{t-k}) \leq  (1-\alpha)^{t-r} \, \lambda_{\min}(\prod_{k=0}^{t-r-1}\tilde{\mJ}_{t-k})
\\[1ex]\label{}
& \, \leq \, (1-\alpha)^{t-r} \, \rho( \prod_{k=0}^{t-r-1}\tilde{\mJ}_{t-k}) \leq  \norm{(1-\alpha)^{t-r} \prod_{k=0}^{t-r-1}\tilde{\mJ}_{t-k}}_2 = \norm{\frac{\partial{\vz_t}}{\partial{\vz_r}}}_2. 
\end{align}
Therefore, 
\begin{align}\label{eq-sig-lamb}
\big[(1-\alpha)\tilde{\lambda}_{\min} \big]^{t-r} \, \leq \, \norm{\frac{\partial{\vz_t}}{\partial{\vz_r}}}_2 \, \leq \, \big[(1-\alpha)\tilde{\sigma}_{\max} \big]^{t-r}.    
\end{align}
Inserting $ \, \alpha = \alpha^{*} \, = \, 1- \frac{1}{\tilde{\sigma}_{\max}}$ (for $\, \tilde{\sigma}_{\max} \geq 1$) into the r.h.s. of \eqref{eq-sig-lamb} gives 
\begin{align}
\lim_{t \to \infty} \norm{\frac{\partial{\vz_t}}{\partial{\vz_r}}}_2 \, \leq \, \lim_{t \to \infty} \big[(\frac{1}{\tilde{\sigma}_{\max}})\tilde{\sigma}_{\max} \big]^{t-r} \, = \, 1,  
\end{align}
and so $\frac{\partial{\vz_t}}{\partial{\vz_r}}$ will not diverge for $t \to \infty$. 
\\[1ex]
$(ii) \, $ If $\tilde{\gamma} \, = \, 1$, then $ \tilde{\sigma}_{\max} = \tilde{\lambda}_{\min} \geq 1$. Hence, substituting $ \, \alpha^{*} \, = \, 1- \frac{1}{\tilde{\sigma}_{\max}} \, = \, 1- \frac{1}{\tilde{\lambda}_{\min}}$ in \eqref{eq-sig-lamb} yields
\begin{align}
\lim_{t \to \infty} \norm{\frac{\partial{\vz_t}}{\partial{\vz_r}}}_2 =1.
\end{align}
For $\tilde{\gamma} \, \neq \, 1$, inserting 
$\alpha^{*} \, = \, 1- \frac{1}{\tilde{\sigma}_{\max}} $ in \eqref{eq-sig-lamb} results in
\begin{align}\label{eq-sig-lamb-2}
0= \lim_{t \to \infty} \big( \tilde{\gamma}\big)^{t-r} \, \leq \, \lim_{t \to \infty} \norm{\frac{\partial{\vz_t}}{\partial{\vz_r}}}_2 \, \leq \, 1,    
\end{align}
and so $\, \lim_{t \to \infty} \norm{\frac{\partial{\vz_t}}{\partial{\vz_r}}}_2 $ may go to zero for $t \to \infty$. Moreover, obviously, the closer $\tilde{\gamma}$ is to $1$, the slower $\frac{\partial{\vz_t}}{\partial{\vz_r}}$ may vanish as $t \to \infty$.
\end{proof}
~\\
\begin{remark}\label{rem:Lip:RNN}
 According to Rademacher’s theorem, for Lipschitz-continuous RNNs $F_{\boldsymbol\theta}$ is differentiable almost everywhere and has a bounded derivative. Therefore, for such RNNs, the nonempty set $ \mathcal{S}_1$ (defined in \eqref{eq:sup:inf}) is bounded from above and so always has a supremum in $\mathbb{R}$.
\end{remark}
~\\
\begin{remark}
Assume that 
\begin{align}
\tilde{\sigma}_{\min} \, = \, \inf \bigg\{ \sigma_{\min}(\tilde{\mJ}_k) \, : \, \tilde{\mJ}_k \in \mathcal{J} \bigg\} \geq 0.    
\end{align}
Since $\tilde{\lambda}_{\min} \, \geq \, \tilde{\sigma}_{\min} \geq 0, \,$ from \eqref{eq-sig-lamb} we have
\begin{align}\label{}
\big[(1-\alpha)\tilde{\sigma}_{\min} \big]^{t-r} \, \leq \,\big[(1-\alpha)\tilde{\lambda}_{\min} \big]^{t-r} \, \leq \, \norm{\frac{\partial{\vz_t}}{\partial{\vz_r}}}_2 \, \leq \, \big[(1-\alpha)\tilde{\sigma}_{\max} \big]^{t-r}.    
\end{align}
\end{remark}
\subsubsection{Proof of proposition \ref{prop:sh}}\label{app:proof_prop_sh}
\begin{proof}
First, the dendPLRNN \eqref{eq:dendPLRNN} can be rewritten in the following form:
\begin{align}\label{eq:dendPLRNN_rewritten}
 \vz_t \, = \,  \mW^{B}_{\Omega(t-1)} \, \vz_{t-1} +  \mW \, \vh^{B}_{\Omega(t-1)} +\vh_0,
\end{align}
in which 
\begin{align}\nonumber
 \mD^{B}_{\Omega(t-1)}&:=  \sum_{b=1}^{B} \alpha_{b} \, \mD^{(b)}_{\Omega(t-1)} , \hspace{.2cm}
\\\nonumber
\vh^{B}_{\Omega(t-1)}&:=\sum_{b=1}^{B} \alpha_{b} 
\, \mD^{(b)}_{\Omega(t-1)} (-\vh_{b}), \hspace{.2cm}
\\\label{defs}
\mW^{B}_{\Omega(t-1)} 
&:= \mA  + \mW \, \mD^{B}_{\Omega(t-1)},
\end{align}
and $\mD^{(b)}_{\Omega(t-1)}= \text{diag} \big(d^{(b)}_{1,t-1}, d^{(b)}_{2,t-1}, \cdots, d^{(b)}_{M,t-1} \big)$ are diagonal binary indicator matrices with $d^{(b)}_{m,t-1}=1$ if $z_{m,t-1}> h_{m,b}$ and $0$ otherwise. Similarly, the
shPLRNN \eqref{eq:shPLRNN} can be brought into the form
\begin{align}\nonumber
    \vz_t & = \Big(\mA + \mW_1 \, \tilde{\mD}_{\Omega(t-1)} \mW_2 \Big) \vz_{t-1} + \mW_1  \tilde{\mD}_{\Omega(t-1)} \vh_2 + \vh_1 
    %
    \\ \label{eq:shPLRNN_re}
    & =:  \tilde{\mW}_{\Omega(t-1)} \vz_{t-1} +  \mW_1 \tilde{\vh}_{\Omega(t-1)} +\vh_1,
\end{align}
where $\tilde{\mD}_{\Omega(t-1)} = \text{diag} \big(\tilde{d}_{1,t-1}, \tilde{d}_{2,t-1}, \cdots, \tilde{d}_{L,t-1} \big)$ denotes an $L \times L$ diagonal binary indicator matrix with $d_{l,t-1}=1$ if $\sum_{j=1}^{M} w^{(2)}_{lj} \,z_{j,t-1}> -h^{(2)}_l$ and $0$ otherwise, where we used the notation $\mW_2 = \big[w^{(2)}_{ij}\big]$ and $\vh_2 = \big[h^{(2)}_i\big]$. In other words, it can be rewritten as the dendPLRNN \eqref{eq:dendPLRNN_rewritten}. Consequently, fixed points of \eqref{eq:shPLRNN} can be computed analogously to dendPLRNNs as 
\begin{align}\label{s-zm-fp1}
  \vz^{*1} = \Big(\mI-  \tilde{\mW}_{\Omega(t^{*1})} \Big)^{-1} \Big[\mW_1 \, \tilde{\vh}_{\Omega(t^{*1})} +\vh_1\Big] ,
\end{align}
where $\vz^{*1}=\vz_{t^{*1}}=\vz_{t^{*1}-1}$, and $\det(\mI-  \tilde{\mW}_{\Omega(t^{*1})}) = P_{ \tilde{\mW}_{\Omega(t^{*1})}}(1) \neq 0$, i.e. $ \tilde{\mW}_{\Omega(t^{*1})}$ has no eigenvalue equal to $1$ (otherwise we are dealing 
with a 
bifurcation or with a continuous set of marginally stable points). The same holds for cycles of \eqref{eq:shPLRNN}: Letting $t+n-1=: t^{*n}$, the periodic point $\vz^{*n}$ of an $n$-cycle is
\begingroup
\allowdisplaybreaks
\begin{align}\nonumber
 &\vz^{*n} 
%
%
\, = \,  \bigg(\mI- \prod_{i=1}^{n} \tilde{\mW}_{\Omega(t^{*n}-i)} \bigg)^{-1}
%
%
\bigg(\sum_{j=2}^{n} \Big[\prod_{i=1}^{n-j+1} \tilde{\mW}_{\Omega(t^{*n}-i)} \mW_1 \, \tilde{\vh}_{\Omega(t^{*n}-n +j-2)} \Big]\,
\\[1ex]\label{cycles}
&  \hspace{.9cm}
+\, \mW_1 \, \tilde{\vh}_{\Omega(t^{*n}-1)} 
 \, + \, \Big(\sum_{j=2}^{n} \prod_{i=1}^{n-j+1} \tilde{\mW}_{\Omega(t^{*n}-i)} + \mI \Big) \vh_1 \bigg),
\end{align}
\endgroup
where $\vz^{*n}=\vz_{t^{*n}}=\vz_{t^{*n}-n}$, if
$(\mI- \prod_{i=1}^{n} \tilde{\mW}_{\Omega(t^{*n}-i)})$ 
is invertible, i.e. 
$$\det \bigg(\mI-\prod_{i=1}^{n} \tilde{\mW}_{\Omega(t^{*n}-i)} \bigg)= P_{\prod_{i=1}^{n} \tilde{\mW}_{\Omega(t^{*n}-i)}}(1) \neq 0.$$
~\\
Now consider an $M$-dimensional dendPLRNN given by \eqref{eq:dendPLRNN} with $B$ bases
and $L = M \cdot B$. Furthermore, set $\widetilde{\bm{A}} := \bm{A}$, $\widetilde{\bm{h}}_1 := \bm{h}_0$, and
\begin{align}
    \widetilde{\bm{W}}_1 &:= \begin{bmatrix}
    \bm{W}\alpha_1 & \bm{W}\alpha_2 & \dots & \bm{W}\alpha_B
    \end{bmatrix} \in \mathbb{R}^{M \times L}, \\
    \widetilde{\bm{W}}_2 &:= \begin{bmatrix} 
    \mathbb{1}_{M \times M} \\
    \mathbb{1}_{M \times M} \\
    \vdots \\
    \mathbb{1}_{M \times M}
    \end{bmatrix} \in \mathbb{R}^{L \times M}, \hspace{.9cm}
    \widetilde{\bm{h}}_2 := - \begin{bmatrix} 
    \bm{h}_1 \\
    \bm{h}_2 \\
    \vdots \\
    \bm{h}_B
    \end{bmatrix} \in \mathbb{R}^{L}.
\end{align}
Then it follows that eq. \eqref{eq:dendPLRNN} can be written as an $M$-dimensional shPLRNN with hidden layer size $L$, 
\begin{align}
  \vz_t = \widetilde{\bm{A}}\bm{z}_{t-1} + \widetilde{\bm{W}}_1\phi\left(\widetilde{\bm{W}}_2\bm{z}_{t-1} + \widetilde{\bm{h}}_2\right) + \widetilde{\bm{h}}_1,
\end{align}
which completes the proof.
\end{proof}

\subsubsection{Proof of bounded orbits of \eqref{eq:bounded_shPLRNN}}\label{app:proof_boundedness}
\begin{proposition}
If $\, \rho(\mA) \, = \, \norm{\mA} < 1$, then every orbit of the clipped shPLRNN
\eqref{eq:bounded_shPLRNN} will be bounded.
\end{proposition}
\begin{proof}
Consider the clipped shPLRNN of the form 
\begin{align}\nonumber
     \vz_t  &\, = \,  \mA \vz_{t-1} + \mW_1 \big[\phi(\mW_2 \vz_{t-1} + \vh_2)  
   - \phi\left(\mW_2 \vz_{t-1}\right)\big] + \vh_1
   \\[1ex]\label{eq:psi}
   &\, := \, \mA \, \vz_{t-1} \, + \, \mW_1 \, \psi(\vz_{t-1}) \, +\, \vh_1.
\end{align}
Obviously, for $\, \vh_2 = \big[h^{(2)}_l\big]$ and every $l \in \{1,2,\cdots, L \}$
\begin{align}
 \psi_l(\vz_{t-1})  \, = \,  \max\big(0, \sum_{j=1}^{M} w^{(2)}_{lj} \,z_{j,t-1}+h^{(2)}_l \big)-\max \big(0, \sum_{j=1}^{M} w^{(2)}_{lj} \,z_{j,t-1} \big) 
\, \leq \, \big|h^{(2)}_l \big|.
\end{align}
This implies
\begin{align}\label{eq:hmax}
\norm{\psi(\vz_{t-1}) }\, = \,\sqrt{\sum_{l=1}^{L} \big(\psi_l(z_{t-1})\big)^2} \, \leq \, \sqrt{L} \, h_{\max}  \, := \, \bar{h}_{\max}, 
\end{align}
where $\, h_{\max} \,=\, \max\limits_{1 \leq l \leq L} \big\{\big|h^{(2)}_l \big| \big\}\,$.  \\

Let $\{ \vz_1,\vz_2, \cdots,\vz_t, \cdots \}$ be an orbit of \eqref{eq:psi}. For $T \in \mathbb{N}$, recursively computing 
$\vz_2,\vz_3, \cdots,\vz_T $ yields
\begin{align}\nonumber
&\vz_2\, = \, \mA\, \vz_1 \, + \, \mW_1  \, \psi(\vz_1) \, + \,\vh_1
\\[1ex]\nonumber
&\vz_3\, = \, \mA^2\, \vz_1 \, + \, \mA\,  \mW_1  \,\psi(\vz_1) + \mW_1 \,\psi(\vz_2)  \, + \,\big[\mA+\mI \big] \vh_1
\\\nonumber
\vdots
\\\label{}
& \vz_{T}\, = \, \mA^{T-1}\, \vz_1 \, + \, \sum_{j=0}^{T-2}\mA^j \,  \mW_1 \,\psi(\vz_{T-1-j}) +\sum_{j=0}^{T-2} \mA^j \, \vh_1. 
\end{align}
Hence, due to \eqref{eq:hmax}, one concludes that 
\begin{align}
\norm{\vz_{T}}\, \leq \, \norm{\mA}^{T-1}\, \norm{\vz_1} \, + \, \bar{h}_{\max}\, \norm{\mW_1} \sum_{j=0}^{T-2} \norm{\mA}^j  +\sum_{j=0}^{T-2} \norm{\mA}^j \, \norm{\vh_1}. 
\end{align}
If $\, \norm{\mA} < 1$, then $\displaystyle{\lim_{T \to \infty}\norm{\mA}^{T-1}} \, = \, 0 \,$, and so
\begin{align}\label{}
 \displaystyle{\lim_{T \to \infty}\norm{ \vz_{T}}} \, \leq \,  \, \bar{h}_{\max} \, \norm{\mW_1} \sum_{j=0}^{\infty} \norm{\mA}^j  +\sum_{j=0}^{\infty} \norm{\mA}^j \, \norm{\vh_1} \, = \, \frac{\bar{h}_{\max}\, \norm{\mW_1}+ \norm{\vh_1}}{1-\norm{\mA}} \, < \, \infty,
\end{align}
 which completes the proof.
\end{proof}
\subsection{Training Protocol}
\paragraph{General procedure}\label{appx:general_training} 
Given a times series $\bm{x}_{1:T}$, we train the model using BPTT with GTF in the following manner: Per epoch, we sample sub-sequences of length $\tilde{T}$ from the observed orbit, $\bm{\tilde{x}}^{(p)}_{1:\Tilde{T}} := \bm{x}_{t_p:t_p+\tilde{T}}$, where $t_p \in [1, T-\Tilde{T}]$ is drawn at random. We arrange multiple sequences in a batch $\left\{\bm{\tilde{x}}^{(p)}_{1:\Tilde{T}}\right\}_{p=1}^S$, where $S$ denotes the batch size. For a single parameter update, we first estimate forcing signals for the entire batch 
using Eq. \eqref{eq:invTF} applied to each time step in each sequence, i.e. $\hat{\bm{z}}^{(p)}_t = \bm{G}^{-1}_{\bm{\varphi}}(\tilde{\bm{x}}^{(p)}_t)$. We then take the estimated teacher signal for the first time step of each batch, $\hat{\bm{z}}^{(p)}_1$, as the initial condition for the RNN, which we propagate forward in time according to $\bm{z}_t = \bm{F}_{\bm{\theta}}(\bm{\tilde{z}}_{t-1})$ (see sect. \ref{subsec:rec_modles}) to produce predictions $\bm{z}^{(p)}_{2:\Tilde{T}}$. These are mapped back into observation space via $\bm{\hat{\tilde{x}}}^{(p)}_t = \bm{G}_{\bm{\varphi}}(\bm{z}^{(p)}_t)$. The MSE between ground truth and predicted orbits is then minimized according to
\begin{align}
    \mathcal{L}_\textrm{MSE}\left(\left\{\bm{\tilde{x}}^{(p)}_{2:\tilde{T}}\right\}, \ \left\{\bm{\hat{\tilde{x}}}^{(p)}_{2:\tilde{T}}\right\}\right) = \frac{1}{S(\Tilde{T}-1)}\sum_{p=1}^S \sum_{t=2}^{\Tilde{T}} \norm{\ \bm{\tilde{x}}^{(p)}_t - \bm{\hat{\tilde{x}}}^{(p)}_t \ }^2_2 .
\end{align}

In all experiments, we use the RAdam \citep{liu2019radam} optimizer with a learning rate starting at $10^{-3}$ which is exponentially reduced to reach $10^{-6}$ at the end of training.
For all datasets, we trained for $5000$ epochs, where one epoch is defined as processing of $50$ batches of size $S=16$. This comes down to a total of $250,000$ parameter updates in each training run. For the Lorenz-63 and Lorenz-96 and the empirical ECG data, we used a sequence length of $\Tilde{T} = 200$. For the EEG data, we used only $\Tilde{T} = 50$.

\paragraph{aGTF annealing protocol}\label{appx:annealing}
 Note that the Jacobians $\tilde{\mJ}_t$ in Eq. \eqref{eq:jac_prod_power} implicitly depend on $\alpha$. Hence, we replace the Jacobians of the forced model by data-inferred Jacobians $\hat{\mJ}_t$ evaluated at estimated teacher signals $\hat{\vz}_t$
 in latent space.
 This approximation generally holds when the model is either strongly forced or when the model dynamics is already close to that of the observed teacher system. To approach this scenario, we employ an $\alpha$-annealing protocol which starts with strong forcing at the beginning of training when the model is still far off from the observed system, and then smoothly decreases forcing throughout training depending on the model's Jacobians as the observed system is captured increasingly better. This procedure is formally described in Algorithm \ref{alg:adaptive_gtf}. Given a batch of teacher signals, we first compute Jacobians at each time step for each sequence in the batch and the norm of Eq. \eqref{eq:jac_prod_power}. We then set $\alpha=0$ if the maximum norm for the current batch is smaller than $1$, and else compute $\alpha$ according to \eqref{eq:alpha_geomean}. Starting with strong forcing at the beginning of training, i.e. $\alpha_0 = 1$, $\alpha_n$ is set up to decay exponentially over the course of training
 until it hits a lower bound given by the norm of
 \eqref{eq:jac_prod_power} of the converged model. This is achieved by using an exponential moving average
 for $\alpha_n$ throughout training, which is controlled by the hyperparameter $\gamma$:
 \begin{equation}
     \alpha_n = 
     \begin{cases}
         (1-\gamma)\alpha + \gamma \alpha_{n-1}, & \text{if} \ \alpha < \alpha_{n-1} \\
         \alpha, & \text{otherwise},
     \end{cases}
 \end{equation}
 where the index-free $\alpha$ is given by Eq. \eqref{eq:alpha_geomean}. Note that we replace $\alpha_n$ with the estimated $\alpha$ of the current training batch if it exceeds the current value $\alpha_{n-1}$. This is to avoid exploding gradients caused by sudden bifurcations during RNN training.
 
\begin{algorithm}[H]
\caption{Adaptive GTF}
\label{alg:adaptive_gtf}
\begin{algorithmic}[1]
\REQUIRE $\alpha_{n-1}$ estimate at previous optimization step, update interval $k$, decay rate $\gamma$, batch of forcing signals $\{\hat{\vz}^{(p)}_{1:T}\}_{p=1}^S$, RNN $\mF_{\theta_n}$
\IF{$n \mod{k} == 0$} 
    \STATE $\hat{\mJ}^{(p)}_{t} = Jacobian(\mF_{\theta_n}, \hat{\vz}^{(p)}_t)$  \qquad \qquad \qquad \qquad \COMMENT{$\triangleright$ Compute Jacobians manually or via automatic differentiation}
    \STATE $\kappa = \max\limits_{p} \norm{\mathcal{\mG}(\hat{\mJ}^{(p)}_{T:2})}$ \qquad \qquad \qquad \qquad \qquad \qquad \qquad \qquad \quad \ \COMMENT{$\triangleright$ Approximate $\norm{\mathcal{G}}$ using one of \eqref{eq:geomean_explog}, \eqref{eq:arithmetic_approx}, \eqref{eq:geomean_upper_bound}}
    \STATE $\alpha = \max{\left(0, 1 - \frac{1}{\kappa}\right)}$  \qquad \qquad \qquad \qquad \qquad \qquad \qquad \qquad \qquad \qquad \quad \ \ \COMMENT{$\triangleright$ Set $\alpha=0$ when Jacobians converge} 
    \IF{$\alpha > \alpha_{n-1}$}
        \STATE $\alpha_n = \alpha$
    \ELSE
        \STATE $\alpha_n = (1-\gamma)\alpha + \gamma \alpha_{n-1}$
    \ENDIF
\ELSE
    \STATE $\alpha_n = \alpha_{n-1}$ 
\ENDIF
\end{algorithmic}
\end{algorithm}
To reduce the resulting training slow-down caused mainly by computation of Jacobians for each and every batch, one could update $\alpha_n$ only every $k$-th optimization step. On the other hand, too sparse $\alpha$-updates may lead to rather lazy reactions to sudden bifurcations in RNN dynamics. In practice, we found $k = 5$ to work well, not significantly harming convergence in training. For all experiments on aGTF we fixed $k=5$, $\alpha_0=1$, and $\gamma=0.999$ (i.e., these parameters do not 
need a grid search but could simply be set to the reasonable ad-hoc values used here).

\paragraph{Approximation of $\mathcal{\mG}$}\label{appx:geomean_approx}
To determine the optimal forcing $\alpha$ 
\eqref{eq:alpha_geomean}, we need an efficient and numerically stable estimate of Eq. \eqref{eq:jac_prod_power}. To this end, we approximate the principal logarithm of the Jacobian product by a sum of principal logarithms:
\begin{equation}\label{eq:geomean_explog}
    \mathcal{\mG}(\hat{\mJ}_{T:2}) = \left(\prod_{k=0}^{T-2}\hat{\bm{J}}_{T-k}\right)^{\frac{1}{T-1}} \overset{\circledone}{=} \exp{\left(\frac{1}{T-1}\log{\left(\prod_{k=0}^{T-2}\hat{\bm{J}}_{T-k}\right)}\right)} \approx  \exp{\left(\frac{1}{T-1}\sum_{t=2}^T\log{\hat{\bm{J}}_{t}}\right)}. 
\end{equation}
Note that equality $\circledone$ strictly holds if the Jacobian product is \textit{non-singular} \citep{higham2011}. 
The approximation made here is exact if Jacobians $\hat{\mJ}_{T:2}$ commute under multiplication.
While this assumption does not hold in general, we find that at least for the shPLRNN trained on the systems and data introduced in this work, Jacobians approximately commute (Fig. \ref{fig:jac_commuting}).
Furthermore, in practice we find that in most cases the arithmetic mean of Jacobians
\begin{equation}\label{eq:arithmetic_approx}
    \mathcal{G}\left(\hat{\mJ}_{T:2}\right) \approx \frac{1}{T-1}\sum_{t=2}^{T} \hat{\mJ}_t
\end{equation}  
can be used as a plug-in replacement for the approximation above, as it produces similar $\alpha$ estimates, as shown in Fig. \ref{fig:alpha_estimates}. This replacement is mainly motivated by the runtime improvement for the arithmetic mean over the matrix exponentiation and logarithms involved in computing \eqref{eq:geomean_explog}.

\begin{figure}[!ht]
    \centering
	\includegraphics[width=0.6\linewidth]{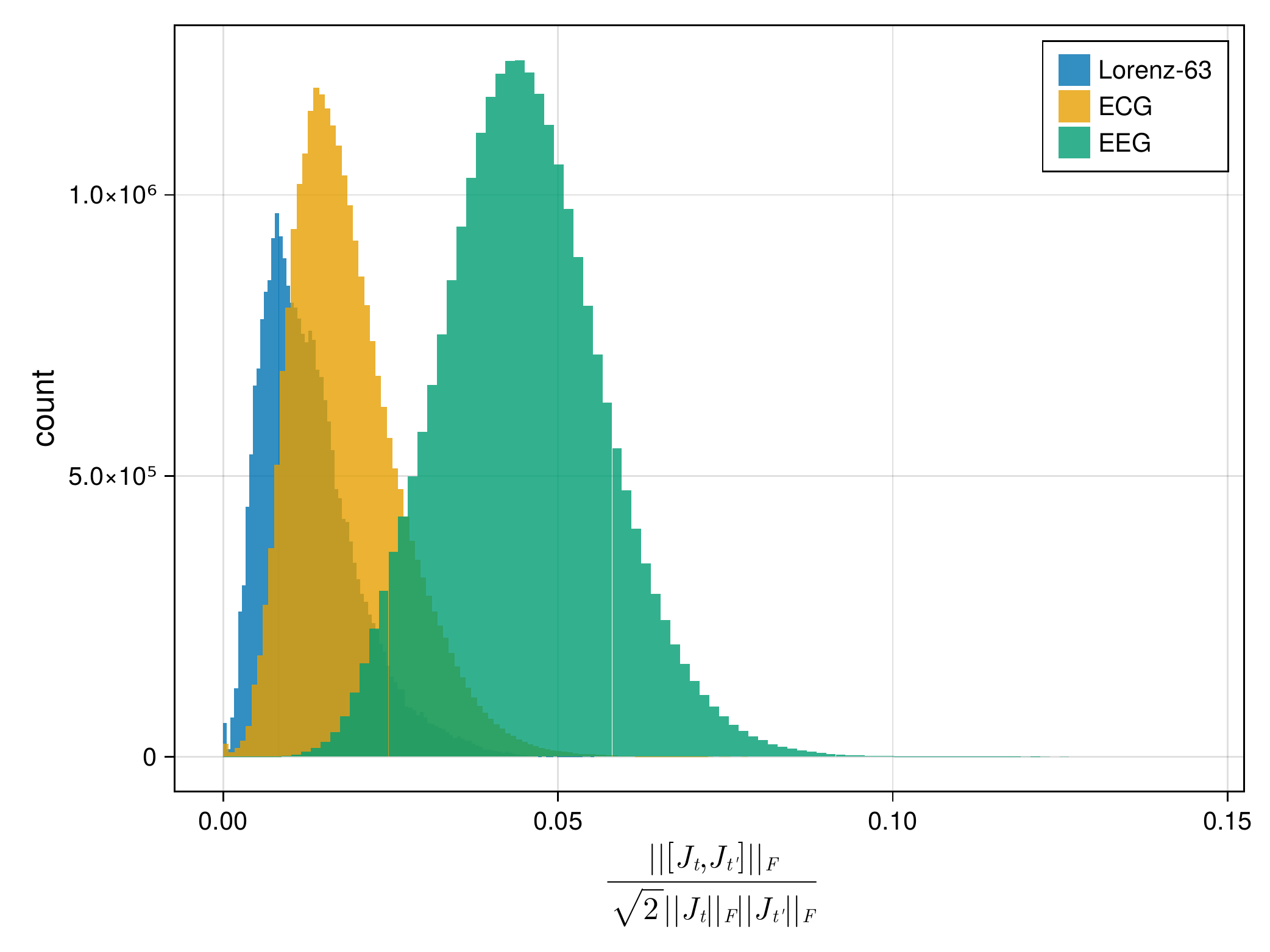}
	\caption{Histograms of relative Jacobian commutator norms for converged shPLRNNs \eqref{eq:bounded_shPLRNN} trained on different datasets. We measure the commutativity of Jacobians by first drawing an orbit of length $T=5000$ and computing the ratio between the Frobenius norm of the commutator of all possible permutations of pair-wise Jacobians in the orbit, and the respective tight upper bound on the commutator norm \citep{bottcher2008frobenius}. For all of the systems shown (Lorenz-63, ECG and EEG), most mass is concentrated below $5\%$ of the upper bound, which is a reference point for maximally non-commuting Jacobians.}
    \label{fig:jac_commuting}
\end{figure}
\begin{figure}
    \centering
    \begin{subfigure}[b]{.49\textwidth}
        \includegraphics[width=\textwidth]{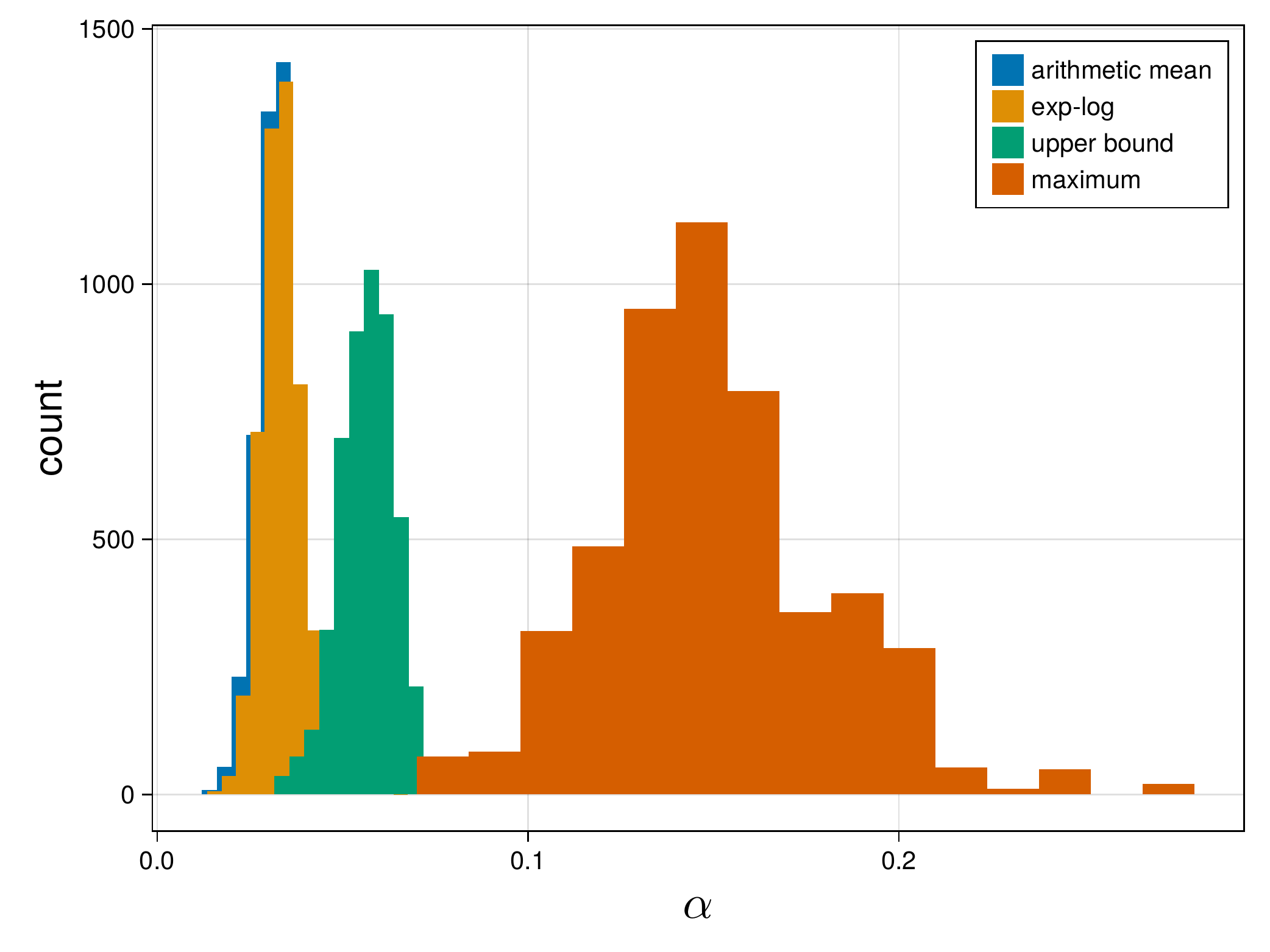}
    \end{subfigure}
    \begin{subfigure}[b]{.49\textwidth}
        \includegraphics[width=\textwidth]{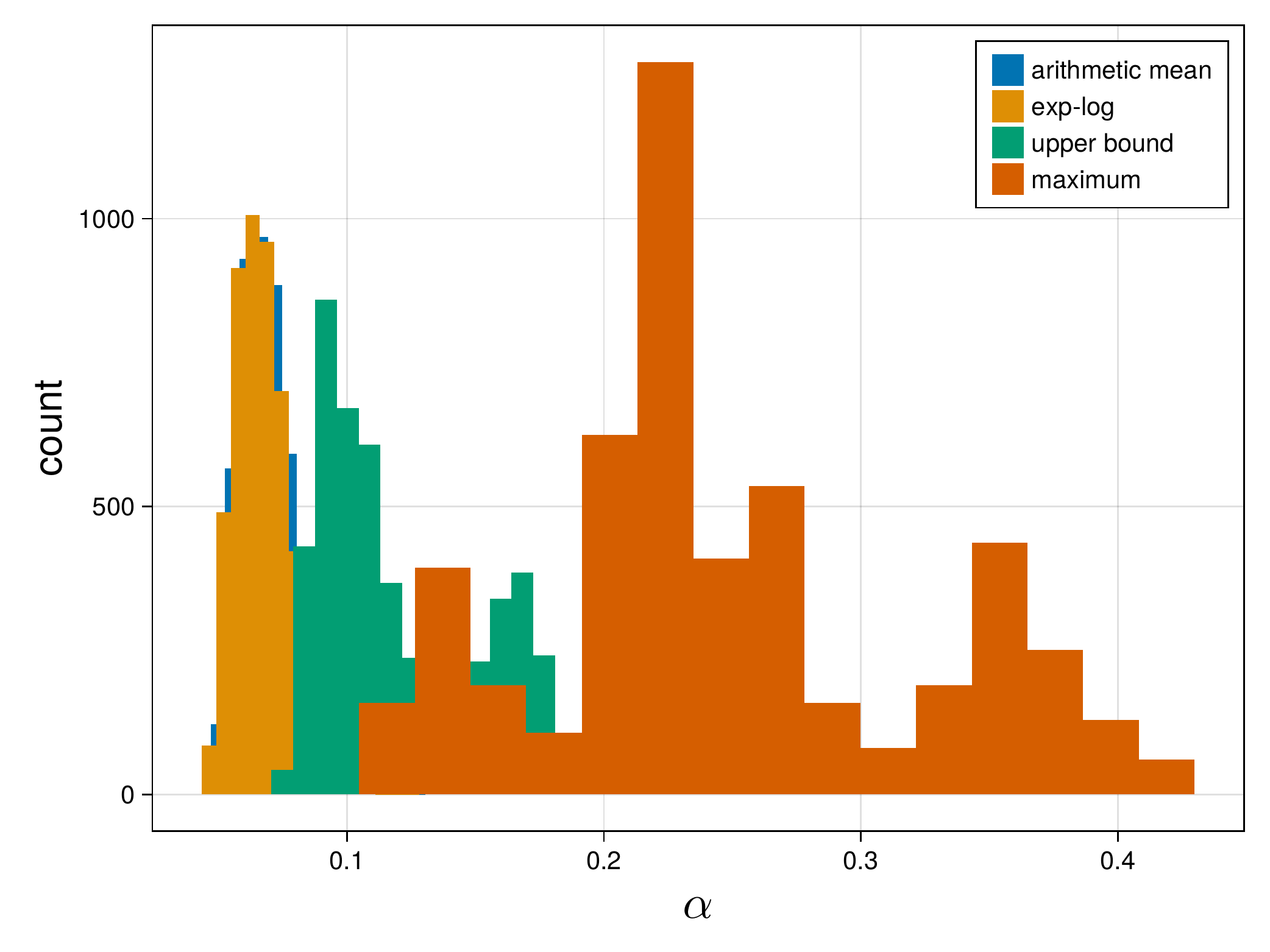}
    \end{subfigure}
    \caption{Histograms of different $\alpha$ estimates for a shPLRNN trained on the Lorenz-63 system (left) and on the ECG data (right) after convergence to the desired dynamics. 
    Legend labels refer to the different approximations: `arithmetic mean' \eqref{eq:arithmetic_approx}, `exp-log' \eqref{eq:geomean_explog}, `upper bound' \eqref{eq:geomean_upper_bound} and `maximum' \eqref{eq:sigma_max_estimate}. To produce this plot, we sampled $5000$ sequences of length $T \in \left[100, 200\right]$ and computed $\alpha$ according to Eq. \eqref{eq:alpha_geomean}. Estimates for the `arithmetic mean' and `exp-log' approximations highly correlate for both datasets (Pearson's $r \approx 0.996$ across all data points in both datasets). For comparison, estimates of $\alpha$ computed through the Jacobian norm upper bound of the shPLRNN given by Eq. \eqref{eq:sigma_upper_bound} are $\alpha \approx 0.87$ for the ECG and $\alpha \approx 0.79$ for the Lorenz-63.}
    \label{fig:alpha_estimates}
\end{figure}

We can also
approximate $\norm{\mathcal{\mG}(\hat{\mJ}_{T:2})}$ by using the upper bound
\begin{align}\label{eq:geomean_upper_bound}
\begin{split}
    \norm{\mathcal{\mG}(\hat{\mJ}_{T:2})} = \norm{\left(\prod_{k=0}^{T-2}\hat{\bm{J}}_{T-k}\right)^{\frac{1}{T-1}}} &\leq \left(\prod_{k=0}^{T-2}\norm{\hat{\bm{J}}_{T-k}}\right)^{\frac{1}{T-1}} \\ &= \exp\left(\frac{1}{T-1}\sum_{t=2}^T \log\norm{\hat{\bm{J}}_{T-k}}\right) \\
    &= \exp\left(\frac{1}{T-1}\sum_{t=2}^T \log \sigma_{max}\left(\hat{\bm{J}}_{T-k} \right)\right).
\end{split}
\end{align}
While technically accurate, the upper bound is biased towards higher values for $\alpha$ than generally required in practice, similar to the estimates in Eq. \eqref{eq:sigma_upper_bound} and \eqref{eq:sigma_max_estimate} (cf. Fig. \ref{fig:alpha_estimates}).

\paragraph{Line search for optimal $\alpha$}\label{appx:line_search}
For comparison to the performance of our aGTF procedure outlined above, we also determined optimal settings for $\alpha$ through a line search over different values $\alpha \in \left[0, 1\right]$ in steps of $0.05$. Figure \ref{fig:alpha_gs} shows $D_{\textrm{stsp}}$ and $D_H$ values for shPLRNNs \eqref{eq:bounded_shPLRNN} on different datasets against $\alpha$. As some of the graphs indicate, performance is often not overly sensitive to the precise choice of $\alpha$, and generally on par with the automatic adjustment of $\alpha$ in aGTF, remediating the need for grid search. 

\paragraph{Latent model regularization} \label{appx:model_reg}
Performance of the annealing protocol can be improved even further
with the manifold attractor regularization suggested in \citet{schmidt2021identifying}, which 
adds the following term with regularization factor $\lambda_{reg}$ to the loss: 
\begin{equation}\label{eq:l2_reg}
    \mathcal{L}_{\textrm{reg}} = \lambda_\textrm{reg} \left( \norm{\mathbb{1} - \mA}_F^2 + \norm{\mW_1}_F^2 + \norm{\mW_2}_F^2 + \norm{\vh_1}_2^2 + \norm{\vh_2}_2^2 \right)
\end{equation}
This has two benefits: 1) It pushes the shPLRNN's Jacobian towards the identity, thereby improving error propagation \cite{schmidt2021identifying} and reducing the amount of necessary forcing by already taming exploding/ vanishing gradients, and 2) it improves accuracy of the approximation of $\mathcal{\mG}$ (see Eq. \eqref{eq:geomean_explog}).

\paragraph{Observation model regularization}
When using a trainable linear observation model \eqref{eq:obs_model}, we need to invert the model to obtain teacher signals in latent space \eqref{eq:invTF}. However, we found that very occasionally $\mB$ becomes ill-conditioned, i.e. close to singular.
To avoid this,
one can regularize the condition number of $\mB$:
\begin{equation}\label{eq:cond_num_reg}
    \mathcal{L}_{cn} = \lambda_{cn} \left(1 - \frac{\sigma_{max}(\mB)}{\sigma_{min}(\mB) + \epsilon} \right)^2,
\end{equation}
where $\sigma_{max}(\mB)$ and $\sigma_{min}(\mB)$ are the largest and smallest singular values of $\mB$, respectively, $\lambda_{cn}$ is a hyperparameter and $\epsilon$ ($=10^{-8}$) a small number added for numerical stability. This regularization pushes the condition number towards $1$, thus
ensuring invertibility. 

\begin{figure}
    \centering
    \begin{subfigure}[b]{.49\textwidth}
        \includegraphics[width=\textwidth]{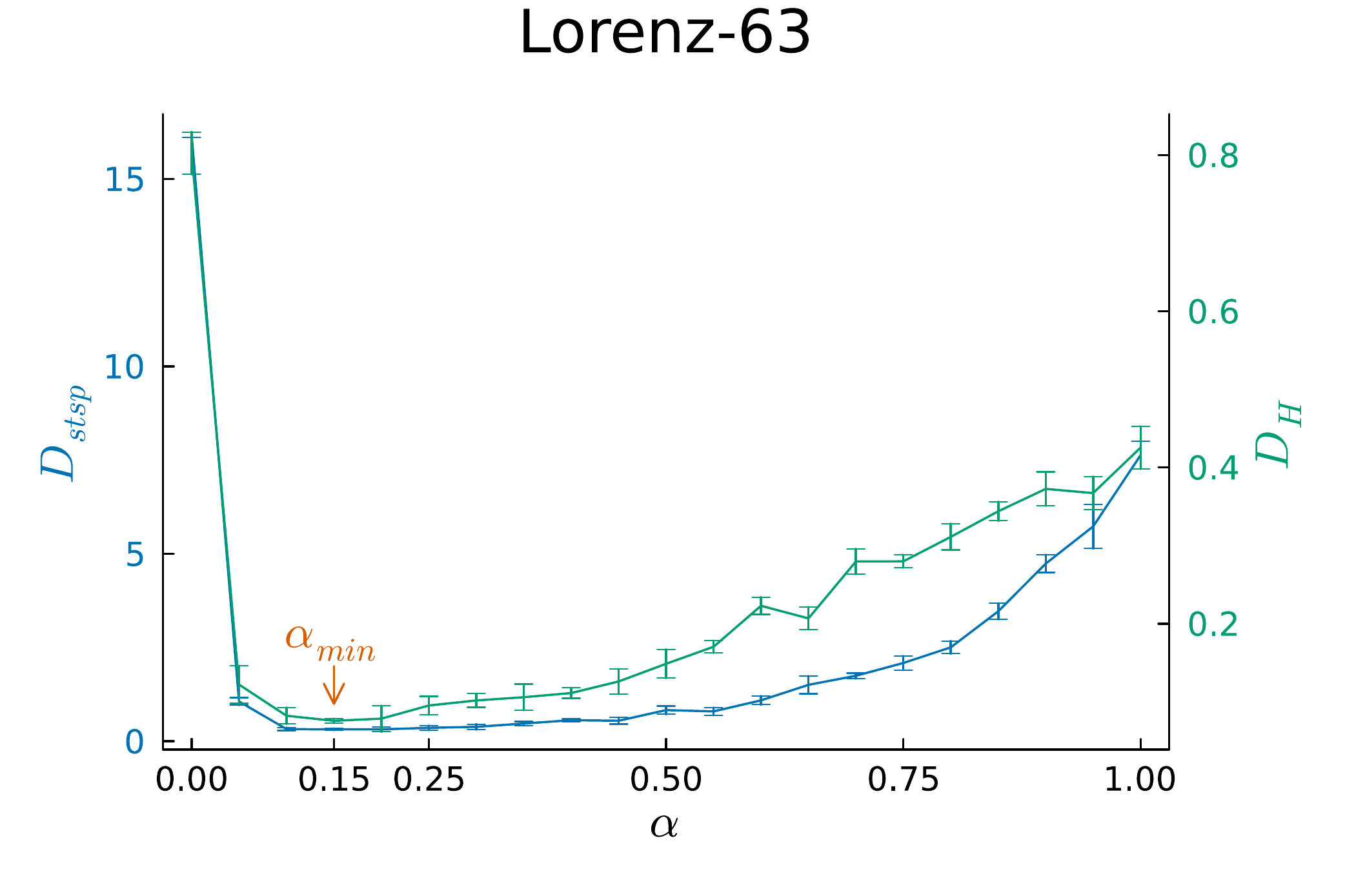}
    \end{subfigure}
    \begin{subfigure}[b]{.49\textwidth}
        \includegraphics[width=\textwidth]{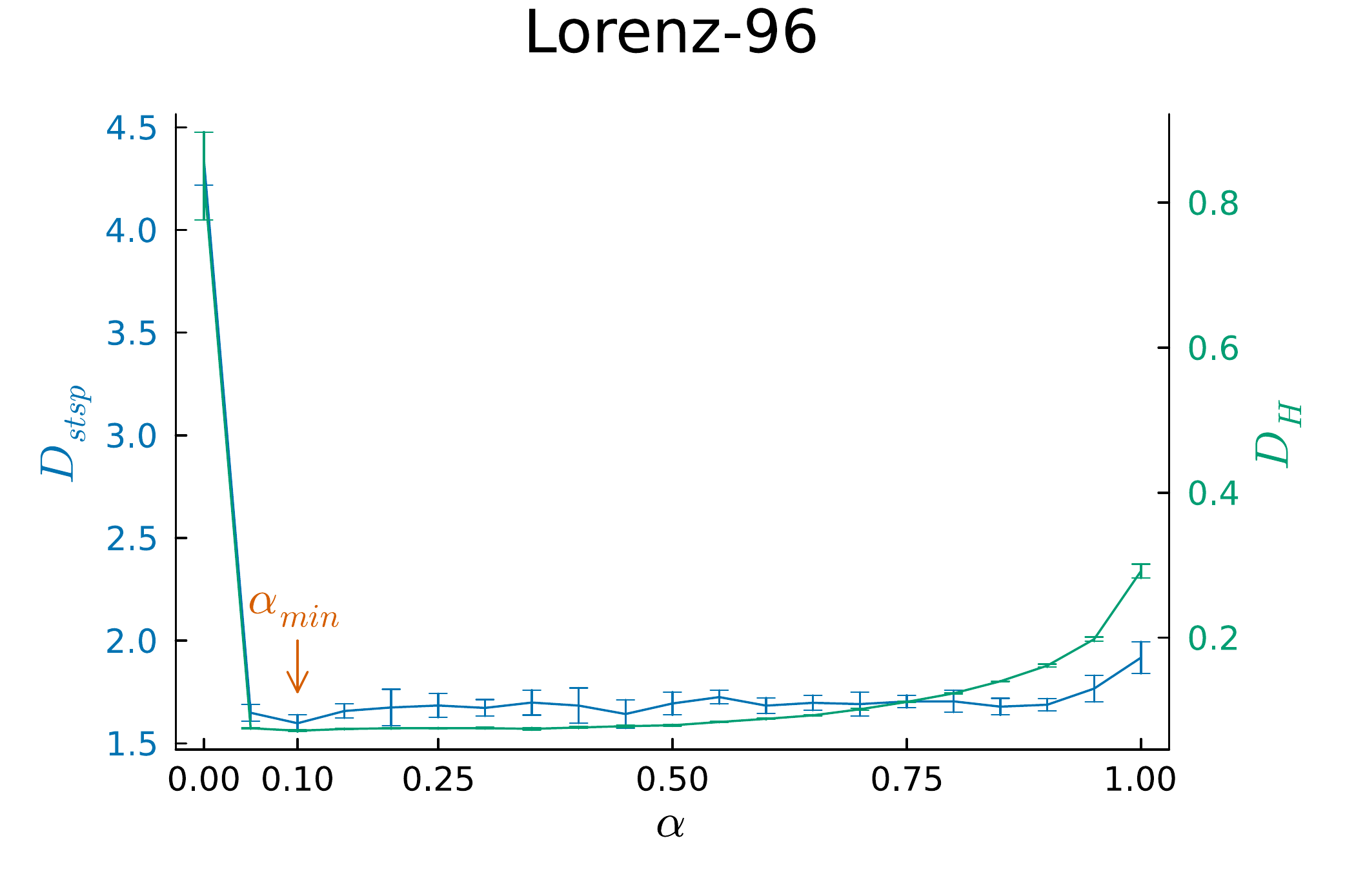}
    \end{subfigure}
    \begin{subfigure}[b]{.49\textwidth}
        \includegraphics[width=\textwidth]{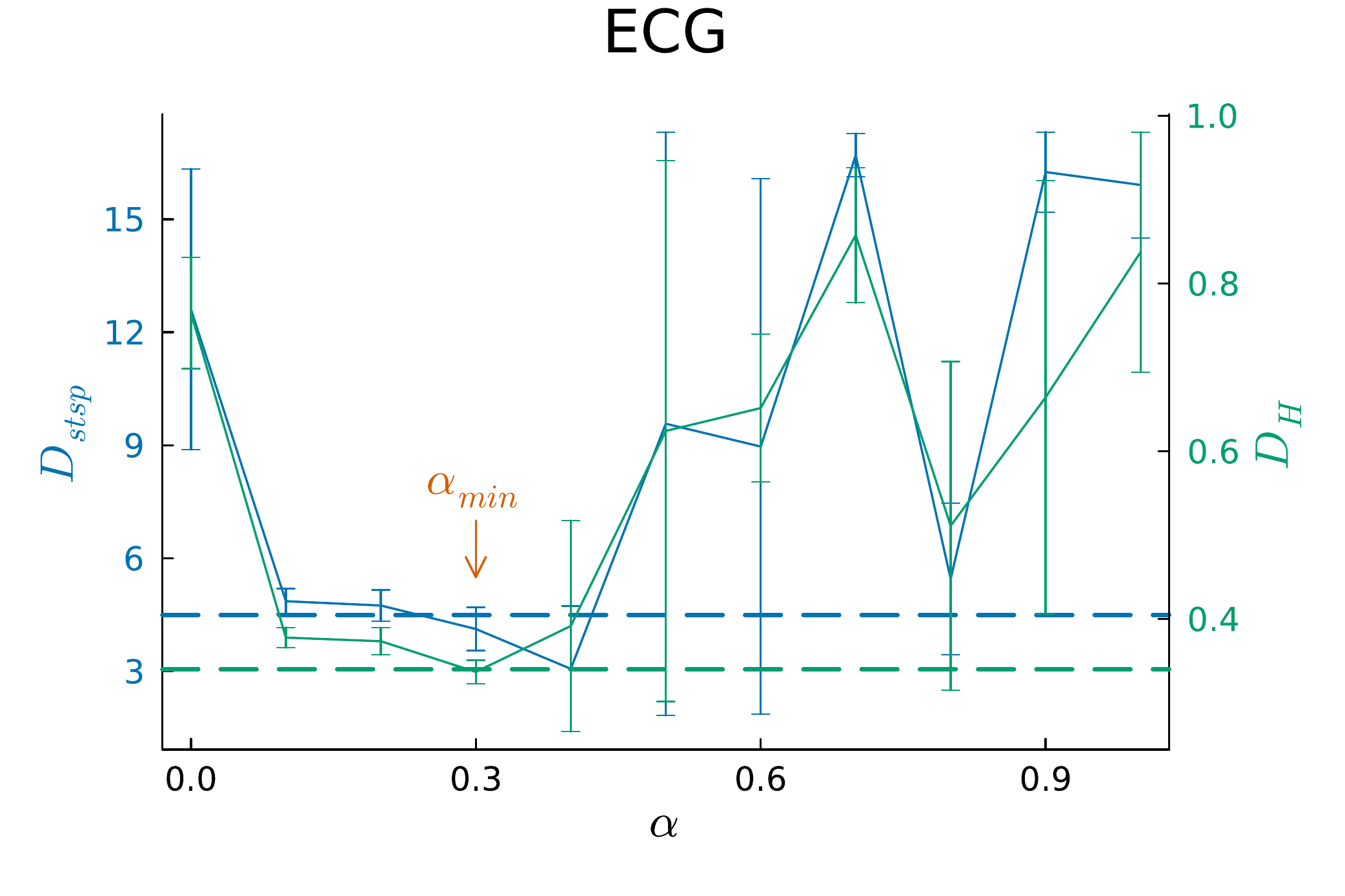}
    \end{subfigure}
    \begin{subfigure}[b]{.49\textwidth}
        \includegraphics[width=\textwidth]{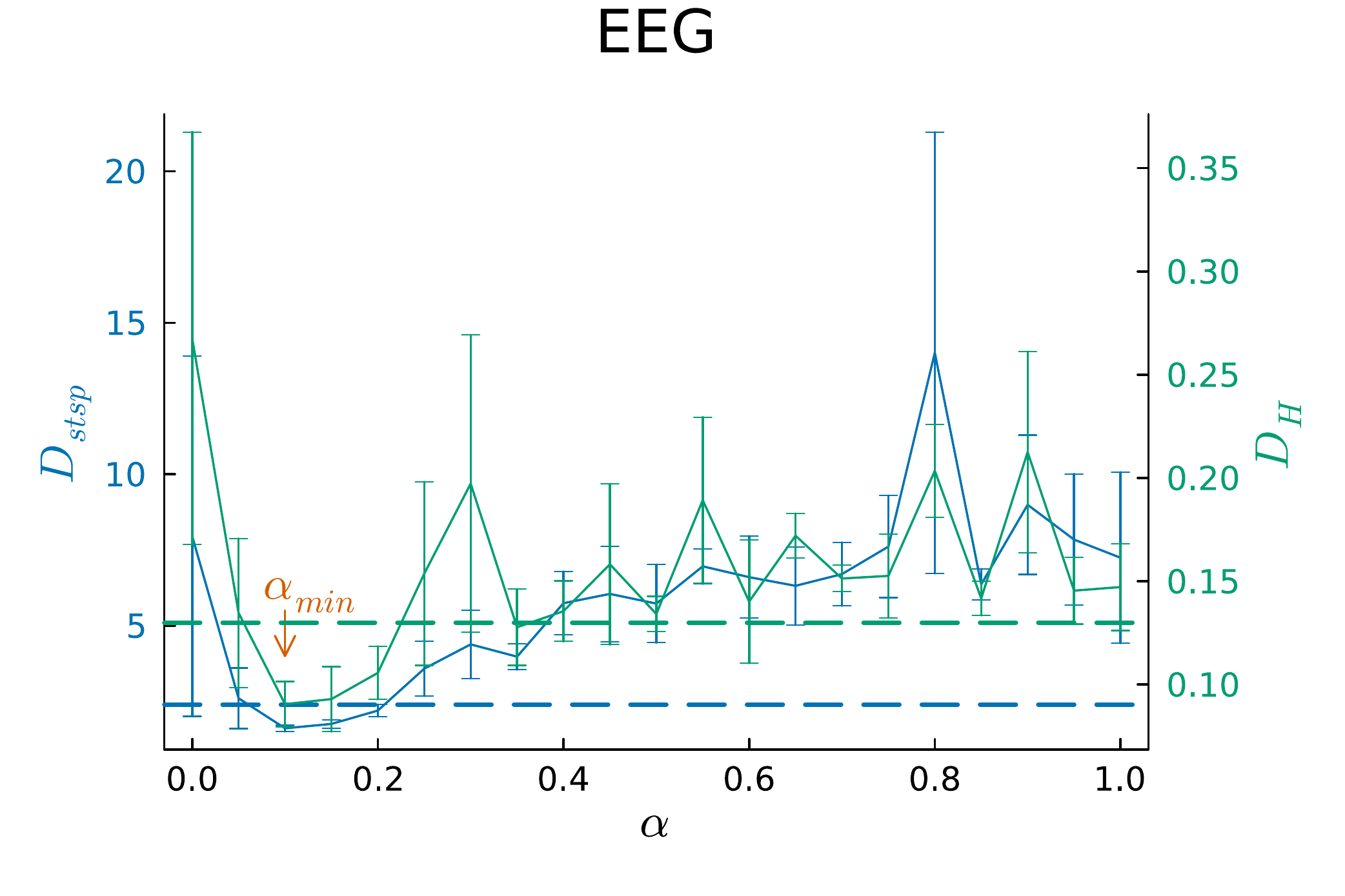}
    \end{subfigure}
    \caption{Line search over different $\alpha$ values. The minima lie at $\alpha_{min} \approx 0.15$ for the Lorenz-63, and $\alpha_{min} \approx 0.1$ for the Lorenz-96 and EEG data, respectively, while for the ECG data we find $\alpha_{min} \approx 0.3$. 
    There appears to be a somewhat larger $\alpha$-range over which performance is similarly good, and thus some lenience regarding the precise adjustment of this value. 
    Dashed blue ($D_\textrm{stsp}$) and green ($D_H$) lines indicate performance levels of aGTF on ECG and EEG data, respectively.}
    \label{fig:alpha_gs}
\end{figure}

\subsection{Benchmark Systems and Real-World Data}\label{appx:datasets}
\paragraph{Benchmark: Lorenz-63}\label{pg:supp:lor63}
Introduced in \citet{lorenz1963deterministic} as one of the first ODE systems for which chaotic behavior was demonstrated, the 3d Lorenz-63 system has become the most common benchmark in the DS reconstruction literature. Designed as a simple model of atmospheric convection, the system exhibits different routes to chaos like period doubling and homoclinic bifurcations and, due to a fundamental symmetry, exhibits the famous butterfly-wing shape (see Fig. \ref{fig:lor63_reconstruction}, \citet{perko2001}).
The system is described by the following set of ODEs:
\begin{align*}
    \dot{x} &= \sigma (y- x) \\
    \dot{y} &= x(\rho-z)-y \\
    \dot{z} &= xy - \beta z , \\
\end{align*}
where we place the system into the chaotic dynamical regime with $\sigma=10$, $\beta=\frac{8}{3}$ and $\rho=28$. For each of the training and test sets we simulate a trajectory of length $T = 10^5$, each starting at a different random initial condition $\bm{u}_0 = (x_0, y_0, z_0)^T \sim \mathcal{N}(\bm{0}, \mathbb{1}_{3\times3})$, using the \href{https://juliadynamics.github.io/DynamicalSystems.jl/latest/}{DynamicalSystems.jl} \citep{Datseris2018} Julia library. The ODE system is integrated using a Runge-Kutta scheme with adaptive step size and a read-out interval of $\Delta t = 0.01$.  Furthermore, we contaminate the training set with Gaussian observation noise, using a noise level of $5\%$ of the data standard deviation. Finally, we standardize each dimension of both training and test set to zero mean and unit variance.

\begin{figure}[!htp]
    \centering
\includegraphics[width=0.90\linewidth]{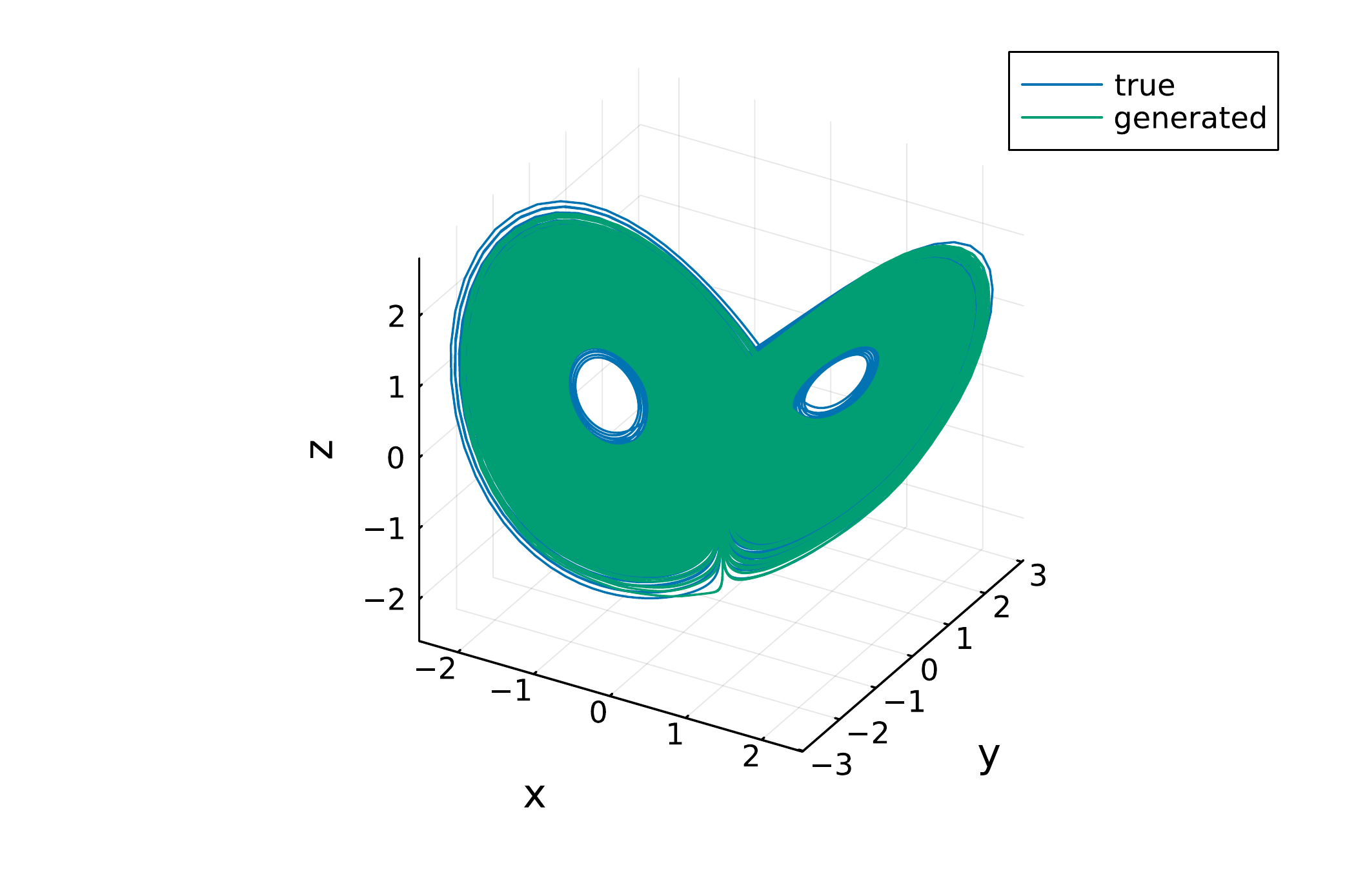}
 \vspace{-.6cm}   
    \caption{Example reconstruction of the Lorenz-63 system by shPLRNN+GTF. Note that the model freely generates the dynamics given the first time point of the test set.}
    \label{fig:lor63_reconstruction}
\end{figure}

\paragraph{Benchmark: Lorenz-96}\label{pg:supp:lor96}
A higher-dimensional, spatially extended system for atmospheric convection with local neighborhood interactions was suggested in \citet{lorenz1996predictability}. The system can be formulated in arbitrarily high dimensions: 
\begin{align*}
    \dot{x}_k = (x_{k+1} - x_{k-2})x_{k-1} - x_k + F, \qquad k=1 \dots N
\end{align*}
with dynamical variables $\{x_k\}$, constant forcing term $F$ and dimension $N$. For our experiments we choose $N=20$ and $F=16$, such that the system is situated within the chaotic regime. We create a training and test set using the same protocol as for the Lorenz-63 model, i.e. we draw two trajectories, one for training and one for testing. The training trajectory is contaminated with $5\%$ Gaussian noise, and both resulting datasets are standardized to have zero mean and unit variance on each dimension. See Fig. \ref{fig:lor96_reconstruction} for an excerpt of the system dynamics in form of a heatmap, together with a reconstruction using the method employed in this work.

\begin{figure}[!htp]
    \centering
\includegraphics[width=0.99\linewidth]{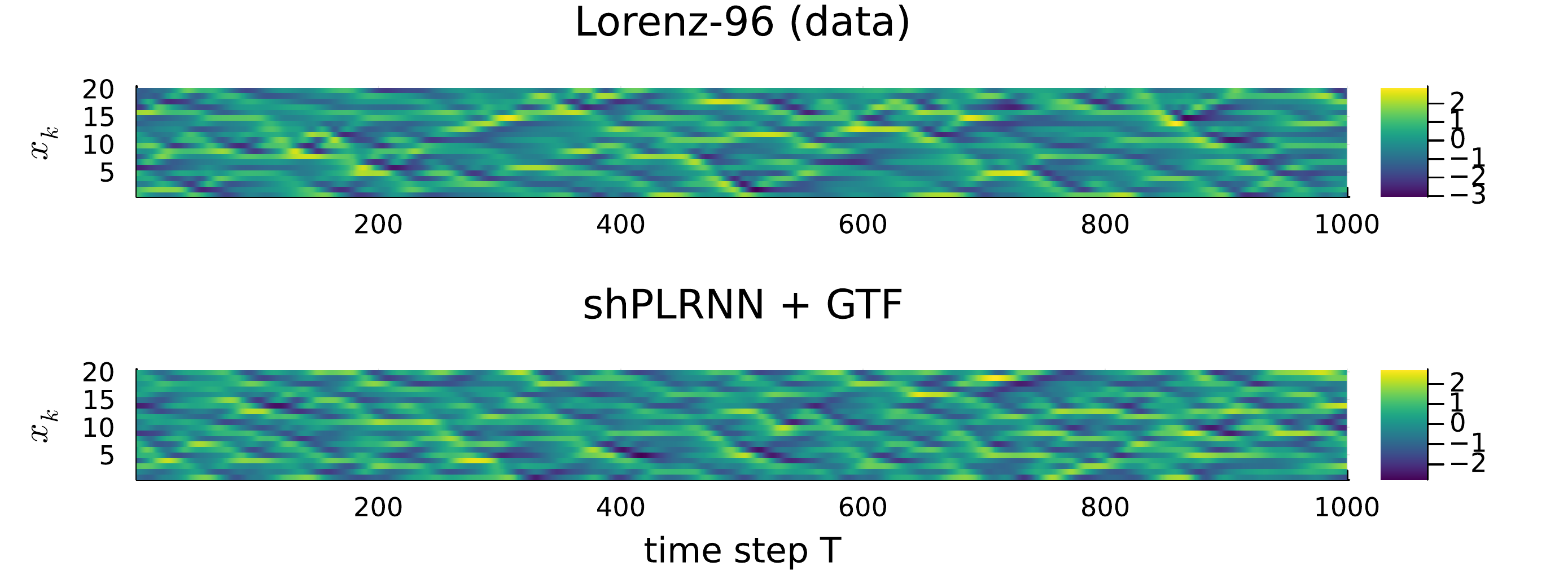}
    \caption{Heatmaps of an excerpt of the Lorenz-96 test set (see \ref{pg:supp:lor96}) and an example reconstruction using the shPLRNN+GTF.}
    \label{fig:lor96_reconstruction}
\end{figure}

\paragraph{Benchmark: Multiscale Lorenz-96}\label{pg:supp:ms_lor96}
Building on the original Lorenz-96 model, \citet{thornes2017use} introduced an extended version which 
models atmospheric weather phenomena evolving on multiple temporal and spatial scales  
through  
nested sets of dynamical variables. The set of ODEs is given by
\begin{align}
\begin{split}
\frac{{dX_k}}{{dt}} &= X_{k-1}(X_{k+1} - X_{k-2}) - X_k + F - \frac{{hc}}{{b}} \sum_{j=1}^{J} Y_{j,k}, \\
\frac{{dY_{j,k}}}{{dt}} &= -cbY_{j+1,k}(Y_{j+2,k} - Y_{j-1,k}) - cY_{j,k} + \frac{{hc}}{{b}}X_k - \frac{he}{d}\sum_{i=1}^{I} Z_{i,j,k}, \\
\frac{{dZ_{i,j,k}}}{{dt}} &= ed Z_{i-1,j,k}(Z_{i+1,j,k} - Z_{i-2,j,k}) - g_Ze Z_{i,j,k} + \frac{he}{d}Y_{j,k}.   \\
\end{split}
\end{align}
These equations describe a system with $K$ slow large-scale variables $X$, each of which coupled to $J$ faster and smaller-scale variables $Y$, which in turn are coupled to $I$ very fast small-scale variables $Z$. Parameters $h$, $b$, $c$, $e$ and $d$ determine the coupling strength between different scales, $F$ is the external forcing strength and $g_Z$ is a damping parameter.
It has previously been used to assess the capability of different machine learning models to predict future DS states \citep{chattopadhyay2020data}, some of which -- like RCs and LSTMs -- also included in this work. \citet{chattopadhyay2020data}
assess forecasting abilities on a subset of the dynamical variables,
using $I=J=K=8$, which results in a $584$-dimensional system, with parameters $b=c=e=d=g_Z=10$, $h=1$, and forcing $F=20$. This places the system into a highly chaotic regime. 
However, the authors assumed that only the $K=8$ slow, large-scale variables $X$ are observed, i.e. all methods were trained on partial observations from the $584$-dimensional system.
To enable direct comparison, we trained N-ODE, dendPLRNN and the shPLRNN on the dataset provided within the authors' codebase. We find that all tested methods (N-ODE, shPLRNN+GTF, dendPLRNN+id-TF) produce both accurate short-term forecasts of the partially observed system,\footnote{See \citet{chattopadhyay2020data} for performance of RCs and LSTMs.} as well as statistically indistinguishable reconstructions of the long-term behavior 
(all within a $10\%$ median absolute deviation and $34\%$ SEM margin; Kruskal-Wallis test $p=0.84$; see Tab. \ref{tab:reduced-mslor96} and Figs. \ref{fig:reduced-mslor96-longterm}, \ref{fig:red_mslor96_forecast_traces}, \ref{fig:reduced-mslor96-shortterm}).

\begin{figure}[!ht]
    \centering
	\includegraphics[width=0.8\linewidth]{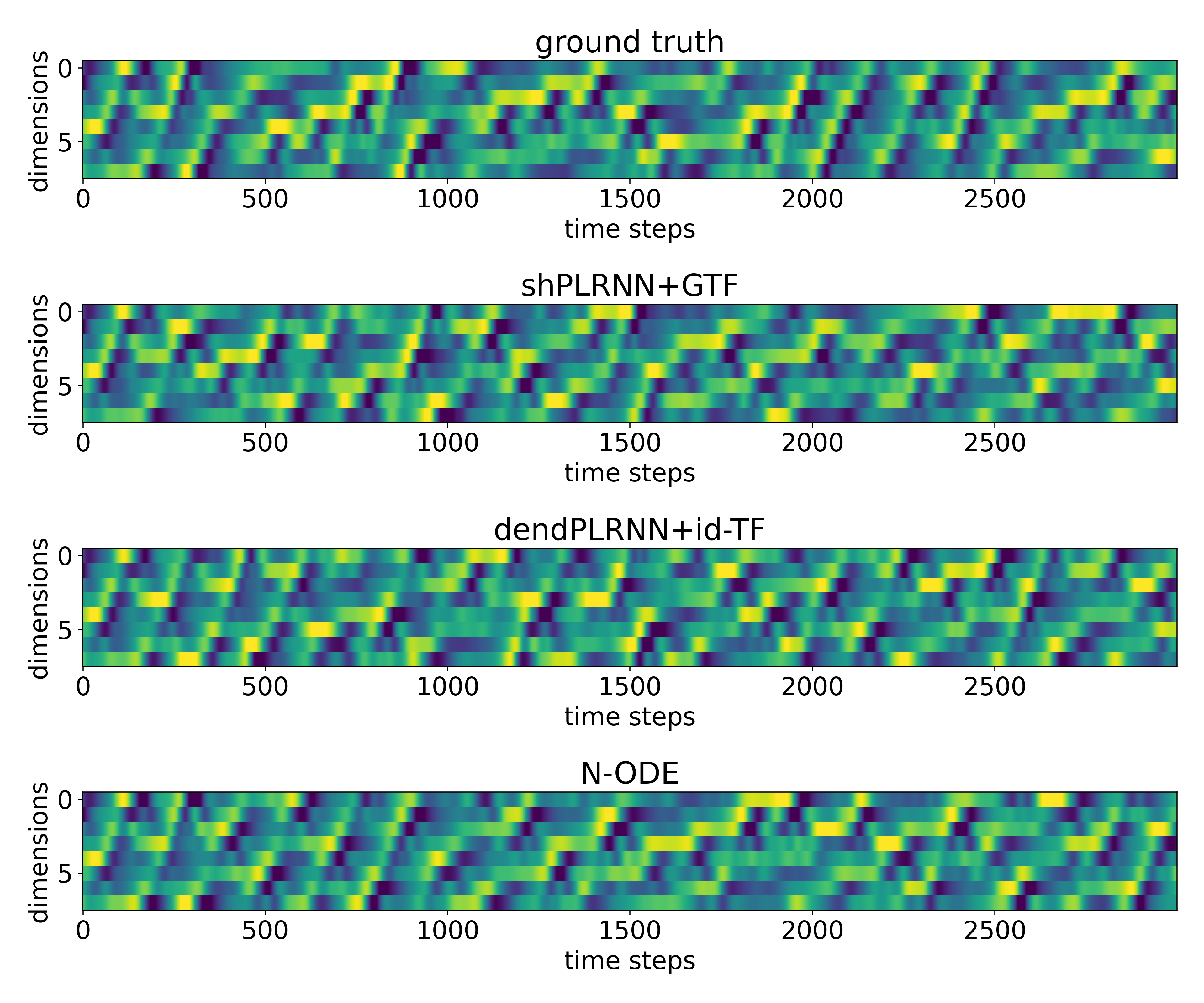}
	\caption{Long term spatio-temporal behavior for the three best performing methods trained on the partially observed multiscale Lorenz-96 model from \citet{chattopadhyay2020data}.}
	\label{fig:reduced-mslor96-longterm}
\end{figure}
\begin{figure}[!h]
    \centering
	\includegraphics[width=0.99\linewidth]{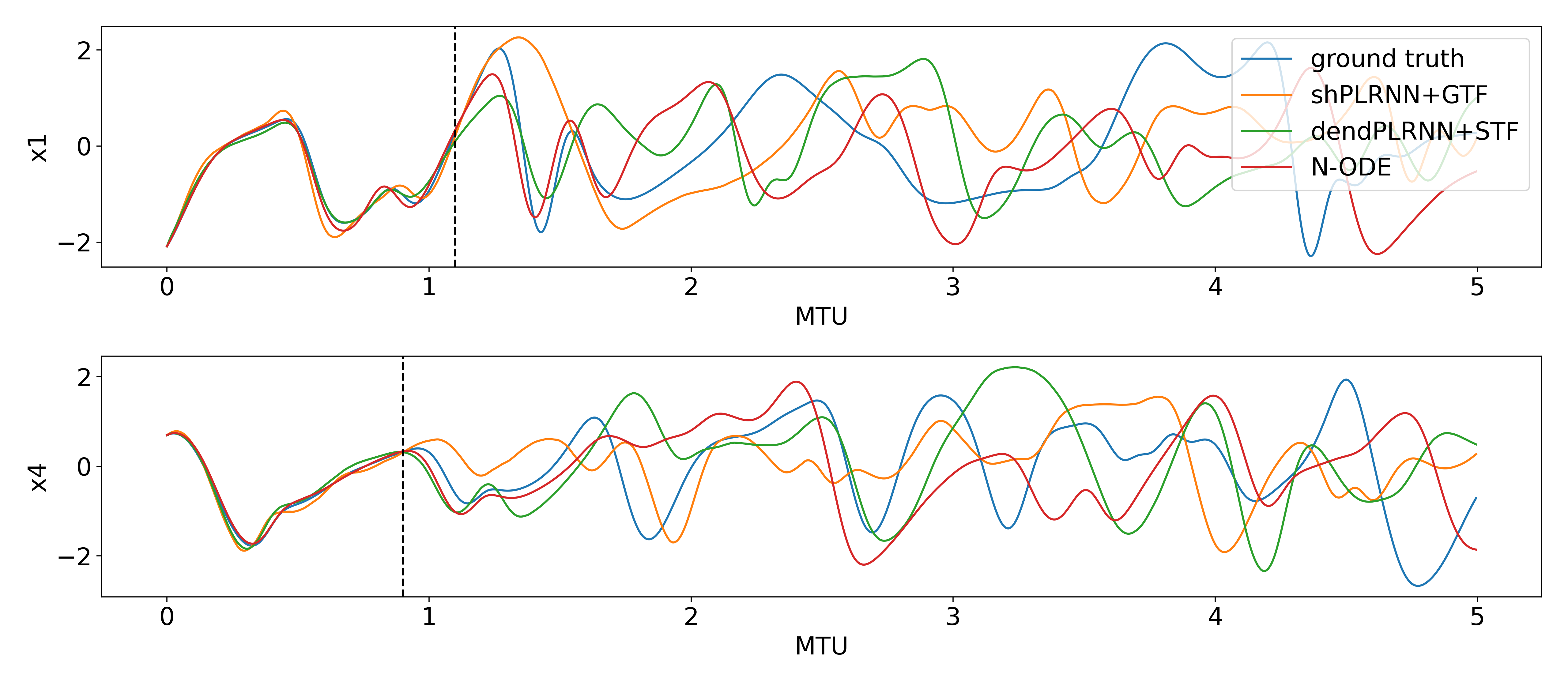}
	\caption{For direct comparison with Fig. 5 in \citet{chattopadhyay2020data}: Short-term forecasts on the partially observed multiscale Lorenz-96 system for the three best performing methods, starting from the first time step of the test set. Time is in model time units (MTU), where $1$ MTU$=200 \Delta t \approx 4.5/\lambda_{max}$. 
 The vertical line indicates the approximate prediction horizon, which for chaotic systems is limited by the system's maximum Lyapunov exponent. }
	\label{fig:red_mslor96_forecast_traces}
\end{figure}
\begin{figure}
    \centering
	\includegraphics[width=0.8\linewidth]{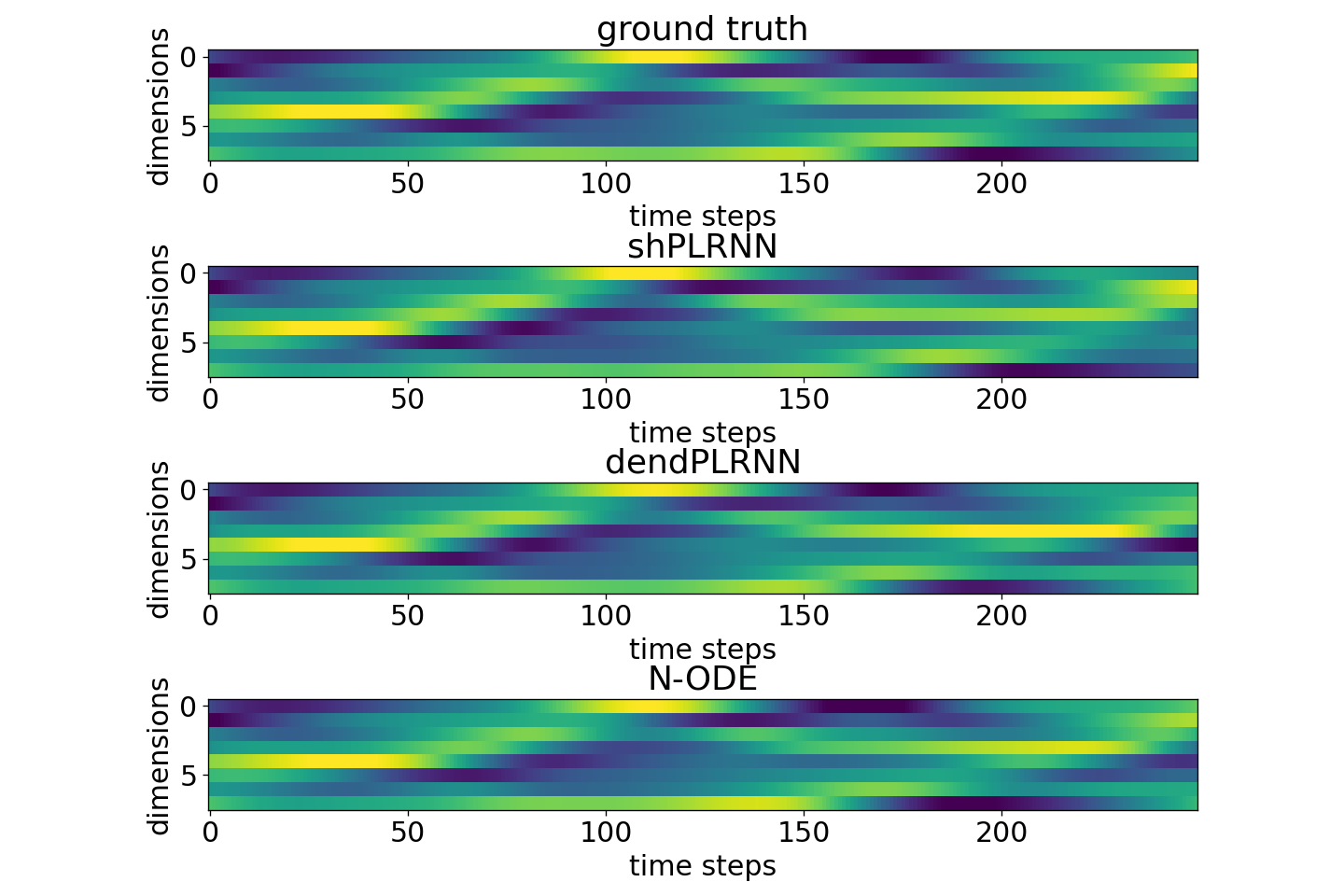}
	\caption{Short term spatio-temporal forecasts for the three best performing methods trained on the partially observed multiscale Lorenz-96 model from \citet{chattopadhyay2020data}.}
	\label{fig:reduced-mslor96-shortterm}
\end{figure}
\begin{table*}[!h]
\centering
\definecolor{Gray}{gray}{0.925}
\caption{Results for the partially observed multiscale Lorenz-96 system. Reported values are median $\pm$ median absolute deviation (MAD) over $20$ independent training runs. `dim' refers to the model's state space dimensionality (number of dynamical variables). $\lvert \bm{\theta} \rvert$ denotes the total number of \textit{trainable} parameters.}
\scalebox{0.90}{
\begin{tabular}{l l c c c c c}
        \toprule
        Dataset	&	Method	&	$D_{\textrm{stsp}}$ $\downarrow$ & $D_H$ $\downarrow$  & $\textrm{PE}(20)$ $\downarrow$	&  dim & $\lvert \bm{\theta} \rvert$ \\
        \midrule
        \multirow{3}{4em}{multiscale Lorenz-96}
        &    shPLRNN + GTF & $({6.1}\pm{1.1}) \cdot 10^{-2}$ & $({7.3}\pm{0.4}) \cdot 10^{-2}$ & $({5.7}\pm{0.6}) \cdot {10}^{{-3}}$ & $8$ & $1780$ \\
        &    dendPLRNN + id-TF      & $({6.3}\pm{1.4})\cdot 10^{-2}$ & $({6.9}\pm{0.4})\cdot 10^{-2}$ & $({3.4}\pm{1.6}) \cdot {10}^{{-3}}$ & $25$ &  $1845$ \\
        &    Neural-ODE      & $({6.2}\pm{0.3})$ $\cdot 10^{-2}$ & $({7.8}\pm{0.3})$ $\cdot 10^{-2}$ & $(4.6 \pm 0.3) \cdot 10^{-3}$ & $8$ & $1708$ \\
    
        \bottomrule
\end{tabular}
}
\label{tab:reduced-mslor96}
\end{table*}

\paragraph{Empirical Dataset: ECG}\label{appx:ecg}
The ECG data used here consists of a single time series with $T=419,973$ time points. Given a sampling frequency of $\si{700 \hertz}$, this corresponds to $\si{600 \second}$ of recording time. We first preprocessed the ECG data by smoothing the time series using a Gaussian filter ($\sigma = 6$, $l = 8\sigma + 1 = 49$), followed by standardization of the time series. We then delay-embed the signal using the PECUZAL algorithm \citep{kramer2021unified} implemented in the \href{https://juliadynamics.github.io/DynamicalSystems.jl/latest/}{DynamicalSystems.jl} \citep{Datseris2018} Julia library. The algorithm uses the $L$-statistic and non-uniform delays to find an optimal delay embedding (for details, see \citet{kramer2021unified} or the \href{https://juliadynamics.github.io/DelayEmbeddings.jl/stable/unified/#Unified-optimal-embedding}{documentation of the algorithm}). We set $\Delta L=0.05$, and use a Theiler window based on the first minimum of the mutual information, leading to an embedding dimension of $m=5$. For our experiments, we use the first $T=100,000$ samples ($\approx  \si{143 \second}$) as the training set and the next $T=100,000$ samples as the test set. 
The maximum Lyapunov exponent
is estimated as $\lambda_{max} = \left(2.19 \pm 0.05\right) \frac{1}{\si{\second}}$ by fitting a line to the average log-distance $log(d(t))$ between trajectories evolving from neighboring states for different embedding dimensions $m$ 
(Fig. \ref{fig:ecg_lyapunov}; cf. \citet{kantz1994robust, skokos2016chaos}), and agrees well with the literature \citep{govindan1998evidence}.
\begin{figure}
    \centering
    \includegraphics[width=0.6\textwidth]{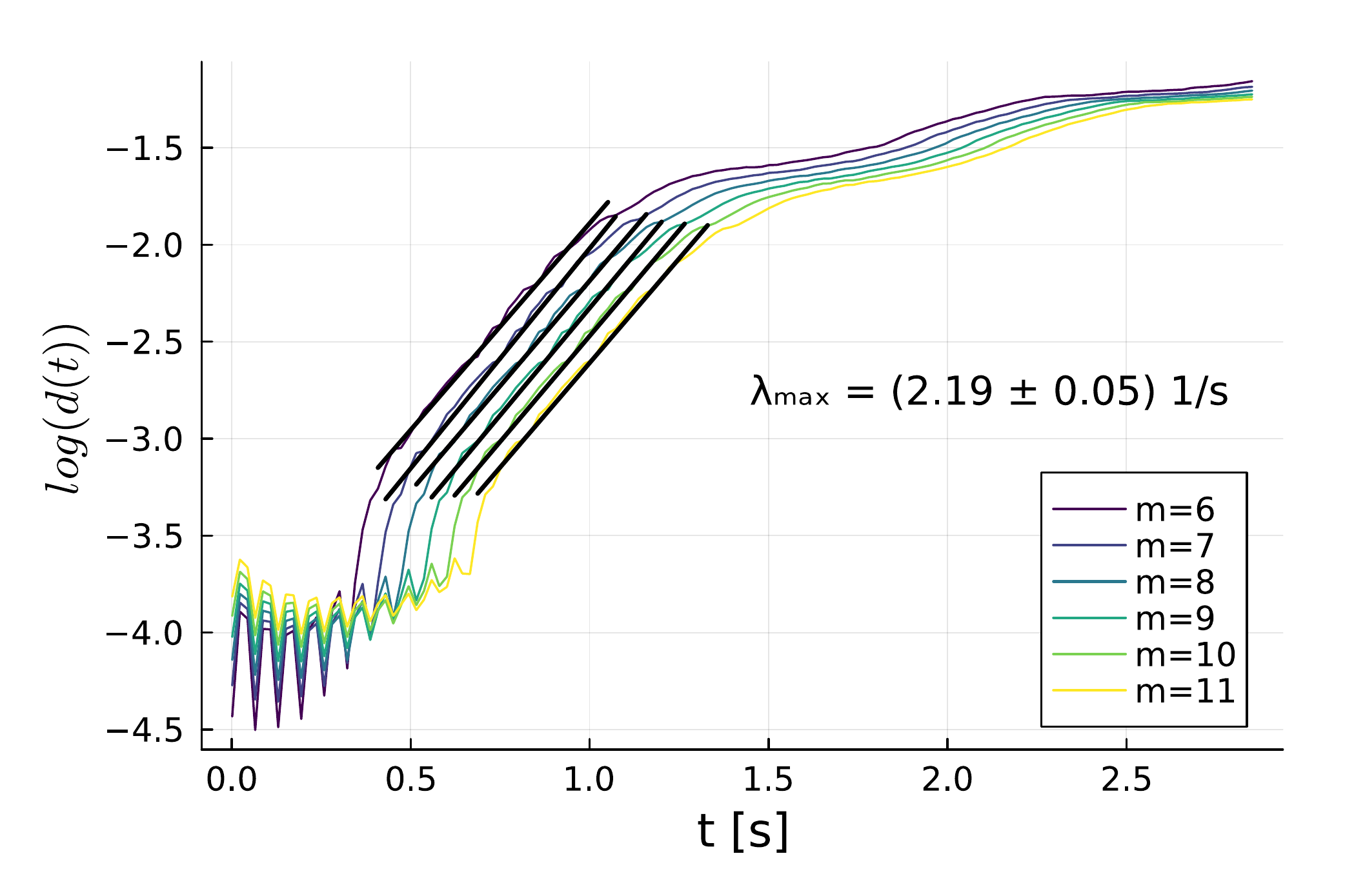}
    \caption{Estimation of the maximum Lyapunov exponent of the ECG time series for different embedding dimensions $m$ and fixed delay $\tau=59$. The delay was determined as the first minimum of the mutual information of the original time series.}
    \label{fig:ecg_lyapunov}
\end{figure}

\paragraph{Empirical Dataset: EEG}\label{pg:supp:EEG} 

The Electroencephalogram (EEG) dataset used here is the same as in \citet{brenner22} and \citet{Mikhaeil22} (the latter authors using a different preprocessing). The dataset consists of recordings from $64$ electrodes placed across the scalp of a human subject sitting still in a chair with eyes open (first subject performing task $1$ labeled ``S001R01''). The EEG 
sampling frequency was $160 \si{\hertz}$. These data are part of a larger study conducted by \citet{schalk_bci2000_2004}, openly available at \href{https://physionet.org/content/eegmmidb/1.0.0/}{PhysioNet} \citep{goldberger2000physiobank}. Here the data was preprocessed as in \citet{brenner22}, standardizing and smoothing each signal using a Hann filter with window length $15$ time bins. 
For computing the invariant (long-term) statistics $D_{\textrm{stsp}}$ and $D_H$ on the EEG data, tested models were (as for all other comparisons) simulated freely starting from just a data-inferred initial condition, but EEG reference values were taken from the training data since EEG time series were much shorter than ECG series (less than $10,000$ time steps for EEG compared to more than $400,000$ for ECG). 
Robust and reliable estimation of long-term properties like power spectra and attractor geometries would thus have been difficult for only a short (left-out) fraction of that time series. Apart from the fact that this was of course handled equally for all methods tested, however, note that this should not affect our measures $D_{\textrm{stsp}}$ and $D_H$ 
much, since (by definition) these long-term properties would not be expected to change on shorter time scales. 
To also make predictions more challenging under these conditions (without separate left-out set), we chose a relatively large 
prediction step $n$. For STF, the predictabiliy time for the EEG data was directly taken from \citet{Mikhaeil22}. However, reliably determining the predictability time from data is exactly one of the issues with STF as proposed, and 
while we did not perform systematic grid search on the forcing interval for STF, we observed that other intervals could improve performance for shPLRNN+STF. A more systematic comparison between GTF and STF therefore still warrants further research. 

\subsection{Details on Evaluation Measures}\label{appx:eval_measures}
\paragraph{Geometrical measure ($D_{\textrm{stsp}}$)}\label{appx:dstsp}
Given $p(\vx)$, generated by trajectories of the true system, and $q(\vx)$, generated by the model, the state space divergence
is defined as 
\begin{align}\label{eq:D_stsp}
    D_{\textrm{stsp}} := D_{\textrm{KL}}\left(p(\vx) \ || \ q(\vx)\right) = \int_{\vx \in \mathbb{R}^N} p(\vx) \log\frac{p(\vx)}{q(\vx)} d\vx
\end{align} 
For low dimensional observation spaces, $p(\vx)$ and $q(\vx)$ can be estimated directly from a binning of space \citep{koppe2019identifying, brenner22}, with the minimum-to-maximum range for binning determined by the extent of the observed (ground truth) attractor. Eq. \eqref{eq:D_stsp} is then approximately given by 
\begin{align}
    D_{\textrm{stsp}} = D_{\textrm{KL}}\left(\hat{p}(\vx) \ || \ \hat{q}(\vx)\right) \approx \sum_{k=1}^{K} \hat{p}_k(\vx) \log\frac{\hat{p_k}(\vx)}{\hat{q_k}(\vx)}
\end{align}
where $K=m^N$ is the total number of bins, with $m$ the number of bins per dimension, $\hat{p}_k(\vx)$ is the relative
number of 
ground truth orbit data points in bin $k$ and, likewise, $\hat{q}_k(\vx)$ that for generated orbits. For high-dimensional systems, a binning approach is no longer sensible. Instead, Gaussian Mixture Models (GMMs) were placed along orbits (see \cite{brenner22}), i.e. $\hat{p}(\vx) = 1/T\sum_{t=1}^T \mathcal{N}(\vx;\ \vx_t, \mSigma)$ and $\hat{q}(\vx) = 1/T\sum_{t=1}^T \mathcal{N}(\vx; \ \hat{\vx}_t, \mSigma)$, where $\vx_t$ and $\hat{\vx}_t$ are observed and generated states, respectively, $\mathcal{N}(\vx; \ \vx_t, \mSigma)$ is a multivariate Gaussian with mean vector $\vx_t$ and covariance matrix $\mSigma = \sigma^2 \mathbb{1}_{N \times N}$, and $T$ is the orbit length. \citet{hershey2007approximating} provide a Monte Carlo approximation to the KL divergence between two GMMs, which is given by
\begin{align}
D_{\textrm{stsp}} = D_{\textrm{KL}}\left(\hat{p}(\vx) \ || \ \hat{q}(\vx)\right) \approx \frac{1}{n} \sum_{i=1}^n \log \frac{\hat{p}(\vx^{(i)})}{\hat{q}(\vx^{(i)})},
\end{align}
with $n$ Monte Carlo samples $\vx^{(i)}$ randomly drawn from the GMM based on observed orbits. We follow \citet{brenner22} in using $m = 30$ for the binning approach and $\sigma^2 = 1.0$ for the GMM approach. 

\paragraph{Temporal Measure ($D_H$)}\label{appx:hellinger_distance} Given power spectra $f_i(\omega)$ and $g_i(\omega)$ of the $i$-th dynamical variable of the observed and reconstructed orbits, respectively, with $\int_{-\infty}^\infty f_i(\omega) d\omega = 1$ and $\int_{-\infty}^\infty g_i(\omega) d\omega = 1$, the Hellinger distance is given by
\begin{align}\label{eq:hellinger_distance}
    H(f_i(\omega), g_i(\omega))= \sqrt{1-\int_{-\infty}^{\infty}\sqrt{f_i(\omega)g_i(\omega)}\ d\omega} \
\end{align} 
In practice, we approximate power spectra by performing a Fast Fourier Transform (FFT; \citet{cooley_algorithm_1965}) yielding $\hat{\vf}_i = \rvert\mathcal{F}x_{i, 1:T}\lvert^2$ and $\hat{\vg}_i = \lvert\mathcal{F}\hat{x}_{i, 1:T}\rvert^2$, with vectors $\hat{\vf}_i$ and $\hat{\vg}_i$ 
discrete 
power spectra of ground truth traces $x_{i, 1:T}$ and model generated traces $\hat{x}_{i, 1:T}$. Power spectra are then smoothed using a Gaussian filter with standard deviation $\tilde{\sigma} = 20$ and window length $l=8\sigma + 1$.
Before computing the Hellinger distance, the power spectra are normalized to fulfill $\sum_\omega \hat{f}_{i, \omega} = 1$ and $\sum_\omega \hat{g}_{i, \omega} = 1$. Finally, $H$ \eqref{eq:hellinger_distance} is computed as 
\begin{align}
    H(\hat{\vf}_i, \hat{\vg}_i) = \frac{1}{\sqrt{2}} \norm{\sqrt{\hat{\vf}_i}-\sqrt{\hat{\vg}_i}}_2,
\end{align}
where the square root is applied elementwise. The final measure $D_H$ is obtained by averaging $H$ across all dimensions:
\begin{align}
    D_H = \frac{1}{N} \sum_{i=1}^N H(\hat{\vf}_i, \hat{\vg}_i) . 
\end{align}

\paragraph{Prediction Error}\label{appx:prediction_error}
We compute an $n$-step prediction error as the MSE between ground truth data and $n$-step ahead predictions of the model:
\begin{align}
    \textrm{PE}(n) = \frac{1}{N(T-n)} \sum_{t=1}^{T-n} \left\lVert\vx_{t+n} - \hat{\vx}_{t+n}\right\rVert_2^2 .
\end{align}
Note that for chaotic systems, due to the exponential divergence of trajectories, the invariant statistics ($D_{\textrm{stsp}}$, $D_{H}$) and the $n$-step prediction error may dissociate, as illustrated in \citet{koppe2019identifying} and \citet{schmidt2021identifying}.

\paragraph{Evaluation Setup} For computing $D_{\textrm{stsp}}$ and $D_{H}$, we first draw a single long orbit of the generative model at hand (i.e. either the shPLRNN or any of the comparison methods employed, cf. Table \ref{tab:SOTA_comparison}) of length $1.25 \cdot T$, where $T$ is the total length of the available (test) data. We then discard the first $0.25\cdot T$ time steps of the model-generated orbit to make sure the measures are evaluated on the limit sets and not on transients of the dynamics.\footnote{This is important, since otherwise erroneously a good reconstruction may be indicated, while truly the reconstructed system may converge to a limit set topologically different from that of the true system, e.g. an equilibrium rather than a chaotic attractor.}
We then use the remaining $T$ time steps to compute both measures. The initial condition for the model is determined from
the first time step of the observed ground truth orbit. For comparison methods requiring a dynamical warm-up phase, such as RC, multiple time steps of the ground truth orbit are provided. The prediction error ($\textrm{PE}(n)$) is computed across the entire test set.

\subsection{GTF and RNN architecture}\label{appx:ablation}
In theory, GTF is independent of the specific RNN architecture or, more generally, map $\mF_{\theta}$ employed. In practice, we observed training the dendPLRNN \cite{brenner22} with GTF does not lead to similarly strong performance boosts over STF-based training as observed for the shPLRNN (Tab. \ref{tab:ablation}), implying that certain RNN designs like the shPLRNN might be more amenable to GTF. 
This could be because for the dendPLRNN we do not have a 1:1 relation between observations and latent states as for the shPLRNN (rather, we need to go to higher dimensions, cf. Tab. \ref{tab:SOTA_comparison}). This makes implementation of GTF for the dendPLRNN less straightforward, since - unlike for the shPLRNN - the mapping onto latent states is underdetermined. For the experiments in Tab. \ref{tab:ablation}, we used an 
implementation of GTF similar to id-TF \cite{brenner22}, only forcing the dendPLRNN's first $N$ latent states and leaving the remaining $M-N$ states unforced. Alternatively, one may use inversion of a trainable linear observation model for projecting the teacher signal into the model's latent space \cite{Mikhaeil22}. 
However, this does not resolve the underdeterminacy (on the contrary, it seemed to even exacerbate the problem, possibly because the additional degrees of freedom introduced by the trainable projection operator might make the GTF forcing even less tight). 
Hence, it needs to be concluded that the shPLRNN also appears to bear a specific architectural advantage, while for other models 
the best way for implementing GTF needs more consideration. 
We also observed that replacing the ReLU activation of the shPLRNN by $\tanh$ diminished performance (Tab. \ref{tab:ablation}). 

\begin{table*}[!h]
\centering
\definecolor{Gray}{gray}{0.925}
\caption{Ablation study. Reported values are median $\pm$ median absolute deviation (MAD) over $20$ independent training runs. `dim' refers to the model's state space dimensionality (number of dynamical variables). $\lvert \bm{\theta} \rvert$ denotes the total number of \textit{trainable} parameters. 
Values for shPLRNN+GTF are copied 
from Tab. \ref{tab:SOTA_comparison}.}
\scalebox{0.99}{
\begin{tabular}{l l c c c c c}
        \toprule
        Dataset	&	Method	&	$D_{\textrm{stsp}}$ $\downarrow$ & $D_H$ $\downarrow$  & $\textrm{PE}(20)$ $\downarrow$	&  dim & $\lvert \bm{\theta} \rvert$ \\
        \midrule
        \multirow{4}{4em}{ECG (5d)}
        &    shPLRNN + GTF  & $\textbf{4.3}\pm\textbf{0.6}$ & $\textbf{0.34}\pm\textbf{0.02}$ & $(\textbf{2.4}\pm\textbf{0.1}) \cdot \textbf{10}^{\textbf{-3}}$ & $5$ & $2785$ \\
        &    shPLRNN + GTF ($tanh$)      & $7.8 \pm 2.1$ & $0.37 \pm 0.03$ & $(8.6 \pm 0.3) \cdot 10^{-3}$ & $5$ & $2785$ \\
        &    dendPLRNN + GTF      & $5.2\pm1.2$ & $0.34\pm0.02$ & $(5.5\pm1.3) \cdot 10^{-3}$ & $35$ &  $3245$ \\
        &    dendPLRNN + STF      & $5.8\pm0.6$ & $0.37\pm0.06$ & $(4.0\pm0.4) \cdot 10^{-3}$ & $35$ &  $3245$ \\

        \midrule
        \multirow{2}{4em}{EEG (64d)}
        &    shPLRNN + GTF        & $\textbf{2.1}\pm\textbf{0.2}$ & $\textbf{0.11}\pm\textbf{0.01}$ & $(\textbf{5.5}\pm\textbf{0.1}) \cdot \textbf{10}^{\textbf{-1}}$ & $16$ & $17952$  \\
        &    shPLRNN + GTF ($tanh$) & $17 \pm 3$ & $0.37 \pm 0.07$ & $({6.6}\pm{0.1}) \cdot {10}^{{-1}}$  & $16$ & $17952$  \\
        \bottomrule
\end{tabular}
}
\label{tab:ablation}
\end{table*}

\subsection{More Details on Comparison Methods}\label{appx:sota_settings}
\begin{table*}[!ht]
\definecolor{Gray}{gray}{0.925}
\caption{SOTA comparisons on the Lorenz-63 and Lorenz-96 benchmark systems. Reported values are median $\pm$ median absolute deviation 
over $20$ independent training runs. `dim' refers to the model's state space dimensionality (number of dynamical variables). $\lvert \bm{\theta} \rvert$ denotes the total number of \textit{trainable} parameters. Note that SINDy has an inbuilt advantage for these two benchmarks over all other methods: Both the Lorenz-63 and Lorenz-96 ODEs are second-order polynomials, and SINDy's polynomial function library included terms up to second order as well. Hence, all that SINDy needs to do in these cases is to determine the appropriate parameters 
(rather than approximating the underlying system). While the library method may be an advantage for systems with known functional form (as in these cases), it may become a severe disadvantage if no detailed structural knowledge is available (as commonly the case in complex empirical scenarios like those evaluated in Table \ref{tab:SOTA_comparison}).}
\centering
\scalebox{0.99}{
\begin{tabular}{l l c c c c c}
        \toprule
        Dataset	&	Method	&	$D_{\textrm{stsp}}$ $\downarrow$ & $D_H$ $\downarrow$  & $\textrm{PE}(20)$ $\downarrow$	&  dim & $\lvert \bm{\theta} \rvert$ \\
        \midrule
        \multirow{7}{4em}{Lorenz-63 (3d)}
        &    shPLRNN + GTF & $\textcolor{blue}{\textbf{0.26}}\pm\textcolor{blue}{\textbf{0.03}}$  & $\textcolor{blue}{\textbf{0.090}}\pm\textcolor{blue}{\textbf{0.007}}$ & $(\textcolor{blue}{\textbf{6.0}}\pm\textcolor{blue}{\textbf{0.5}}) \cdot \textcolor{blue}{\textbf{10}}^{\textcolor{blue}{\textbf{-4}}}$ & $3$ & $365$  \\
        &    dendPLRNN + id-TF     & $0.9\pm0.2$  & $0.15\pm0.03$ & $(2.2\pm0.7) \cdot 10^{-3}$ & $10$ & $361$ \\
        &	 {RC}	            & $0.52\pm0.12$ &	$0.19\pm0.04$ & $(5\pm2) \cdot 10^{-3}$ & $201$ & $603$ \\
        &	 {LSTM-TBPTT}	            & $0.46\pm0.22$   &	$0.11\pm0.03$ & $(1.1\pm0.3) \cdot 10^{-3}$  & $30$ & $1188$ \\
        &	 {SINDy}	            & $\textbf{0.24}\pm\textbf{0.00}$    & $\textbf{0.091}\pm\textbf{0.000}$ & $(\textbf{6.1}\pm\textbf{0.0}) \cdot \textbf{10}^{\textbf{-4}}$ & $3$ & $30$  \\
        &	 {N-ODE}	        & $0.63\pm0.2$  & $0.15\pm0.05$ & $(2.3\pm0.3) \cdot 10^{-3}$ & $3$ & $353$  \\
        &	 {LEM}	        & $0.39\pm0.24$  & $0.12\pm0.05$ & $(6.0\pm0.9) \cdot 10^{-3}$ & $14$ & $360$  \\

        \midrule
        \multirow{7}{4em}{Lorenz-96 (20d)}
        &    shPLRNN + GTF      & $1.68\pm0.06$  & $\textcolor{blue}{\textbf{0.072}}\pm\textcolor{blue}{\textbf{0.001}}$ & $(1.21\pm0.02) \cdot 10^{-1}$ & $20$ & $4540$  \\
        &    dendPLRNN + id-TF     & $\textcolor{blue}{\textbf{1.65}}\pm\textcolor{blue}{\textbf{0.05}}$  & $0.083\pm0.005$ & $(\textcolor{blue}{\textbf{1.1}}\pm\textcolor{blue}{\textbf{0.1}}) \cdot \textcolor{blue}{\textbf{10}}^{\textcolor{blue}{\textbf{-1}}}$ & $60$ & $5740$ \\
        &	 {RC}	            & $2.40\pm0.15$ & $0.14\pm0.02$ &  $(4.9\pm0.4) \cdot 10^{-1}$ & $600$ & $12000$ \\
        &	 {LSTM-TBPTT}	            	 & $5\pm1$ & $0.31\pm0.04$ & $(1.14\pm0.04) \cdot 10^{0}$ & $80$ & $10580$ \\
        &	 {SINDy}	            	 & $\textbf{1.59} \pm \textbf{0.00}$ & $\textbf{0.06}\pm\textbf{0.00}$ & $(\textbf{4.6}\pm\textbf{0.0}) \cdot \textbf{10}^{\textbf{-3}}$ & $20$ & $4620$ \\
        &	 {N-ODE}	        &  $1.77\pm0.07$   &  $0.076\pm0.01$   &  $(2.5\pm0.02) \cdot 10^{-1}$ & $20$ & $4530$  \\
        &	 {LEM}	        &  $7.2\pm2.3$   &  $0.54\pm0.13$   &  $(1.3\pm0.06) \cdot 10^{0}$ & $46$ & $4620$ \\
        \bottomrule
\end{tabular}
}
\label{tab:SOTA_comparison_benchmark}
\end{table*}
\paragraph{SINDy} For SINDy we used the \href{https://github.com/dynamicslab/pysindy}{PySINDy} package \citep{desilva2020, Kaptanoglu2022} with the sequentially thresholded least squares (STLSQ) optimizer. We scanned the threshold hyperparameter, which determines the sparsity of the final solution (the higher the threshold, the sparser the solution), in the range $\left[0, \ 1\right)$. Furthermore, since all employed datasets are noisy, we used PySINDy's \texttt{SmoothedFiniteDifference} for empirical vector field estimates, leading to more robust estimates. For the Lorenz-63 and Lorenz-96 systems, we used a polynomial basis library (\texttt{PolynomialLibrary}) up to
order $2$. Note that this essentially means SINDy only needs to figure out the right values of the parameters for these problems, as its set of equations is almost the same as in the ground truth systems to begin with. For the empirical ECG and EEG data, we experimented with terms up to $5$-th order and also tried a Fourier basis (\texttt{FourierLibrary}), but did not manage to obtain any solution which would not quickly diverge during numerical integration. 
Orbits for the Lorenz-63 and Lorenz-96 systems were drawn using the \texttt{LSODA} integrator with absolute tolerance $10^{-4}$ and relative tolerance $10^{-6}$.

\paragraph{RC and LSTM} For RC and LSTM we used the \href{https://github.com/pvlachas/RNN-RC-Chaos}{official repository} provided by the authors \citep{pathak2018model, vlachas2018data, vlachas2020backpropagation}. For RC \citep{pathak2018model}, we searched for optimal hyperparameters for \{\texttt{dynamics\_length}, \ \texttt{regularization}, \ \texttt{noise\_level}\}. For LSTMs trained with truncated BPTT \citep{vlachas2018data}, we determined optimal hyperparameters for \{\texttt{hidden\_state\_propagation\_length}, \ \texttt{regularization}, \ \texttt{learning\_rate}, \  \texttt{noise\_level}, \ \texttt{sequence\_length}\}.

\paragraph{dendPLRNN + id-TF} For comparison to the method in \citet{brenner22}, we used the \href{https://github.com/DurstewitzLab/dendPLRNN}{corresponding repository}. We scanned across different values for the sparse TF interval $\tau$ (\texttt{teacher\_forcing\_interval}) and sequence length (\texttt{seq\_len}), biased by the values reported in \citet{brenner22}.

\paragraph{LEM} For comparison to the method proposed in \citet{rusch2022lem}, we used code provided by the authors on their public \href{https://github.com/tk-rusch/LEM}{respository}. The
size of the network was determined such that
a comparable number of trainable parameters was obtained as for the other methods, while optimal hyperparameters for the learning rate and time constant $\Delta t$ were selected via grid search.

\paragraph{Neural ODE}
For comparison to N-ODEs \cite{chen2018neural}, we used the implementation provided in the \texttt{torchdiffeq} package. As hyperparameters we considered the employed activation function \{$relu, tanh\}$ and the sequence length $\in \{1,5,10,25,50\}$ used per minibatch. We further tried several fixed-step numerical solvers (\texttt{rk4},  \texttt{euler}, \texttt{midpoint}), which had little influence on the results, while an adaptive-step solver (\texttt{adaptive\_heun}) led to unacceptably long training times. 
We further probed the N-ODE variants Latent-ODE and ODE-RNN proposed in \cite{rubanova_latent_2019}, using the implementation provided on the authors' \href{https://github.com/YuliaRubanova/latent_ode}{github page}. The results in Table \ref{tab:node} imply that neither of these accomplishes more faithful
reconstructions of EEG data than `vanilla' N-ODE. In fact, the values for $D_{\textrm{stsp}}$ and $D_H$ obtained for ODE-based methods essentially all reflect ``chance level'' in the sense that the resulting long-term dynamics bears no similarity
with the one observed.

We stress that most of these methods achieved fairly good reconstruction results on the simulated benchmarks (i.e., the various Lorenz systems introduced in Appx. \ref{appx:datasets}), see Tables \ref{tab:reduced-mslor96} and \ref{tab:SOTA_comparison_benchmark}. 
They may also be competitive in their short-term forecasts on the empirical data (EEG \& ECG, see Table \ref{tab:SOTA_comparison} and Fig. \ref{fig:eeg_forecast}).
However, as Figs. \ref{fig:eeg_reconstructions}, \ref{fig:eeg_reconstructions_heatmaps} and \ref{fig:ecg_reconstructions} demonstrate, unlike shPLRNN+GTF and dendPLRNN, they failed to reproduce the long-term behavior of the empirically observed systems, i.e. they failed to reconstruct underlying geometrical and invariant temporal properties as essential in this context. 

\begin{table*}[!h]
\centering
\definecolor{Gray}{gray}{0.925}
\caption{Results for Latent-ODE and ODE-RNN \cite{rubanova_latent_2019} on the EEG data. Reported values are median $\pm$ median absolute deviation (MAD) over $20$ independent training runs. `dim' refers to the model's state space dimensionality (number of dynamical variables). $\lvert \bm{\theta} \rvert$ denotes the total number of \textit{trainable} parameters. Values for shPLRNN+GTF and N-ODE are copied from Table \ref{tab:SOTA_comparison}.}
\scalebox{0.99}{
\begin{tabular}{l l c c c c c}
        \toprule
        Dataset	&	Method	&	$D_{\textrm{stsp}}$ $\downarrow$ & $D_H$ $\downarrow$  & $\textrm{PE}(20)$ $\downarrow$	&  dim & $\lvert \bm{\theta} \rvert$ \\
        \midrule
        \multirow{4}{4em}{EEG (64d)}
        &    shPLRNN + GTF        & $\textbf{2.1}\pm\textbf{0.2}$ & $\textbf{0.11}\pm\textbf{0.01}$ & $(\textbf{5.5}\pm\textbf{0.1}) \cdot \textbf{10}^{\textbf{-1}}$ & $16$ & $17952$  \\
        &    N-ODE    & $20 \pm 0.5$ & $0.47 \pm 0.01$ & $(5.5 \pm 0.2) \cdot 10^{-1}$ & $64$ & $17995$  \\
        &    Latent ODE      & $16.1 \pm 3$ & $0.47 \pm 0.02$ & $(5.6 \pm 0.2) \cdot 10^{-1}$ & 64 & 17915 \\
        &    ODE-RNN      & $13.9 \pm 2.1$ & $0.59 \pm 0.03$ & $(9.1 \pm 0.6) \cdot 10^{-1}$ & 64 & 17859  \\
        \bottomrule
\end{tabular}
}
\label{tab:node}
\end{table*}
\begin{figure}[!ht]
    \centering
    \begin{subfigure}[b]{.49\textwidth}      \includegraphics[width=\textwidth]{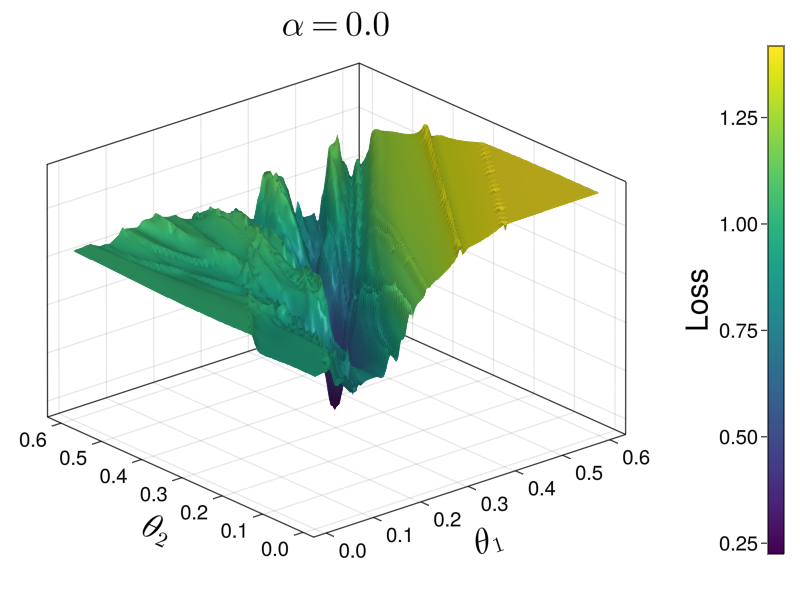}
    \end{subfigure}
    \begin{subfigure}[b]{.49\textwidth}    \includegraphics[width=\textwidth]{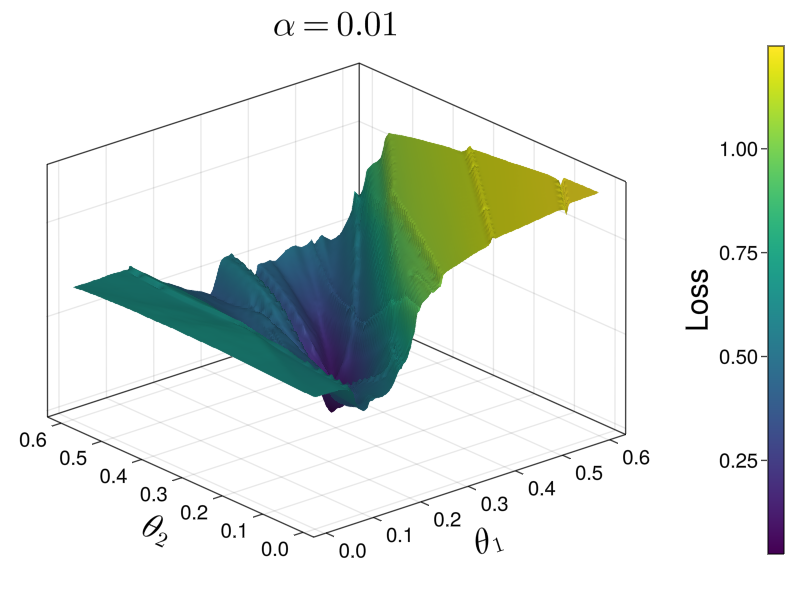}
    \end{subfigure}
        \begin{subfigure}[b]{.49\textwidth}     \includegraphics[width=\textwidth]{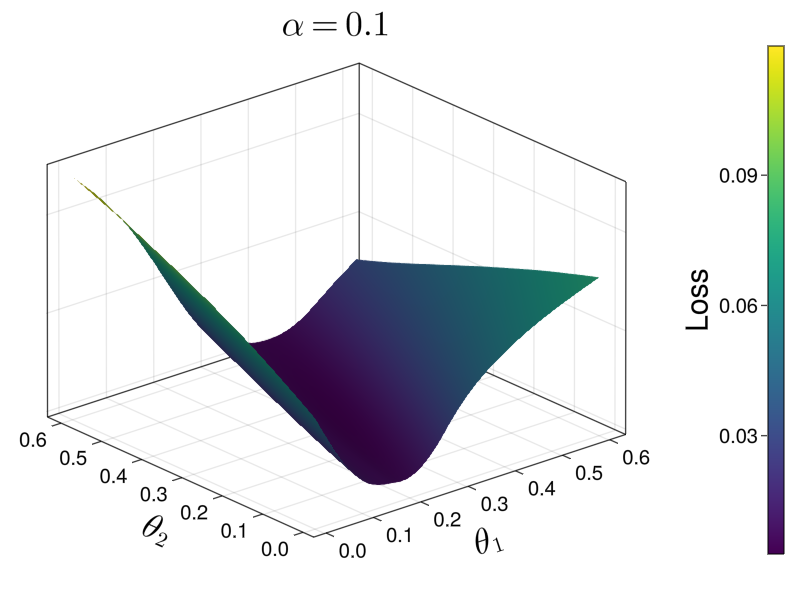}
    \end{subfigure}
    \caption{GTF smoothens loss landscapes. Loss as a function of two arbitrarily chosen ($\theta_1 = w^{(1)}_{1,2}$, $\theta_2 = w^{(1)}_{2,1}$) shPLRNN parameters. The shPLRNN \eqref{eq:bounded_shPLRNN} was first trained on the Lorenz-63 system for a couple of epochs ($\alpha=0.15$, $\Tilde{T}=200$, $S=16$), after which the loss was determined for various $\alpha$ values based on a random batch of the training data. For larger $\alpha$ the loss landscape smoothens out.}
    \label{fig:loss_landscape_smoothing}
\end{figure}

\begin{figure}[!htp]
    \centering
\includegraphics[width=0.99\linewidth]{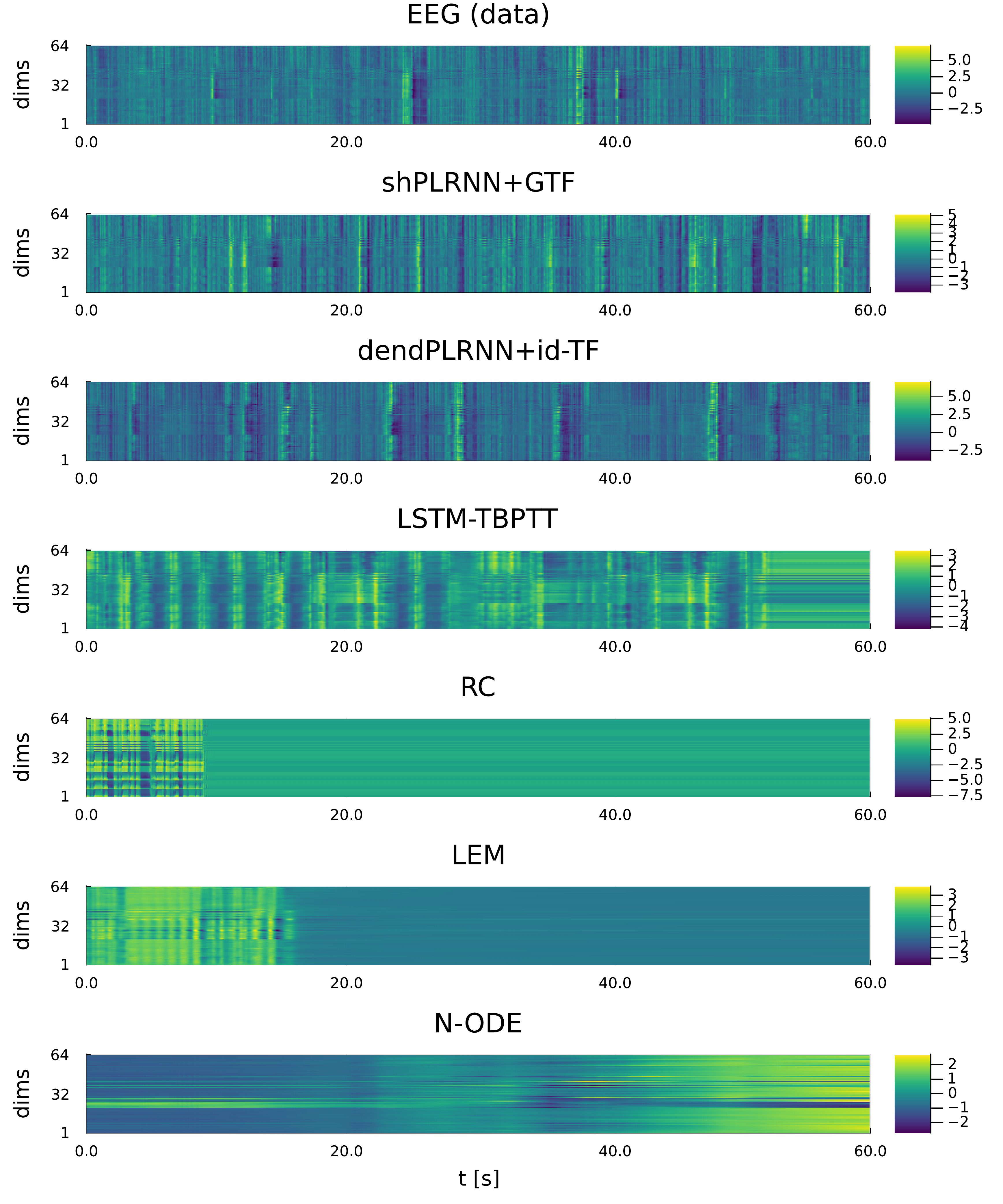}
    \caption{Example heatmaps of EEG reconstructions provided by 
    the methods employed in Table \ref{tab:SOTA_comparison}. We used the same models as for Fig. \ref{fig:eeg_reconstructions}. To make the heatmaps comparable, each channel was standardized for each method separately.}
    \label{fig:eeg_reconstructions_heatmaps}
\end{figure}
\begin{figure}[!htp]
    \centering
\includegraphics[width=0.75\linewidth]
{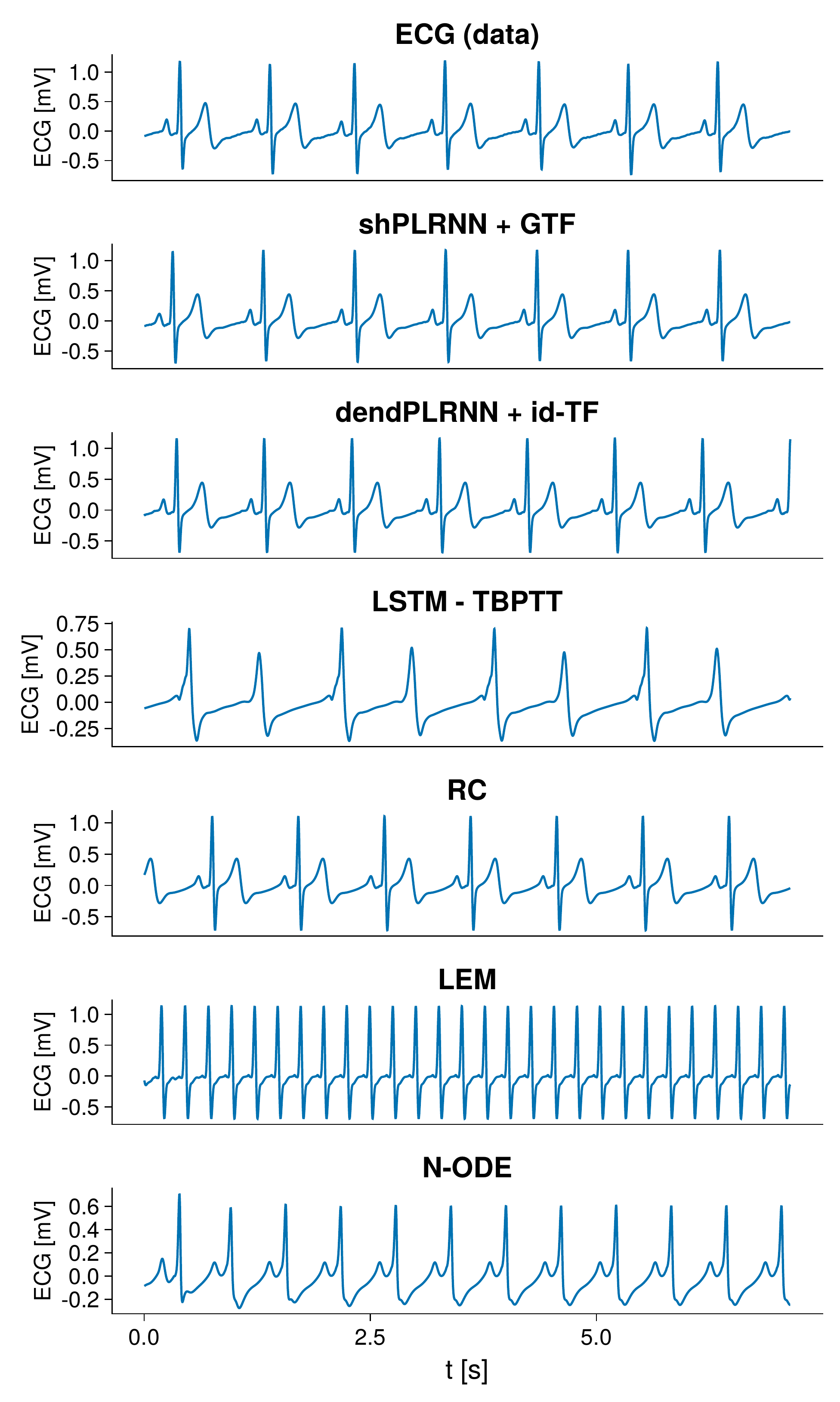}
    \caption{Example time traces of ECG reconstructions provided by 
    the methods employed in Table \ref{tab:SOTA_comparison}. For each method we picked the best reconstruction out of $20$.}
    \label{fig:ecg_reconstructions}
\end{figure}
\begin{figure}[!ht]
    \centering
\includegraphics[width=0.7\linewidth]{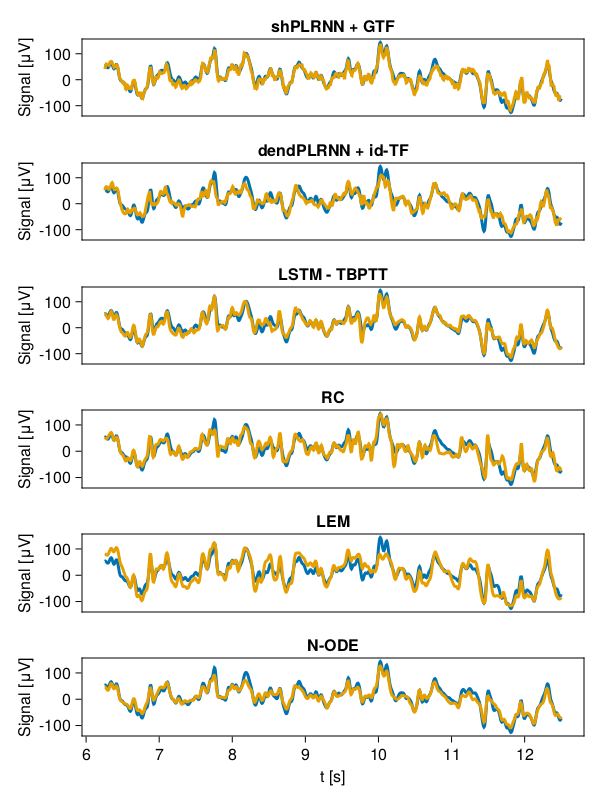}
    \caption{Model \textit{short-term} predictions: Excerpt of EEG time series (blue) vs. $5$-step-ahead predictions (yellow) for the different DS reconstruction models \& methods compared in this work.
    Note that essentially all methods provide reasonable short-term forecasts, yet most fail to produce non-trivial limiting dynamics (cf. Fig. \ref{fig:eeg_reconstructions}).}
    \label{fig:eeg_forecast}
\end{figure}

\end{document}